\documentclass[sigconf,nonacm]{aamas}

\usepackage{balance} 

\usepackage{booktabs}
\usepackage{amsmath}
\usepackage{subfigure}
\usepackage{dsfont}
\usepackage{algorithm}
\usepackage{algorithmic}
\usepackage{multirow}
\newtheorem{theorem}{Theorem}
\newtheorem{definition}{Definition}
\newtheorem{proposition}{Proposition}
\usepackage{eucal}
\usepackage[cal=cm]{mathalfa}
\usepackage{hyperref}

\makeatletter
\gdef\@copyrightpermission{
  \begin{minipage}{0.2\columnwidth}
   \href{https://creativecommons.org/licenses/by/4.0/}{\includegraphics[width=0.90\textwidth]{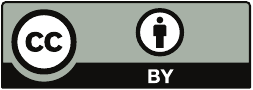}}
  \end{minipage}\hfill
  \begin{minipage}{0.8\columnwidth}
   \href{https://creativecommons.org/licenses/by/4.0/}{This work is licensed under a Creative Commons Attribution International 4.0 License.}
  \end{minipage}
  \vspace{5pt}
}
\makeatother

\setcopyright{ifaamas}
\acmConference[AAMAS '25]{Proc.\@ of the 24th International Conference
on Autonomous Agents and Multiagent Systems (AAMAS 2025)}{May 19 -- 23, 2025}
{Detroit, Michigan, USA}{Y.~Vorobeychik, S.~Das, A.~Nowé  (eds.)}
\copyrightyear{2025}
\acmYear{2025}
\acmDOI{}
\acmPrice{}
\acmISBN{}

\acmSubmissionID{<<OpenReview submission id>>}

\title[AAMAS-2025 Formatting Instructions]{Hierarchical Learning-based Graph Partition for \\Large-scale Vehicle Routing Problems}

\author{Yuxin Pan}
\affiliation{
  \institution{The Hong Kong University of Science and Technology}
  \city{Hong Kong}
  \country{China}
  }
\email{yuxin.pan@connect.ust.hk}

\author{Ruohong Liu}
\affiliation{
  \institution{University of Oxford}
  \city{Oxford}
  \country{United Kingdom}}
\email{ruohong.liu@eng.ox.ac.uk}

\author{Yize Chen}
\affiliation{
  \institution{University of Alberta}
  \city{Edmonton}
  \country{Canada}}
\email{yize.chen@ualberta.ca}

\author{Zhiguang Cao}
\affiliation{
  \institution{Singapore Management University}
  \city{Singapore}
  \country{Singapore}}
\email{zhiguangcao@outlook.com}

\author{Fangzhen Lin}
\affiliation{
  \institution{The Hong Kong University of Science and Technology}
  \city{Hong Kong}
  \country{China}}
\email{flin@cs.ust.hk}

\begin{abstract}
Neural solvers based on the divide-and-conquer approach for Vehicle Routing Problems (VRPs) in general, and capacitated VRP (CVRP) in particular, integrates the global partition of an instance with local constructions for each subproblem to enhance generalization. However, during the global partition phase, misclusterings within subgraphs have a tendency to progressively compound throughout the multi-step decoding process of the learning-based partition policy. This suboptimal behavior in the global partition phase, in turn, may lead to a dramatic deterioration in the performance of the overall decomposition-based system, despite using optimal local constructions. To address these challenges, we propose a versatile Hierarchical Learning-based Graph Partition (HLGP) framework, which is tailored to benefit the partition of CVRP instances by synergistically integrating global and local partition policies. Specifically, the global partition policy is tasked with creating the coarse multi-way partition to generate the sequence of simpler two-way partition subtasks. These subtasks mark the initiation of the subsequent K local partition levels. At each local partition level, subtasks exclusive for this level are assigned to the local partition policy which benefits from the insensitive local topological features to incrementally alleviate the compounded errors. This framework is versatile in the sense that it optimizes the involved partition policies towards a unified objective harmoniously compatible with both reinforcement learning (RL) and supervised learning (SL). Additionally, we decompose the synchronized training into individual training of each component to circumvent the instability issue. Furthermore, we point out the importance of viewing the subproblems encountered during the partition process as individual training instances. Extensive experiments conducted on various CVRP benchmarks demonstrate the effectiveness and generalization of the HLGP framework. The source code is available at \href{https://github.com/panyxy/hlgp_cvrp}{https://github.com/panyxy/hlgp\_cvrp}.

\footnotetext{This paper has been accepted as a Full Paper at the 24th International Conference on Autonomous Agents and Multiagent Systems (AAMAS 2025).}
\end{abstract}


\newcommand{\BibTeX}{\rm B\kern-.05em{\sc i\kern-.025em b}\kern-.08em\TeX}

\begin{document}


\pagestyle{fancy}
\fancyhead{}


\maketitle 


\section{Introduction}
The Vehicle Routing Problem (VRP) is a widely-studied NP-hard problem which has many real-world applications including transportation~\cite{garaix2010vehicle}, logistic~\cite{cattaruzza2017vehicle}, and digital e-commerce~\cite{liu2017capacitated}. Exact methods for solving VRP often use mixed integer linear programming (MILP) techniques and employ MILP solvers to generate optimal solutions with theoretical guarantees~\cite{laporte1983branch}. However, these methods so far are not computationally efficient enough to handle large-scale instances, particularly for the applications with time-sensitive and dynamically changing VRP scenarios. In contrast, heuristic methods such as LKH3~\cite{helsgaun2017extension} and HGS~\cite{vidal2012hybrid} aim to generate high-quality solutions quickly. They commonly improve the quality of the existing solution incrementally by local search techniques. However, in addition to the heavy reliance on the quality of handcrafted local operators, these methods are not robust and often need to start from scratch for the problem instances with slight variations.

More recently, there has been much work on neural network-based solvers for VRP. Experimentally they have been shown capable of inferring near-optimal efficiently solutions for instances which fall within the training data distribution. These learning-based solvers typically use one of the following methods: constructive, iterative, and divide-and-conquer. The constructive method, as a pioneering paradigm, incrementally deduces a complete solution starting from an empty state~\cite{vinyals2015pointer, nazari2018reinforcement, kool2018attention, xin2020step, kwon2020pomo, kim2022sym, liu20242d}. However, challenges arise when dealing with out-of-distribution instances due to the limited expressivity of neural network and the intricate search landscape. To mitigate this performance degradation, the iterative method merges a neural network-based policy with heuristic local operators to progressively refine the current solution~\cite{lu2019learning, chen2019learning, hottung2020neural, ma2021learning, xin2021neurolkh, ma2024learning}. Yet, this approach relies on numerous improvement steps with well-crafted local operators for satisfactory solutions. By comparison, the divide-and-conquer approach embraces either a heuristic-based partition policy~\cite{fisher1981generalized, fu2021generalize, kim2021learning, li2021learning, zong2022rbg, cheng2023select} or a neural partition policy~\cite{pan2023h, hou2023generalize, ye2024glop, zheng2024udc} to globally divide the entire graph into subgraphs and employ a local construction policy to solve subproblems. However, a failure in either component policy may lead to a significant performance drop. Moreover, heuristic-based partition rules often result in local optima, and neural partition policies may be vulnerable to distribution shifts. Hence, there is a pressing need for a more generalizable and meticulous partition policy, which is the focus of this paper.

In the divide-and-conquer paradigm, the local construction policy agent benefits from the local topological features within subproblems insensitive against distribution and scale shifts, contributing to the (near-)optimality of solutions of subproblems~\cite{jiang2023multi, gao2023towards, fang2024invit}. However, during the multi-step decoding process of the learning-based partition policy for Capacitated VRP (CVRP) instances~\cite{hou2023generalize, ye2024glop}, the decoding of clustered nodes in each step relies on the partial partition solution from the preceding step. This implies that errors in the clustering from earlier steps have a tendency to propagate and result in a chain of misclusterings in subsequent steps, called as compounded errors. Consequently, even with an optimal local construction policy, deficiencies in the partition task lead to substantial deviations from the ideal policy in the overall system. Therefore, we argue that the partition task holds a critical position in the overall decomposition-based system for solving CVRP. Furthermore, the success of the local construction policy inspires us to introduce a local partition policy which aims to progressively alleviate compounded errors by harnessing local topological features in the partition task. We thus consider to implement a hierarchical learning (HL) framework specifically designed for the partition task in CVRP, which is capable of seamlessly integrating both global and local partition policies. In prevailing HL frameworks, a high-level policy is adopted to derive a series of simpler sub-tasks which are then delegated to the low-level policy, with the aim of facilitating exploration~\cite{pateria2021hierarchical}. These frameworks predominantly focus on reinforcement learning and undergo joint training of the associated policies~\cite{levy2017learning}. Yet, these HL frameworks have not been extensively explored in addressing compounded errors within the graph partition task of large-scale CVRP. In contrast, our study extends the HL framework to the partition task of CVRP and demonstrates its efficacy in mitigating compounded errors.

In this paper, we present a versatile Hierarchical Learning-based Graph Partition (HLGP) framework specifically tailored for the partition task in CVRP, which synergistically integrates both global and local partitioning policies. To be specific, our method formulates the partition problem of CVRP using a multi-level HL framework. At the global partition level, the global partition policy is responsible for initiating a coarse multi-way partition to create a series of simpler 2-way partition subtasks. These subtasks stand as the starting point for the subsequent \emph{K} local partition levels. At each local partition level, a tailored sequence of subtasks is derived from the partition solution of the preceding level. These subtasks are then fed into the local partition policy. Such a setup allows the local partition policy to mitigate the potential misclustering arising from the previous level by leveraging the insensitive local topological features inherent in the subtask. By enabling the local partition policy to traverse through the subtasks at each local partition level,  the compounded misclusterings can be mitigated progressively across levels as a consequence.

Our proposed HLGP framework is versatile, featuring a unified objective that effortlessly accommodates both reinforcement learning (RL) and supervised learning (SL) for training the partition policies. It is worth noting that unlike prior approaches that utilized SL to directly train the neural solver, our method explores the application of SL for training the partition policy, usually omitting the need for permutation information. Additionally, the joint training of the involved policies is disentangled to mitigate training instability. Moreover, by conducting in-depth analyses in both the RL and SL settings, we shed light on the importance of viewing the subproblems encountered during the partition process as individual training instances. Empirically, the proposed HLGP framework convincingly demonstrates its superiority through extensive experiments on various CVRP benchmarks over previous SOTA methods. In particular, our method can scale up to CVRP10K instances with around $10\%$ performance improvement over current literature.

\section{Related Works}
Learning-based methods for solving combinatorial optimization problems (COPs) typically fall into three main categories: \textbf{constructive methods}, \textbf{iterative methods}, and \textbf{divide-and-conquer methods}. Constructive methods aim to progressively infer complete solutions using the autoregressive mechanism~\cite{vinyals2015pointer, nazari2018reinforcement, kool2018attention, kwon2020pomo, kim2022sym, son2023meta, zhou2023towards, manchanda2022generalization, qiu2022dimes, gao2023towards, jiang2024ensemble, grinsztajn2023winner, bi2022learning}. Impressively, SL-driven constructive policies, such as BQ~\cite{drakulic2024bq}, LEHD~\cite{luo2024neural} and SIL~\cite{luo2024self} can mitigate the high GPU memory demands associated with gradient backpropagation by eliminating the need for delayed rewards in the training of RL algorithms. Iterative methods offer the benefit of consistently improving a given solution until convergence~\cite{lu2019learning, chen2019learning, hottung2020neural, ma2021learning, xin2021neurolkh, ma2024learning} by integrating local improvement operators into RL policies. The divide-and-conquer paradigm can exploit local topological features that remain insensitive to distribution or scale shifts, thus alleviating performance degradation. Some methods harness heuristic rules for the partitioning process~\cite{fu2021generalize, kim2021learning, cheng2023select, li2021learning, zong2022rbg}. In contrast, H-TSP~\cite{pan2023h}, TAM~\cite{hou2023generalize}, GLOP~\cite{ye2024glop}, and UDC~\cite{zheng2024udc} opt to use a learning-based policy to globally divide the entire instance into subproblems, which are then addressed by a pretrained local construction policy. \emph{For a more detailed review of related algorithms used to solve VRPs, please refer to the Appendix-D.}

\section{Preliminaries}
\textbf{CVRP Formulation. }A CVRP instance $I$ is defined as a tuple, represented by $I=(G, D, N_{\mathrm{max}})$. The graph $G$ consists of a depot node $v_{0}$ and $N_{v}$ customer nodes $v_{i}$ ($1 \leq i \leq N_{v}$). $D$ and $N_{\mathrm{max}}$ denote the vehicle capacity and maximum allowable number of times vehicles returning to the depot, respectively. Each node is associated with its coordinates $(a_{i}, b_{i})$ and each customer node is further associated with a demand $d_{i}$. $N_{\mathrm{max}}$ is accordingly defined as $\lceil\sum_{i=1}^{N_{v}} d_{i}/D\rceil + 1$. The distance between any pair of nodes can be measured by the Euclidean distance. In the CVRP, the vehicle needs to visit each customer exactly once, fulfills their demands without exceeding $D$, and returns to the depot to reload goods if necessary. A feasible solution $\mathcal{T} \in \mathbb{S}_{\mathcal{T}}$ can be described as a node permutation where the depot node can occur multiple times, while each customer node appears only once. Furthermore, the feasible solution $\mathcal{T}$ also can be decomposed into $N_{\tau}$ feasible subtours. In each subtour $\tau_{i}$ ($1 \leq i \leq N_{\tau}$), the starting and ending nodes are the depot, and the intermediate nodes are customers. The travel cost $e(\tau_{i})$ associated with $\tau_{i}$ is the sum of Euclidean distances along this subtour. Thus, the objective is to minimize $e(\mathcal{T})=\sum_{i=1}^{N_{\tau}}e(\tau_{i})$. $\mathbb{S}_{\mathcal{T}}$ denotes the space of feasible solutions, as does the notation used for $\mathbb{S}$ in the following sections.

\noindent \textbf{Global Partition and Local Construction (GPLC). } In the context of CVRP, the partition task involves clustering nodes into distinct groups, ensuring that each cluster includes the depot node, each customer node belongs to a single cluster, and the total demand within each cluster does not surpass $D$. Each cluster of nodes forms a subgraph. A feasible partition solution $\mathcal{C} \in \mathbb{S}_{\mathcal{C}}$ comprises $N_{c}$ subgraphs. Each subgraph $c_{i}$ ($1 \leq i \leq N_{c}$) consists of the depot node along with various customer nodes distinct from those in other subgraphs (i.e., $\cup_{i=1}^{N_{c}}c_{i} = G$ and $\cap_{i=1}^{N_{c}}c_{i} = \{ v_{0} \}$). The feasible partition solution $\mathcal{C}$ can also be represented as a node list where the order of customer nodes is ignored between the two consecutively visited depot nodes. Therefore, a feasible solution $\mathcal{T}$ can be equivalently seen as a feasible partition solution $\mathcal{C}$ when disregarding the node permutation information in $\mathcal{T}$. 

Obviously, both the original CVRP and the partition problem within CVRP revolve around feasible solutions $\mathcal{T}$ and $\mathcal{C}$ composed of discrete value variables. This fact prompts prevalent learning-based methods to employ stochastic policies as the neural solver to handle whatever types of problems (permutation or partition) within the context of CVRP. Let $\Delta(\cdot)$ denote the space of the probability measure. In the GPLC paradigm~\cite{hou2023generalize, ye2024glop}, a stochastic partition policy $\pi_{\mathrm{part}}(\mathcal{C}|I) \in \Delta(\mathbb{S}_{\mathcal{C}})$ is used to derive a feasible partition solution $\mathcal{C}$ by dividing the graph $G$. Then, a (near-)optimal local permutation policy $\pi_{\mathrm{perm}}^{\ast}(\mathcal{T}|\mathcal{C}) \in \Delta(\mathbb{S}_{\mathcal{T}})$ can generate the feasible subtour $\tau_{i}$ for each subproblem $(c_{i}, D, 1)$. The objective is to identify an optimal partition policy $\pi_{\mathrm{part}}^{\ast}$ that minimizes $e(\mathcal{T})$. \emph{However, prior GPLC methods lack theoretical foundations of the partition problem. Therefore, we introduce Theorem~\ref{thm:gplc} to establish the rationality of the partition problem for CVRP.} Please see Appendix-C.1 for proofs.
\begin{theorem}
\label{thm:gplc}
    The objective in solving an original CVRP instance $I$ is to identify a (permutation) policy $\pi(\mathcal{T}|I) \in \Delta(\mathbb{S}_{\mathcal{T}})$ so as to minimize the expected cost $\mathbb{E}_{\mathcal{T} \sim \pi}[e(\mathcal{T})]$. If $\pi_{\mathrm{perm}}^{\ast} \in \Delta(\mathbb{S}_{\mathcal{T}})$ is optimal for each subproblem $(c_{i}, D, 1)$, then the original objective can be reframed as identifying an optimal partition policy $\pi_{\mathrm{part}}^{\ast} \in \Delta(\mathbb{S}_{\mathcal{C}})$ to minimize the expected cost, expressed as:
    \begin{equation}
    \label{equ:gplc} \min_{\pi_{\mathrm{part}}} \; \mathbb{E}_{\mathcal{C} \sim \pi_{\mathrm{part}}} [\sum_{i=1}^{N_{c}} \mathbb{E}_{\tau_{i} \sim \pi_{\mathrm{perm}}^{\ast}} (e(\tau_{i})) ],
    \end{equation}
    where $\pi_{\mathrm{perm}}^{\ast}(\mathcal{T}|\mathcal{C}) = \prod_{i=1}^{N_{c}} \pi_{\mathrm{perm}}^{\ast}(\tau_{i}|c_{i})$ implies that $\tau_{i}$ is independently sampled given the corresponding $c_{i}$. As aforementioned, viewing a feasible partition solution $\mathcal{C}$ as a node list implies that the partition policy can incrementally construct the partition solution. This process involves conditioning the current selected node $\mathcal{C}[n]$ on the partial partition solution $\mathcal{C}[0:n-1]$ ($\mathcal{C}[0]=\emptyset$) and the given problem instance $I$, written as:
    \begin{equation}
    \label{equ:part_pi}
        \pi_{\mathrm{part}}(\mathcal{C}|I)=\prod_{n=1}^{N_{\mathrm{sol}}}\pi_{\mathrm{part}}(\mathcal{C}[n]|\mathcal{C}[0:n-1], I),
    \end{equation}
    where $N_{\mathrm{sol}}$ denotes the length of partition solution. Please note that we abuse the notation of $N_{\mathrm{sol}}$ to denote the length of different partition solutions for brevity in the following sections. Since the objective defined in Equation~\ref{equ:gplc} essentially aligns with that associated with RL, the common approach involves training neural policies using RL.
\end{theorem}
\begin{figure}[t]
    \centering
    \includegraphics[width=1\columnwidth]{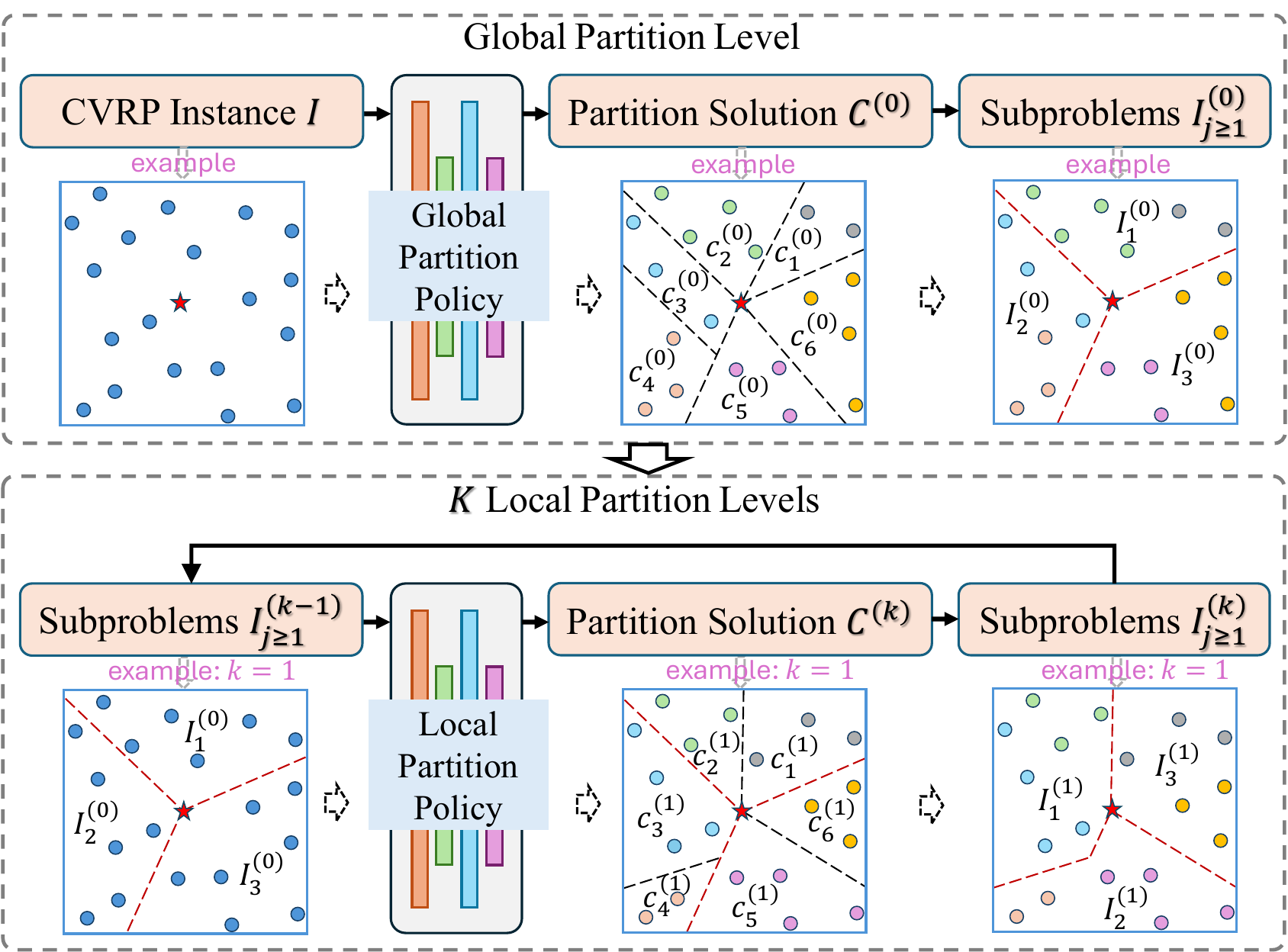}
    \caption{The proposed HLGP framework. $I_{j\geq1}^{k}$ represents a sequence of subproblems. Following the HLGP framework, the sequence of subproblems $I_{j\geq1}^{K}$ are fed to a permutation policy to derive the respective subtours.
    }
    \label{fig:hlgp}
    \vspace{-10pt}
\end{figure}
\section{Hierarchical Learning-based Graph Partition}
Our proposed HLGP framework is built upon the GPLC paradigm. Likewise, we assume the optimal local permutation policy $\pi_{\mathrm{perm}}^{\ast}$ is obtainable by leveraging LKH3~\cite{helsgaun2017extension} or the neural solver used in GLOP~\cite{ye2024glop}. It is evident from the partition policy expression in Equation~\ref{equ:part_pi} that decoding the nodes at each step hinges on the partial partition solution obtained in the preceding steps. Consequently, inaccuracies in clustering from earlier steps tend to propagate, resulting in a chain of misclustering in subsequent steps. These misclusterings in the partition policy exacerbate notably when confronted with substantial distribution or scale shifts. This empirical challenge in CVRP thus motivates us to develop a HL framework for solving the partition problem in CVRP. We anticipate that this HL framework can progressively mitigate compounded errors by incorporating both global and local partition policies.

\subsection{HL Formulation of Partition Problem}
In this section, we begin by introducing the \emph{feasible cost function} $f(\mathcal{C})$ for a feasible partition solution $\mathcal{C}$ as defined in Definition~\ref{def:part_cost}. Following this, various forms of $f(\mathcal{C})$ will be presented in the subsequent sections to align with both RL and SL objectives for training the partition policies. 
\begin{definition}
\label{def:part_cost}
Let $\pi_{\mathrm{part}}^{\ast}$ denote the optimal partition policy obtained by optimizing the objective in Equation~\ref{equ:gplc}. Given a cost function $f(\mathcal{C}): \mathbb{S}_{\mathcal{C}} \rightarrow \mathbb{R}$, if $\pi_{\mathrm{part}}^{\ast}$ can be derived by optimizing the objective $\min_{\pi_{\mathrm{part}}}\mathbb{E}_{\mathcal{C}\sim\pi_{\mathrm{part}}}[f(\mathcal{C})]$, then $f(\mathcal{C})$ is a feasible cost function.
\end{definition}

By leveraging this well-defined feasible cost function $f(\mathcal{C})$, the goal of the partition problem is to minimize $\mathbb{E}_{\mathcal{C}} [f(\mathcal{C})]$. Then, we reformulate the partition problem using a multi-level HL framework. In this framework, the global partition policy $\pi_{\mathrm{Gpart}}$ and the local partition policy $\pi_{\mathrm{Lpart}}$ work together in synergy to execute the partition task, as depicted in Figure~\ref{fig:hlgp}. At the global partition level, $\pi_{\mathrm{Gpart}}$ creates an initial coarse feasible partition $\mathcal{C}^{(0)}=\{ c^{(0)}_{1}, \ldots, c^{(0)}_{N_{c}} \}$, where $c^{(0)}_{i}$ denotes the subgraph at this level. In this partition solution $\mathcal{C}^{(0)}$, each pair of subgraphs $(c^{(0)}_{i}, c^{(0)}_{i\%N_{c}+1})$ (where $1 \leq i \leq N_{c}$) is stipulated as neighboring subgraphs as defined by a specific heuristic rule. For instance, a simple heuristic involves rearranging subgraphs in $\mathcal{C}^{(0)}$ based on the polar angles of their centroids in a Polar coordinate system centered at the depot node. This coarse multi-way partition $\mathcal{C}^{(0)}$ serves as the entry point of the subsequent $K$ local partition levels. At each local partition level $k \in \{1, ..., K\}$, the subproblems are sequentially formed by reuniting pairs of neighboring subgraphs from $\mathcal{C}^{(k-1)}$. Each subproblem $I^{(k-1)}_{j} (1 \leq j \leq \lfloor\frac{N_{c}}{2}\rfloor)$ is defined as:
\begin{equation}
\label{equ:loc_subp}
\begin{split}
    &I^{(k-1)}_{j} = (G_{j}^{(k-1)}, D, 2); \\
    &G_{j}^{(k-1)} = c_{(m+k-1) \% N_{c} + 1}^{(k-1)} \cup c_{(m+k) \% N_{c} + 1}^{(k-1)},
\end{split}
\end{equation}
where $m=2(j-1)$. There are $\lfloor\frac{N_{c}}{2}\rfloor$ subproblems in each local partition level. For each subproblem, the vehicle is only allowed to return twice to the depot by subproblem definition. Please note that each pair of consecutive subproblems $I^{(k-1)}_{j}$ and $I^{(k-1)}_{j+1}$ do not overlap in terms of the subgraphs they contain. Additionally, this technique for creating subproblems can be described as initially left-shifting the subgraphs in $\mathcal{C}^{(k-1)}$ by $k-1$ places and then merging the neighboring subgraphs without overlaps. At each local partition level $k \geq 1$, the subproblem sequence is directed to the local partition policy $\pi_{\mathrm{Lpart}}$. This allows the local partition policy $\pi_{\mathrm{Lpart}}$ to address potential misclusterings from the preceding level by utilizing the robust local topological features. As a result, the local partition policy can traverse through subproblems at each local partition level, gradually reducing accumulated misclusterings across levels. Moreover, upon completion of solving the subproblem $I^{(k-1)}_{j}$, the pair of subgraphs $(c_{(m+k-1)\%N_{c}+1}^{(k-1)}, c_{(m+k)\%N_{c}+1}^{(k-1)})$ is transitioned to the corresponding subgraph pair $(c_{(m+k-1)\%N_{c}+1}^{(k)}, c_{(m+k)\%N_{c}+1}^{(k)})$. Consequently, the resolution of the subproblem sequence results in an update from $\mathcal{C}^{(k-1)}$ to $\mathcal{C}^{(k)}$.

Within the overall HLGP framework, the global partition policy $\pi_{\mathrm{Gpart}}$ is formulated identical to the partition policy in the GPLC method, written as:
\begin{equation}
\label{equ:glb_part_pi}
\pi_{\mathrm{Gpart}}(\mathcal{C}^{(0)}|I)=\prod_{n=1}^{N_{\mathrm{sol}}}\pi_{\mathrm{Gpart}}(\mathcal{C}^{(0)}[n]|\mathcal{C}^{(0)}[0:n-1], I),
\end{equation}
where $\mathcal{C}^{(0)}[n]$ and $\mathcal{C}^{(0)}[0:n-1]$ denote the $n$-th selected node and the partial solution in $\mathcal{C}^{0}$, respectively. In contrast, the local partition policy addresses the series of subproblems produced from the previous partition solution $\mathcal{C}^{(k-1)}$ to construct the partition solution $\mathcal{C}^{(k)}$. Let $\mathcal{C}_{j}^{(k-1)}$ denote the partition solution for the subproblem $I_{j}^{(k-1)}$. Again, the partition solution $\mathcal{C}_{j}^{(k-1)}$ can be either represented as a node list where $\mathcal{C}_{j}^{(k-1)}[n]$ and $\mathcal{C}_{j}^{(k-1)}[0:n-1]$ indicate the $n$-th node and partial solution within it respectively, or decomposed into two subgraphs $c_{(m+k-1) \% N_{c} + 1}^{(k)}$, $ c_{(m+k) \% N_{c} + 1}^{(k)}$ both of which also belong to $\mathcal{C}^{(k)}$. Thus, it can be expressed as:
\begin{equation}
\label{equ:loc_part_pi}
\begin{split}
    &\pi_{\mathrm{Lpart}}(\mathcal{C}^{(k)} | \mathcal{C}^{(k-1)}, k) 
    = \prod_{j=1}^{\lfloor\frac{N_{c}}{2}\rfloor} \pi_{\mathrm{Lpart}}(\mathcal{C}_{j}^{(k-1)} | I_{j}^{(k-1)}) \\
    = &\prod_{j=1}^{\lfloor\frac{N_{c}}{2}\rfloor} \prod_{n=1}^{N_{\mathrm{sol}}} \pi_{\mathrm{Lpart}} (\mathcal{C}_{j}^{(k-1)}[n]|\mathcal{C}_{j}^{(k-1)}[0:n-1] , I_{j}^{(k-1)}).
\end{split}
\end{equation}
Please note that in the LHS of Equation~\ref{equ:loc_part_pi}, the parameter $k$ representing the level is included as an input to the local partition policy. This inclusion is necessary as the parameter $k$ governs the construction of different subproblem sequences for each level. As a result, the objective of HLGP framework is to minimize the expected cost by optimizing both $\pi_{\mathrm{Gpart}}$ and $\pi_{\mathrm{Lpart}}$, written as: 
\begin{equation}
\label{equ:hl_obj}
 \min_{\pi_{\mathrm{Gpart}}, \pi_{\mathrm{Lpart}}} \mathbb{E}_{\mathcal{C}^{(0)}} \mathbb{E}_{\mathcal{C}^{(1)}} \cdots \mathbb{E}_{\mathcal{C}^{(K)}} [f(\mathcal{C}^{(K)})],
\end{equation}
where $\mathcal{C}^{(0)}$ and $\mathcal{C}^{(k)}$ ($k\geq1$) are sampled from $\pi_{\mathrm{Gpart}}(\mathcal{C}^{(0)}|I)$ and $\pi_{\mathrm{Lpart}}(\mathcal{C}^{(k)}|\mathcal{C}^{(k-1)},k)$, respectively.

\subsection{RL-driven HLGP}

In the HLGP framework, the imperative task at hand involves the joint optimization for the global and local partition policies, as illustrated in Equation~\ref{equ:hl_obj}. To address this intricate optimization challenge through RL algorithms, a rigorous formulation utilizing a multi-level Markov Decision Process (MDP) is required. However, Equation~\ref{equ:hl_obj} essentially revolves around evaluating $\mathcal{C}^{(K)}$ at the $K$-th level. The absence of direct evaluations for $\mathcal{C}^{(k)}, k < K$, primarily contributes to the instability concern during the joint training via RL. We thus equivalently convert it to one involving direct evaluations at each level, as outlined in Theorem~\ref{thm:rl_obj}.

\begin{theorem}
\label{thm:rl_obj}
Let $g(c_{i})$ denote $\mathbb{E}_{\tau_{i} \sim \pi_{\mathrm{perm}}^{\ast}(\cdot|c_{i})}(e(\tau_{i}))$. It is clear that $f(\mathcal{C}) = \sum_{i=1}^{N_{c}} g(c_{i})$ acts as a feasible cost function. Then, the optimization problem defined in Equation~\ref{equ:hl_obj} can be transformed equivalently as follows:
\begin{equation}
\label{equ:trans_obj}
\begin{split}
     \min_{\pi_{\mathrm{Gpart}}, \pi_{\mathrm{Lpart}}} &\mathbb{E}_{\mathcal{C}^{(0)}} [f(\mathcal{C}^{(0)})] +  \mathbb{E}_{\mathcal{C}^{(0)}} \mathbb{E}_{\mathcal{C}^{(1)}}[f(\mathcal{C}^{(1)}) - f(\mathcal{C}^{(0)})]+ \\
    &  \cdots + \mathbb{E}_{\mathcal{C}^{(0)}} \mathbb{E}_{\mathcal{C}^{(1)}} \cdots \mathbb{E}_{\mathcal{C}^{(K)}} [f(\mathcal{C}^{(K)}) - f(\mathcal{C}^{(K-1)})].
\end{split}
\end{equation}
The evaluation for $\mathcal{C}^{(k)}, k \geq 1$, can further be derived as:
\begin{align}
\label{equ:eval_C}
    &f(\mathcal{C}^{(k)}) - f(\mathcal{C}^{(k-1)}) = \sum_{j=1}^{\lfloor\frac{N_{c}}{2}\rfloor}
    [h(\mathcal{C}^{(k)}, k, m) - h(\mathcal{C}^{(k-1)}, k, m)]; \nonumber \\
    &h(\mathcal{C}^{(k)}, k, m) = g(c_{(m+k-1)\%N_{c}+1}^{(k)}) + g(c_{(m+k)\%N_{c}+1}^{(k)}),
\end{align}
where $m=2(j-1)$.
\end{theorem}

Please see Appendix-C.2 for proofs. Theorem~\ref{thm:rl_obj} breaks down the objective described in Equation~\ref{equ:hl_obj} into ${K+1}$ components, with each component associated with the direct evaluation of the respective partition solution. Notably, except for the evaluation of $\mathcal{C}^{(0)}$ which solely considers its own cost $f(\mathcal{C}^{(0)})$, the evaluation of $\mathcal{C}^{(k)}$, $k \geq 1$ hinges on the difference between its own cost $f(\mathcal{C}^{(k)})$ and the cost $f(\mathcal{C}^{(k-1)})$ from the preceding level. At each local partition level $k \geq 1$, the local partition policy is responsible for resolving each subproblem $I_{j}^{k-1}$, leading to the modification of each pair of subgraphs $(c_{(m+k-1)\%N_{c}+1}^{(k-1)}, c_{(m+k)\%N_{c}+1}^{(k-1)})$ to the respective subgraph pair $(c_{(m+k-1)\%N_{c}+1}^{(k)}, c_{(m+k)\%N_{c}+1}^{(k)})$. We are thus allowed to proceed with the derivation as indicated in Equation~\ref{equ:eval_C}. Given the optimization problem stated above, we present the formulation utilizing a multi-level MDP in Proposition~\ref{prop:mlmdp}.

\begin{proposition}
\label{prop:mlmdp}
In the multi-level MDP framework, at the global partition level, for $t\geq1$, the state $x_{t}^{(0)}\in\mathbb{X}^{(0)}$ comprises problem instance $I$ and the partial partition solution $\mathcal{C}^{(0)}[0:t-1]$ ($\mathcal{C}^{(0)}[0] = \emptyset$). The initial distribution $\mu^{(0)}$ aligns with the problem instance distribution $p_{I}$. The action $u_{t}^{(0)}\in\mathbb{U}^{(0)}$ involves selecting a node denoted as $\mathcal{C}^{(0)}[t]$, from unvisited customer nodes and the depot node. Let $i_{t}$ index subgraphs such that at timestep $t$, the agent is constructing $i_{t}$-th subgraph $c_{i_{t}}^{(0)}$. If the subgraph $c_{i_{t}}^{(0)}$ is created, then the reward $r_{t}^{(0)}$ is set as $-g(c_{i_{t}}^{(0)})$; otherwise, it remains at $0$. The global partition policy, parameterized by $\theta_{G}$, is thus specified as $\pi_{\theta_{G}}(u_{t}^{(0)}|x_{t}^{(0)})$. 

At each local partition level $k\geq1$, the local partition policy is tasked with solving the sequence of subproblems obtained from $\mathcal{C}^{(k-1)}$. In this context, we use $j_{t}$ as an index for subproblems, indicating that the $j_{t}$-th subproblem denoted as $I_{j_{t}}^{(k-1)}$, is currently being addressed but remains incomplete at timestep $t$. The state $x_{t}^{(k)}\in\mathbb{X}^{(k)}$ consists of the subproblem sequence and the partial solution of $I_{j_{t}}^{(k-1)}$. The initial state distribution $\mu^{(k)}$ corresponds to the distribution of the subproblem sequence. The action $u_{t}^{(k)}\in\mathbb{U}^{(k)}$ involves selecting a node for solving $I_{j_{t}}^{(k-1)}$. When $I_{j_{t}}^{(k-1)}$ is successfully solved, the index $j_{t}$ will proceed to the next subproblem, and the reward $r_{t}^{(k)}$ is set as $-(h(\mathcal{C}^{(k)}, k, m) - h(\mathcal{C}^{(k-1)}, k, m))$ (where $m=2(j_{t}-1)$). Otherwise, the reward remains at $0$. Thus, the local partition policy parameterized by $\theta_{L}$, is defined as $\pi_{\theta_{L}}(u_{t}^{(k)}|x_{t}^{(k)})$. The objective is to maximize the sum of expected returns across levels, as defined below:
\begin{equation}
\label{equ:rl_obj}
    J(\theta_{G}, \theta_{L}) = \mathbb{E}_{\omega^{(0)}} [\sum_{t=1}^{T^{(0)}}r_{t}^{(0)}] + \cdots+
\mathbb{E}_{\omega^{(0)}}\cdots\mathbb{E}_{\omega^{(K)}} [\sum_{t=1}^{T^{(K)}}r_{t}^{(K)}],
\end{equation}
where $T^{(k)}$ and $\omega^{(k)}$ denote the horizon and the trajectory at level $k$.
\end{proposition}

\begin{figure*}[t]
\centering
\subfigure[RL-driven HLGP]{\label{fig:rl_hlgp}\includegraphics[width=0.4\textwidth]{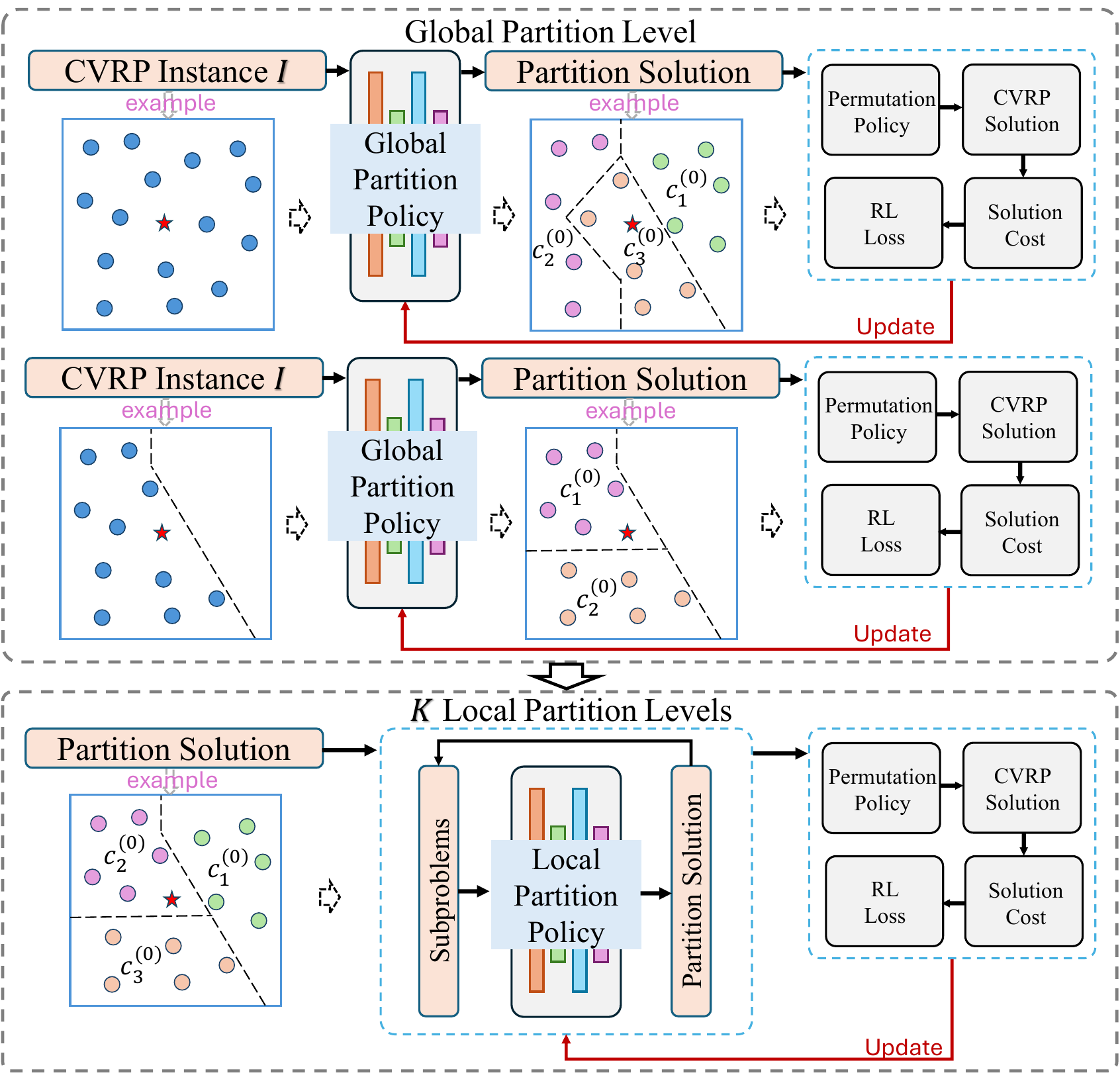}}
\hspace{0.1\textwidth}
\subfigure[SL-driven HLGP]{\label{fig:sl_hlgp}\includegraphics[width=0.4\textwidth]{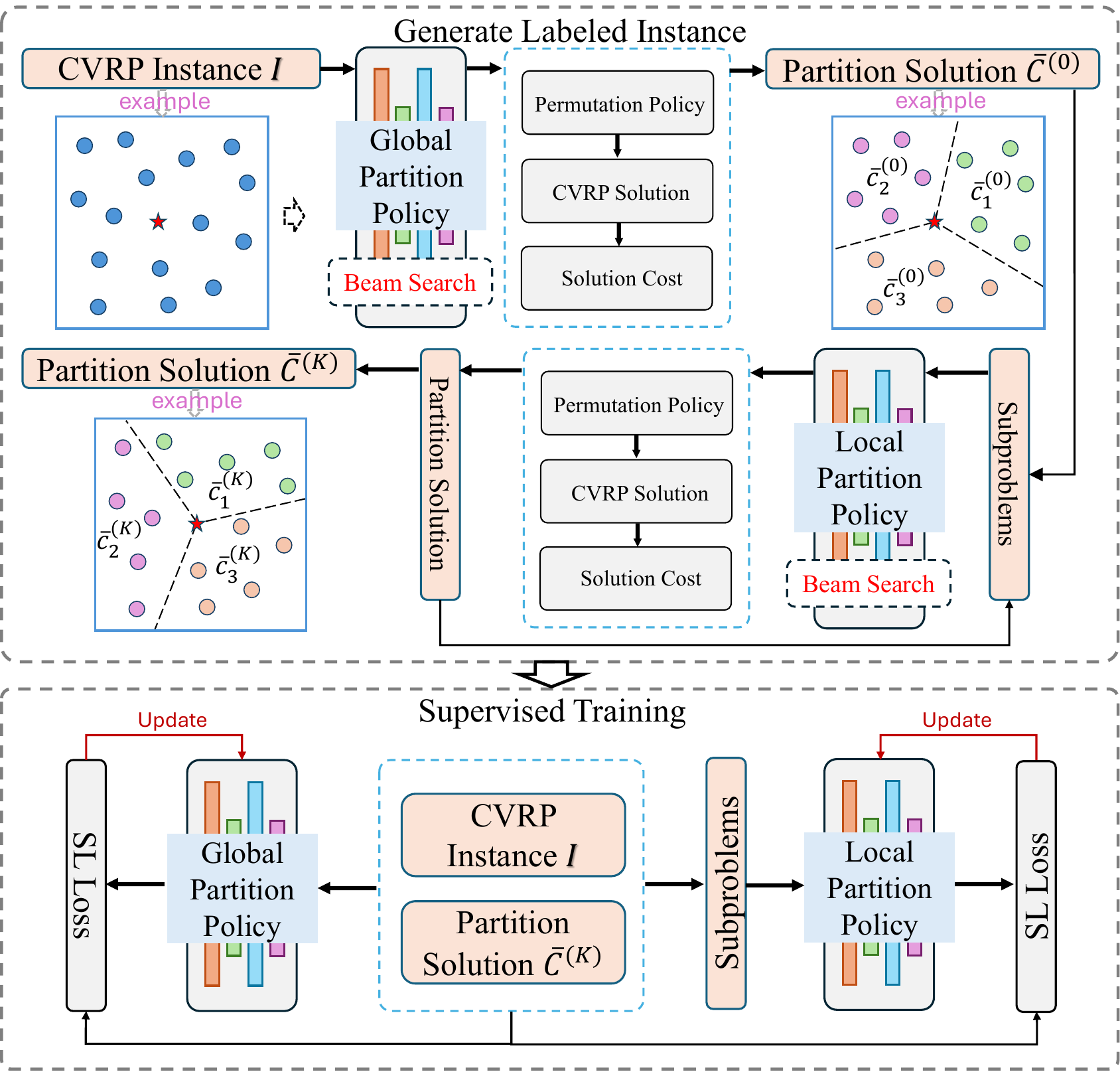}}
\caption{RL-driven HLGP replaces the initially generated partial partition solution with the complete partition solution of subproblems within $\mathcal{C}^{(0)}$ at level 0. SL-driven HLGP requires labeled instances for training $\pi_{\theta_{G}}$ and $\pi_{\theta_{L}}$.
\vspace{-10pt}
}
\label{fig:training_flow}
\end{figure*}

Notably, although Equation~\ref{equ:rl_obj} isolates the evaluation exclusively for $\omega^{(k)}$, the evaluation impacted by the trajectories $\omega^{(k+1)}, ..., \omega^{(K)}$ still remains. This implies that the underlying MDP at level $k$ remains nonstationary. 
We thus take the following optimization problem as an approximation:
\begin{equation}
\begin{split}
\label{equ:approx_rl_obj}
& \hat{J}(\theta_{G}, \theta_{L}) = L(\theta_{G}, \lambda_{G}, 0) + \sum_{k=1}^{K}L(\theta_{L}, \lambda_{L}, k); \\
& L(\theta, \lambda, k) = \mathbb{E}_{\omega^{(k)} \sim \hat{\mu}^{(k)}, \pi_{\theta}}[\sum_{t=1}^{T^{(k)}}r_{t}^{(k)}] + \lambda \mathcal{H}(\pi_{\theta}), 
\end{split}
\end{equation}
where $\hat{\mu}^{(k)}$ is a surrogate initial state distribution at level $k$, $\mathcal{H}(\pi_{\theta})$ is the entropy term, and $\lambda$ denotes the hyperparameter. The entropy term is typically defined to minimize the KL divergence between the policy and a uniform distribution. In Equation~\ref{equ:rl_obj}, $\omega^{(k)}, k\geq1$ is drawn from the initial distribution $\mu^{(k)}$ and the local partition policy $\pi_{\theta_{L}}$. However, $\mu^{(k)}$ heavily relies on preceding partition solutions derived from both the global and local partition policies. Therefore, in Equation~\ref{equ:approx_rl_obj}, the surrogate initial distribution $\hat{\mu}^{(k)}$ is introduced to eliminate this dependency. Please note that $\hat{\mu}^{(0)} = \mu^{(0)}$. As a result, the optimization for $\pi_{\theta_{G}}$ and $\pi_{\theta_{L}}$ is decoupled. 

In the context of RL-driven HLGP, we incorporate the surrogate objective defined in Equation~\ref{equ:approx_rl_obj} into the REINFORCE algorithm~\cite{williams1992simple} to update $\pi_{\theta_{G}}$ and $\pi_{\theta_{L}}$. In each iteration $n \geq 0$ of REINFORCE, the existing global partition policy denoted as $\pi_{\theta_{G}^{n}}$ is employed to sample $\omega^{(0)}$ for the update of $\pi_{\theta_{G}^{n}}$. At each local partition level $k \geq 1$, the current local partition policy denoted as $\pi_{\theta_{L}^{n}}$ is additionally leveraged to sample the partition solution $\mathcal{C}^{k-1}$, crucial for $\hat{\mu}^{(k)}$. Following this, $\omega^{(k)}$ is sampled to update $\pi_{\theta_{L}^{n}}$. Please refer to Appendix-B for the pseudocode.

Furthermore, in the standard theoretical analysis of REINFORCE algorithm conducted in~\cite{zhang2021sample}, the upper bound of regret includes the term represented by $||\frac{d}{\mu}||_{\infty}$, where $d$ and $\mu$ stand for the stationary state distribution and the initial state distribution. However, the existing method using REINFORCE algorithm for whatever types of problems (permutation or partition) in the context of CVRP ignores the potential risks highlighted in the regret bound. We exemplify the partition problem as a case study to elucidate this issue. The support set of $\mu$ consists solely of the problem instances $I$. Let $N_{v}(t)$ denote the number of customer nodes selected before timestep $t$. In contrast, during the partition process, at each step $t > 1$, the partition policy is indeed responsible to solve the subproblem denoted as $I_{N_{v}(t)}$ in which the graph comprises depot and unvisited customers. Let $c_{i_{t}}$ denote the subgraph under construction. The capacity in $I_{N_{v}(t)}$ is accordingly subtracted from the total demand of the visited node in $c_{i_{t}}$, reverting back to $D$ once $c_{i_{t}}$ is fully formed. The support set of $d$ thus includes the subproblems $I_{N_{v}(t)}$.  This significant discrepancy in support sets inevitably results in an infinite $||\frac{d}{\mu}||_{\infty}$ in the regret bound. This observation inspires us to incorporate the subproblems $I_{N_{v}(t)}$ encountered during the partition process into the training of involved partition policies to reduce the mismatch between support sets.

In the practical implementation, a problem instance $I$ is initially generated from the instance distribution $p_{I}$, which is used to train $\pi_{\theta_{G}}$ via RL. If a new subgraph $c_{i_{t}}$ is formed at timestep $t$, then the subproblem $I_{N_{v}(t)}$ is treated as an individual problem instance, denoted as $I \leftarrow I_{N_{v}(t)}$, for the training of $\pi_{\theta_{G}}$. This procedure continues until $G=\emptyset$ in $I$, and reverts back to $p_{I}$ for a new instance $I$. For efficiency reasons, we do not include all subproblems. In inference, the partition solution $\mathcal{C}^{(0)}$ is formed by sequentially replacing the partial partition solution with the corresponding complete partition solution of the subproblem. An example is shown in Figure~\ref{fig:rl_hlgp}. The training and inference procedure utilizing the encountered subproblem can similarly be applied to $\pi_{\theta_{L}}$. Additionally, we utilize the isomorphic Graph Neural Netwok (GNN) as presented in GLOP~\cite{ye2024glop} to serve as the backbones of $\pi_{\theta_{G}}$ and $\pi_{\theta_{L}}$ correspondingly.

\subsection{SL-driven HLGP}
In this section, we pivot towards an SL training strategy to optimize the objective of the partition problem as defined in Equation~\ref{equ:hl_obj}. Here, the optimal partition solver $\pi_{\mathrm{part}}^{\ast}$ is presumed to be available in advance. The optimal partition solution $\bar{\mathcal{C}}=\{\bar{c}_{1}, \ldots, \bar{c}_{N_{c}}\}$ for each instance $I$ is accordingly obtainable from $\pi_{\mathrm{part}}^{\ast}$. We thus adopt $f(\mathcal{C})=-\mathds{1}(\mathcal{C} = \bar{\mathcal{C}}) = -\mathds{1}(c_{1}=\bar{c}_{1}, \ldots, c_{N_{c}}=\bar{c}_{N_{c}})$, where $\mathds{1}(\cdot)$ denotes the indicator function. Recall that the optimal solution $\bar{\mathcal{T}}$ of the CVRP instance can be equivalently seen as the optimal partition solution $\bar{\mathcal{C}}$ when disregarding the node permutation information in $\bar{\mathcal{T}}$. Therefore, by setting $\bar{\mathcal{C}} = \bar{\mathcal{T}}$, the feasible cost function can be defined as $f(\mathcal{C}) = -\mathds{1}(\mathcal{C}[0] = \bar{\mathcal{C}}[0], \ldots, \mathcal{C}[N_{\mathrm{sol}}] = \bar{\mathcal{C}}[N_{\mathrm{sol}}])$. Upon utilizing this design of $f(\mathcal{C})$, Theorem~\ref{thm:sl_obj} simplifies the optimization objective of HLGP framework.
\begin{theorem}
\label{thm:sl_obj}
Given $f(\mathcal{C})=-\mathds{1}(\mathcal{C} = \bar{\mathcal{C}})$, the optimization objective in the HLGP framework for a problem instance $I$ is to identify $\pi_{\theta_{G}}$ and $\pi_{\theta_{L}}$ so as to minimize:
\begin{equation}
    L(\theta_{G}, \theta_{L}, \bar{\mathcal{C}}) = 
    - \log\pi_{\theta_{G}}(\bar{\mathcal{C}}|I) - \sum_{i=1}^{N_{c}} \log\pi_{\theta_{L}}(\bar{\mathcal{C}}_{i}|I_{i}),
\end{equation}
where $I_{i} = (G_{i}, D, 2)$ denotes the subproblem, with $G_{i} = \bar{c}_{i} \cup \bar{c}_{i\%N_{c}+1}$, and $\bar{\mathcal{C}}_{i} = \{\bar{c}_{i}, \bar{c}_{i\%N_{c}+1} \}$ represents the corresponding label.
\end{theorem}

\begin{table*}[t]
    \caption{Comparative results on uniformly distributed CVRP instances. OOM stands for out-of-memory. The symbol $\ast$ indicates that the results are obtained from the original paper. The notation $\downarrow$ indicates that a lower value is better.}
    \setlength\tabcolsep{0.9pt}
    \renewcommand{\arraystretch}{0.9}
    \small
    \centering
    \begin{tabular}{lccccccccccccccc}
    \toprule[1.2pt]
       \multirow{2}{*}{Methods}  & \multicolumn{3}{c}{CVRP1K} & \multicolumn{3}{c}{CVRP2K} & \multicolumn{3}{c}{CVRP5K} & \multicolumn{3}{c}{CVRP7K} & \multicolumn{3}{c}{CVRP10K} \\
    \cmidrule(r){2-4} \cmidrule(r){5-7} \cmidrule(r){8-10} \cmidrule(r){11-13} \cmidrule(r){14-16}
        & \normalsize{$Avg._{\downarrow}$} & \normalsize{$Std._{\downarrow}$} & \normalsize{$Time_{\downarrow}$} & \normalsize{$Avg._{\downarrow}$} & \normalsize{$Std._{\downarrow}$} & \normalsize{$Time_{\downarrow}$} & \normalsize{$Avg._{\downarrow}$} & \normalsize{$Std._{\downarrow}$} & \normalsize{$Time_{\downarrow}$} & \normalsize{$Avg._{\downarrow}$} & \normalsize{$Std._{\downarrow}$} & \normalsize{$Time_{\downarrow}$} & \normalsize{$Avg._{\downarrow}$} & \normalsize{$Std._{\downarrow}$} & \normalsize{$Time_{\downarrow}$} \\
    \midrule[0.7pt]
    LKH & 42.17 & 0.80 & 14.18s & 58.06 & 1.06 & 31.92s	& 126.59 & 2.81 & 6.80m & 172.80 & 4.04 & 18.21m & 240.23 & 5.42 & 41.29m \\
    HGS & 41.12 & 0.77 & 4.57m & 56.24 & 1.07 & 12.68m & 122.84 & 2.87 & 18.80m & 165.37 & 3.95 & 20.15m & 226.59 & 5.29 &	24.64m \\
    \midrule[0.7pt]
    AM (ICLR'19) & 59.18 & 2.81 &	8.84s &\multicolumn{3}{c}{OOM}	&\multicolumn{3}{c}{OOM}  &\multicolumn{3}{c}{OOM}	& \multicolumn{3}{c}{OOM} \\
    POMO (NeurIPS'20) & 100.99 & 7.43 & 4.63s & 255.13 &30.02 &39.35s &\multicolumn{3}{c}{OOM} &\multicolumn{3}{c}{OOM}	&\multicolumn{3}{c}{OOM} \\	
    Sym-POMO (NeurIPS'22)& 98.04 & 2.86 & 5.71s & 157.89 & 2.96 &45.23s	&\multicolumn{3}{c}{OOM} &\multicolumn{3}{c}{OOM} &\multicolumn{3}{c}{OOM} \\
    AMDKD (NeurIPS'22)& 84.16 & 0.98 & 4.27s & 188.75 & 4.39 & 34.39s &\multicolumn{3}{c}{OOM} &\multicolumn{3}{c}{OOM} &\multicolumn{3}{c}{OOM} \\
    Omni-POMO (ICML'23)& 47.80 &0.77	& 4.45s	& 74.01 & 1.05 &35.86s	&\multicolumn{3}{c}{OOM} &\multicolumn{3}{c}{OOM} &\multicolumn{3}{c}{OOM} \\
    ELG-POMO (IJCAI'24)& 46.41 & 0.40 & 9.53s & 66.07 & 0.55 & 67.19s &\multicolumn{3}{c}{OOM} &\multicolumn{3}{c}{OOM} &\multicolumn{3}{c}{OOM} \\
    INViT (ICML'24)& 46.56 & 0.81 & 17.08s & 64.94 & 1.09 & 36.67s & 139.45 & 2.86 & 141.07s & 186.57 & 3.93 & 4.81m & 254.17 & 5.39 & 6.96m \\
    \midrule[0.7pt]
    LEHD (NeurIPS'23) & 42.80 & 0.82 & 40.25s & 60.48 & 1.12 & 72.48s & 136.80 &2.86 & 3.22m & 188.11 & 4.00 & 6.52m & 266.06 & 5.56 & 11.92m \\		
    BQ (NeurIPS'23)& 43.12 & 0.80 & 4.75s & 60.95 & 1.10 &15.66s & 137.14 & 3.00 & 79.89s & 188.78 & 4.23 & 1.88m & 265.81 & 5.97 & 3.30m \\
    \midrule[0.7pt]
    L2I (ICLR'20)& 49.79 & 1.10 & 18.60s & 95.58 & 5.44 & 44.88s & 262.84 & 9.99 & 2.64m & 506.73 & 25.25	& 3.61m & 1263.23  & 4.00 & 4.07m \\
    NLNS (ECAI'20) & 53.51 & 0.83 & 12.08s & 81.54 & 1.12 & 12.15s & 180.84 & 2.99 & 12.62s & 243.50 &4.17 &13.16s & 331.27 & 5.53 & 13.97s	\\
    DACT (NeurIPS'21) & 50.43 & 0.35	& 75.47s & 71.17 & 0.51	& 5.40m &\multicolumn{3}{c}{OOM} &\multicolumn{3}{c}{OOM} &\multicolumn{3}{c}{OOM}\\
    \midrule[0.7pt]
    L2D (NeurIPS'21) & 46.45 & 0.87 & 4.67s & 64.04 & 1.21 & 7.54s & 135.09 & 3.02 & 16.11s & 182.21 & 4.13 & 24.37s & 246.38 & 5.55 & 27.59s \\
    RBG$^{\ast}$ (KDD'22) & 74.00 & -	&13.00s	& 137.60 &-& 42.00s	&\multicolumn{3}{c}{-}&\multicolumn{3}{c}{-}&\multicolumn{3}{c}{-} \\
    TAM-AM$^{\ast}$ (ICLR'23) & 50.06 & 0.98	& 0.76s	& 74.31 & 1.42	&2.20s &172.22&- &11.78s	&233.44&- &26.47s &\multicolumn{3}{c}{-}	\\
    TAM-LKH$^{\ast}$ (ICLR'23) & 46.34 & 0.84	& 1.82s	& 64.78 & 1.18	&5.63s	&144.64&- &17.19s &196.91&- &33.21s &\multicolumn{3}{c}{-} \\
    GLOP-G (AAAI'24) &47.21&0.90	&0.73s	&63.60&1.41	&1.74s	&141.67&3.67	&2.37s	&191.75&4.99	&3.50s	&266.96&6.46 &13.74s \\
    GLOP-LKH (AAAI'24) &46.03&0.99	&0.78s	&63.10&1.43	&1.83s	&140.51&3.72	&4.31s	&191.45&5.06	&7.34s	&267.50&6.50 &16.47s \\
    \midrule[0.7pt]
    RL-driven HLGP &43.78&0.85	&3.72s &59.58&1.17 &10.03s	&128.12 &3.06 &82.59s & 173.71 &4.39 & 1.96m & 238.62 & 6.03 & 5.13m \\
    SL-driven HLGP &\textbf{41.95}&0.78	&8.31s &\textbf{57.67}&1.08 &32.40s	&\textbf{124.13}&2.79 &97.27s &\textbf{166.73}&3.91 & 2.15m & \textbf{227.07} & 5.25 & 3.39m \\
    \bottomrule[1.2pt]
    \end{tabular}
    \label{tab:1}
    \vspace{-10pt}
\end{table*}

Please see Appendix-C.3 for proofs. We can observe that the optimization objective for $\pi_{\theta_{G}}$ and $\pi_{\theta_{L}}$ is totally disentangled in Theorem~\ref{thm:sl_obj}. Accordingly, the optimization objective for instances sampled from $p_{I}$ is to minimize:
\begin{equation}
\label{equ:sl_obj}
    J(\theta_{G}, \theta_{L})= \mathbb{E}_{(I, \bar{\mathcal{C}})\sim p_{I}, \pi_{\mathrm{part}}^{\ast}} [L(\theta_{G}, \theta_{L}, \bar{\mathcal{C}})].
\end{equation}
This objective involves evaluating the performance of $\pi_{\theta_{G}}$ and $\pi_{\theta_{L}}$ on the sampled trajectories induced by $\pi_{\mathrm{part}}^{\ast}$. However, $\pi_{\mathrm{part}}^{\ast}$ is practically unavailable for the NP-hard partition problem, since it is impossible to directly get supervised labels. Inspired by recent techniques known as self-imitation learning~\cite{son2023meta, luo2024self}, our goal is to acquire high-quality labeled instances from a behavioral policy $\hat{\pi}_{\mathrm{part}}$. The working pipeline of $\hat{\pi}_{\mathrm{part}}$ can be described as follows (see Figure~\ref{fig:sl_hlgp}): At the global partition level, $\hat{\mathcal{C}}^{(0)}$ is first generated using beam search with $\pi_{\theta_{G}}$. Then, at each local partition level $k\geq1$, the partition solution is further refined locally using beam search with $\pi_{\theta_{L}}$ to obtain the ultimate partition solution $\hat{\mathcal{C}}^{(K)}$. Thus, during training, $\bar{\mathcal{C}}=\hat{\mathcal{C}}^{(K)}$ serves as the label. The practical loss function is thus defined as:
\begin{equation}
\label{equ:sl_loss}
\hat{J}(\theta_{G}, \theta_{L}) = \mathbb{E}_{(I, \bar{\mathcal{C}})\sim p_{I}, \hat{\pi}_{\mathrm{part}}} [L(\theta_{G}, \theta_{L}, \bar{\mathcal{C}})]  + \mathrm{reg}(\theta_{G}, \theta_{L}),
\end{equation}
where $\mathrm{reg}(\theta_{G}, \theta_{L}) = \lambda_{G}\frac{||\theta_{G}||^{2}}{2} + \lambda_{L}\frac{||\theta_{L}||^{2}}{2}$, with hyperparameters $\lambda_{G}$ and $\lambda_{L}$. Therefore, this loss function is incorporated into the imitation learning algorithm to iteratively optimize $\pi_{\theta_{G}}$ and $\pi_{\theta_{L}}$. In each iteration $n \geq 0$, the algorithm deploys $\hat{\pi}_{\mathrm{part}}^{n}$ (which relies on $\theta_{G}^{n}$ and $\theta_{L}^{n}$) and gathers the labeled instance $(I, \bar{\mathcal{C}})$. Online gradient updates are then executed based on estimated gradients to update $\theta_{G}^{n}$ and $\theta_{L}^{n}$. Please refer to Appendix-B for the pseudocode.

Here, let us delve deeper into illustrating the training process for the global partition policy $\pi_{\theta_{G}}$ as a case study to elucidate the intricacies of the SL algorithm for the partition problem. A similar analysis can be conducted for the local partition policy $\pi_{\theta_{L}}$. The global partition policy $\pi_{\theta_{G}}$ requires to imitate the whole trajectory induced by the behavioral policy $\hat{\pi}_{\mathrm{part}}$. Following the formulation in Equation~\ref{equ:glb_part_pi}, the global partition policy can directly output the node sequence. Subsequently, the log-probability of this node sequence in $\hat{\mathcal{C}}$ is maximized to update $\theta_{G}$. This log-probability of the node sequence contains the product of a series of conditional probabilities, represented as $\log \prod_{t=1}^{N_{\mathrm{sol}}} \pi_{\theta_{G}}(\bar{\mathcal{C}}[t] | \bar{\mathcal{C}}[0:t-1], I)=\sum_{t=1}^{N_{\mathrm{sol}}} \log \pi_{\theta_{G}}(\bar{\mathcal{C}}[t] | \bar{\mathcal{C}}[0:t-1], I)$. This sum of log-probabilities prompts us to consider $(\bar{\mathcal{C}}[0:t-1], I)$ as an individual training sample, with its corresponding label being a singular $\bar{\mathcal{C}}[t]$. In this context, at timestep $t \geq 1$, the global partition policy is addressing a subproblem $I_{N_{v}(t)}$ determined by $(\bar{\mathcal{C}}[0:t-1], I)$, and it aligns with the same definition as in the RL setting. 
Hence, rather than generating the entire node sequence for behavioral imitation, the labeled instance $(I_{N_{v}(t)}, \bar{\mathcal{C}}[t])$ is fed to the global partition policy to imitate one-step behavior at each time step $t \geq 1$. In practice, we employ a variant Transformer model of BQ~\cite{drakulic2024bq} as the backbones of $\pi_{\theta_{G}}$ and $\pi_{\theta_{L}}$, which aligns with the analysis above. Therefore, we underscore the importance of viewing the subproblems encountered during training as individual training instances within both the contexts of RL and SL training paradigms.

\begin{table*}[]
\caption{Comparative results on various distributed CVRP instances. ``G" denotes the Gaussian distribution. ``E" denotes the Explosion distribution. ``R" denotes the Rotation distribution. The notation $\downarrow$ indicates that a lower value is better.}
\setlength\tabcolsep{0.9pt}
\renewcommand{\arraystretch}{0.9}
\small
\centering
\begin{tabular}{lcccccccccccccccccc}
\toprule[1.2pt]
\multirow{2}{*}{Methods}  & \multicolumn{3}{c}{CVRP1K+G} & \multicolumn{3}{c}{CVRP1K+E} & \multicolumn{3}{c}{CVRP1K+R} & \multicolumn{3}{c}{CVRP2K+G} & \multicolumn{3}{c}{CVRP5K+E} &\multicolumn{3}{c}{CVRP7K+R} \\
\cmidrule(r){2-4} \cmidrule(r){5-7} \cmidrule(r){8-10} \cmidrule(r){11-13} \cmidrule(r){14-16} \cmidrule(r){17-19}
& \normalsize{$Avg._{\downarrow}$} & \normalsize{$Std._{\downarrow}$} & \normalsize{$Time_{\downarrow}$} & \normalsize{$Avg._{\downarrow}$} & \normalsize{$Std._{\downarrow}$} & \normalsize{$Time_{\downarrow}$} & \normalsize{$Avg._{\downarrow}$} & \normalsize{$Std._{\downarrow}$} & \normalsize{$Time_{\downarrow}$} & \normalsize{$Avg._{\downarrow}$} & \normalsize{$Std._{\downarrow}$} & \normalsize{$Time_{\downarrow}$} & \normalsize{$Avg._{\downarrow}$} & \normalsize{$Std._{\downarrow}$} & \normalsize{$Time_{\downarrow}$} & \normalsize{$Avg._{\downarrow}$} & \normalsize{$Std._{\downarrow}$} & \normalsize{$Time_{\downarrow}$} \\
\midrule[0.7pt]
LKH3            & 32.52           & 1.21   & 37.35s & 38.01            & 1.48  & 15.55s  & 37.50           & 1.33  & 15.35s  & 42.60           & 1.62  & 64.06s  & 103.45           & 4.39   & 6.67m  & 156.04           & 6.81  & 25.69m \\
HGS             & 31.84           & 1.19   & 15.57m & 37.13            & 1.46  & 6.52m   & 36.62           & 1.32  & 7.78m   & 41.64           & 1.61  & 19.80m  & 101.16           & 4.40   & 16.53m & 151.04           & 6.72  & 21.04m \\
\midrule[0.7pt]
AM (ICLR'19) & 93.62           & 20.23  & 9.32s  & 58.99            & 4.79  & 8.74s   & 60.80           & 5.42  & 8.72s   & \multicolumn{3}{c}{OOM}           & \multicolumn{3}{c}{OOM}            & \multicolumn{3}{c}{OOM}           \\
POMO (NeurIPS'20) & 56.83           & 2.72   & 4.78s  & 74.88            & 4.84  & 4.54s   & 75.26           & 5.52  & 4.41s   & 101.75          & 7.32  & 38.31s  & \multicolumn{3}{c}{OOM}            & \multicolumn{3}{c}{OOM}           \\
Sym-POMO (NeurIPS'22) & 97.59           & 5.35   & 5.57s  & 95.08            & 5.11  & 5.65s   & 106.88          & 6.11  & 5.53s   & 134.32          & 5.20  & 40.56s  & \multicolumn{3}{c}{OOM}            & \multicolumn{3}{c}{OOM}           \\
AMDKD (NeurIPS'22) & 58.71           & 1.98   & 4.14s  & 71.10            & 2.04  & 4.17s   & 73.32           & 2.00  & 4.11s   & 108.11          & 3.82  & 32.96s  & \multicolumn{3}{c}{OOM}            & \multicolumn{3}{c}{OOM}           \\
Omni-POMO (ICML'23) & 35.47           & 1.20   & 4.30s  & 41.80            & 1.47  & 4.30s   & 41.30           & 1.34  & 4.28s   & 51.02           & 1.76  & 35.31s  & \multicolumn{3}{c}{OOM}            & \multicolumn{3}{c}{OOM}           \\
ELG-POMO (IJCAI'24) & 36.49           & 0.63   & 9.86s  & 41.64            & 0.75  & 9.67s   & 41.31           & 0.69  & 9.48s   & 49.34           & 0.86  & 68.73s  & \multicolumn{3}{c}{OOM}            & \multicolumn{3}{c}{OOM}           \\
INViT (ICML'24) & 35.67           & 1.28   & 19.80s  & 42.11            & 1.53  & 19.80s   & 41.22           & 1.37  & 19.74s  & 47.31           & 1.75  & 46.92s  & 113.26           & 4.79   & 3.17m  & 169.38           & 7.47  & 4.66m  \\
\midrule[0.7pt]
LEHD   (NeurIPS'23)         & 33.99           & 1.23   & 36.27s & 38.96            & 1.50  & 36.19s  & 38.44           & 1.36  & 36.16s  & 47.48           & 1.65  & 64.80s   & 116.70           & 4.38   & 2.85m  & 176.14           & 6.91  & 5.93m  \\
BQ   (NeurIPS'23)           & 34.71           & 1.25   & 3.88s  & 39.64            & 1.50  & 3.90s   & 39.17           & 1.39  & 3.91s   & 47.74           & 1.67  & 11.34s  & 120.23           & 4.75   & 72.12s & 181.04           & 7.59  & 76.47s \\
\midrule[0.7pt]
L2I    (ICLR'20)         & 37.42           & 1.28   & 13.56s & 44.05            & 1.58  & 14.15s  & 43.56           & 1.41  & 14.01s  & 63.33           & 3.26  & 26.14s  & 204.51           & 10.31  & 2.10m  & 348.70           & 17.47 & 4.37m  \\
NLNS    (ECAI'20)        & 41.31           & 1.27   & 12.15s & 46.52            & 1.51  & 12.15s  & 47.44           & 1.36  & 12.16s  & 60.38           & 1.89  & 12.22s  & 142.87           & 4.65   & 12.73s & 221.69           & 6.51  & 13.02s \\
DACT  (NeurIPS'21)   & 37.03           & 0.57   & 68.63s & 43.10            & 0.66  & 67.92s  & 43.50           & 0.57  & 67.98s  & 49.30           & 0.83  & 4.52m   & \multicolumn{3}{c}{OOM}            & \multicolumn{3}{c}{OOM}           \\
\midrule[0.7pt]
L2D    (NeurIPS'21)         & 35.26           & 1.24   & 2.60s  & 41.09            & 1.50  & 2.52s   & 40.40           & 1.37  & 2.63s   & 46.29           & 1.69  & 4.24s   & 108.95           & 4.63   & 10.15s & 162.90           & 6.99  & 19.04s \\
GLOP-G    (AAAI'24)      & 39.20            & 1.40   & 0.44s  & 43.44            & 1.63  & 0.43s   & 43.46           & 1.46  & 0.41s   & 50.55           & 1.97  & 1.88s   & 117.65           & 4.80   & 7.07s  & 178.37           & 6.84  & 7.98s  \\
GLOP-LKH   (AAAI'24)   & 38.71           & 1.42   & 1.22s  & 42.83            & 1.67  & 0.93s   & 42.80           & 1.49  & 0.77s   & 50.42           & 1.98  & 3.84s   & 117.04           & 4.83   & 9.85s  & 178.08           & 6.90  & 11.35s \\
\midrule[0.7pt]
RL-driven HLGP         & 34.58           & 1.26   & 3.47s  & 39.85            & 1.54  & 3.43s   & 39.36           & 1.38  & 3.47s   & 44.80           & 1.70  & 10.37s  & 106.27           & 4.52   & 70.80s & 160.73           & 7.20  & 3.04m  \\
SL-driven HLGP  & \textbf{32.55}  & 1.21   & 7.21s  & \textbf{37.96}   & 1.48  & 7.55s   & \textbf{37.40}  & 1.34  & 7.41s   & \textbf{42.85}  & 1.65  & 30.61s  & \textbf{102.27}  & 4.47   & 84.12s & \textbf{152.47}  & 6.91  & 98.49s \\
\bottomrule[1.2pt]
\end{tabular}
\label{tab:2}
\vspace{-10pt}
\end{table*}

\section{Experiments}

\subsection{Training and Evaluation Settings}

In the training for both RL-driven and SL-driven HLGP, we adhere to the problem settings used in GLOP~\cite{ye2024glop}. During the training phase, each CVRP instance consists of 1000 customer nodes distributed uniformly, with a vehicle capacity of 200. During the evaluation phase, our focus is on assessing the generalization capabilities of various models. Therefore, we utilize diverse datasets with varying scales and distributions. Each evaluation dataset can specify the number of customer nodes as 1000, 2000, 5000, 7000, or 10000. The customer nodes in each dataset are distributed according to a Uniform distribution, a Gaussian distribution, an explosion pattern, or a rotation pattern. Except for the dataset with 1000 customer nodes setting capacity as 200, the capacity for the other datasets is set to 300. Each dataset comprises 128 instances. The process of generating problem instances aligns with the methodologies outlined in~\cite{kool2018attention, zhou2023towards}. Please refer to Appendix-A.1 for more training details. Note that the code of our implementation, along with the Appendix, has been uploaded as the supplementary material.

During the evaluation phase, we compare our proposed RL-driven and SL-driven HLGP models with various methods. The classical heuristic methods include LKH3~\cite{helsgaun2017extension} and HGS~\cite{vidal2012hybrid}. In learning-based constructive methods, AM~\cite{kool2018attention}, POMO~\cite{kwon2020pomo}, and Sym-POMO~\cite{kim2022sym} serve as baselines trained with RL. AMDKD~\cite{bi2022learning} and Omni-POMO~\cite{zhou2023towards} target generalization issues specifically. ELG-POMO~\cite{gao2023towards} and INViT~\cite{fang2024invit} harness local topological features. Within the realm of iterative methods, L2I~\cite{lu2019learning}, NLNS~\cite{hottung2020neural}, and DACT~\cite{ma2021learning} integrate RL-based policies with local operators to iteratively refine a given solution. In the context of the divide-and-conquer paradigm, RBG~\cite{zong2022rbg} and L2D~\cite{li2021learning} employ heuristic rules for the partition policy, while GPLC paradigms TAM~\cite{hou2023generalize} and GLOP~\cite{ye2024glop} utilize neural partition policies. We adhere to the official implementations of these methods and the instructions provided by other works~\cite{hou2023generalize, ye2024glop, fang2024invit} that cite these methods to replicate the experimental results. For RBG and TAM, we directly use the reported results from the papers. For a fair comparison, we only consider baseline methods that either have official code available for reproduction or have been extensively reported in previous papers. Further details are deferred to the Appendix-A.6. In addition, for comparison, the metrics include the average solution costs ($Avg.$), the standard deviation of solution costs ($Std.$), and average inference time ($Time$).

\subsection{Performance Comparisons}
In Table~\ref{tab:1}, our proposed RL-driven and SL-driven HLGP are compared with various previous methods on cross-size datasets. When compared to the state-of-the-art method, LEHD, RL-driven HLGP shows only a slight performance drop, notably on the CVRP1K dataset. Across other cross-size datasets, RL-driven HLGP consistently delivers superior solutions and is notably more efficient than LEHD. In comparison to all other learning-based baselines, SL-driven HLGP consistently demonstrates its superiority in terms of average cost while maintaining efficiency. Moreover, compared to classical heuristic solvers, SL-driven HLGP can swiftly produce high-quality solutions. In the most challenging case, CVRP10K, SL-driven HLGP stands out as the only method capable of generating high-quality solutions within 4 minutes for each instance. Additionally, owing to the adopted HL framework to mitigate compounded errors, both RL-driven HLGP and SL-driven HLGP outperform their respective baselines (GLOP and BQ).

Table~\ref{tab:2} displays the comparison of our proposed methods with various existing methods on both cross-distribution datasets and cross-size and distribution datasets. When compared to BQ and LEHD, our RL-driven HLGP exhibits a minor performance decline on cross-distribution datasets. However, on the cross-size and distribution datasets, RL-driven HLGP consistently showcases superior generalization by efficiently producing improved solutions. In comparison to all previous learning-based methods, SL-driven HLGP consistently outperforms on both cross-size and cross-size and distribution datasets. Moreover, the performance of SL-driven HLGP closely approaches that of HGS while being dramatically more efficient. This justifies the use of supervised partition policy especially for larger instances. Please refer to the Appendix-A.3 and Appendix-A.4 for the hyperparameter studies, ablation studies, and visualization results.

\section{Conclusion}
In this work, we propose a novel hierarchical learning-based framework for the graph partition in CVRP. The global partition policy and local partition policy synergistically tackle the partition task to progressively alleviate the compounded misclusterings. Our methods adopt a unified objective function harmoniously compatible with both RL and SL training methods. Meanwhile, we thoroughly analyze the significance of treating the subproblems encountered during training as individual instances in both the RL and SL settings. Experimental results unequivocally demonstrate the generalization capability of proposed HLGP framework in finding low-cost CVRP solutions under distribution and scale shifts. 
In future, we plan to extend our HLGP to more different VRP variants like CVRPTW and Min-max CVRP, as well as other types of COPs.



\begin{acks}
This research is supported by a generous research grant from Xiaoi Robot Technology Limited, the National Research Foundation, Singapore under its AI Singapore Programme (AISG Award No: AISG3-RP-2022-031), and the Singapore Ministry of Education (MOE) Academic Research Fund (AcRF) Tier 1 grant.
\end{acks}



\bibliographystyle{ACM-Reference-Format} 
\bibliography{bibfile}


\newpage
\renewcommand{\algorithmicrequire}{\textbf{Input:}}
\renewcommand{\algorithmicensure}{\textbf{Output:}}

\setcounter{section}{0}
\setcounter{equation}{0}
\renewcommand\thesection{\Alph{section}} 

\section{Experiments}

In this section, we present the specifics of both training and evaluation for our proposed HLGP methods, including the problem settings, hyperparameter configurations, and practical training protocols. Subsequently, we conduct both a hyperparameter study and an ablation study to scrutinize the influence of key hyperparameters and modules in the HLGP framework on overall performance. Additionally, we include visualization results for intuitive comparisons. Furthermore, detailed implementations of baseline methods are provided.

\subsection{Training Details}

During the training phase for both RL-driven and SL-driven HLGP, we strictly adhere to the standard problem settings used in GLOP~\cite{ye2024glop}. In these settings, each CVRP instance consists of 1000 nodes uniformly distributed within a unit square area. The vehicle capacity, denoted as $D$, is fixed at 200. The maximum permissible number of times the vehicle returns to the depot is accordingly calculated as $\lceil \sum_{i=1}^{N_{v}} d_{i} / D \rceil + 1$, where $N_{v}$ represents the number of customer nodes and $d_{i}$ signifies the demands associated with each customer node. The demand $d_{i}$ is uniformly sampled from the range $\{1, \ldots, 9 \}$.

The RL-driven HLGP adopts the same isomorphic GNN architecture~\cite{joshi2019efficient, qiu2022dimes} as employed in GLOP~\cite{ye2024glop} for both the global and local partition policies to generate partition heatmaps. These heatmaps, specific to each CVRP instance, aid in progressively decoding feasible partition solutions that adhere to the predefined CVRP constraints. Subsequently, the same pretrained local permutation policy, as utilized in GLOP~\cite{ye2024glop}, is applied to resolve the subproblems derived from these feasible partition solutions. The resulting subsolutions to these subproblems constitute the overall CVRP solution. Consequently, the cost of the CVRP solution corresponds to the evaluation of the feasible partition solution.

The training phase of the RL-driven HLGP comprises 20 epochs, with each epoch consisting of 256 iterations. During each iteration, 5 CVRP instances are randomly generated to train both the global and local partition policies. For detailed training procedures, please refer to the Algorithm Pseudocodes. Consistent with the approach in GLOP~\cite{ye2024glop}, 20 solutions are simultaneously generated for both global and local partition policies to calculate the baseline, thereby reducing gradient variance. The Adam Optimizer is employed with an initial learning rate of $3 \times 10^{-4}$ for the involved partition policies. Additionally, a Cosine Annealing scheduler is utilized to gradually decrease the learning rate as the training epoch progresses. Similarly, we impose a constraint to maintain the maximum L2 norm of gradients below 1. The RL-driven HLGP is trained on a NVIDIA RTX 3090 GPU and an INTEL(R) XEON(R) GOLD 5218R CPU@2.10GHz.

\begin{table*}[hbt]
    \caption{Hyperparameter study for the number of levels in RL-driven HLGP (denoted as $K$). ``G" denotes the Gaussian distribution. ``U" denotes the Uniform distribution.}
    \setlength\tabcolsep{0.9pt}
    \renewcommand{\arraystretch}{0.9}
    \small
    \centering
    \begin{tabular}{cccccccccccccccc}
    \toprule[1.2pt]
       \multirow{2}{*}{RL-driven HLGP}  & \multicolumn{3}{c}{$K=1$} & \multicolumn{3}{c}{$K=2$} & \multicolumn{3}{c}{$K=3$} & \multicolumn{3}{c}{$K=4$} & \multicolumn{3}{c}{$K=5$} \\
    \cmidrule(r){2-4} \cmidrule(r){5-7} \cmidrule(r){8-10} \cmidrule(r){11-13} \cmidrule(r){14-16}
        & \normalsize{$Avg.$} & \normalsize{$Std.$} & \normalsize{$Time$} & \normalsize{$Avg.$} & \normalsize{$Std.$} & \normalsize{$Time$} & \normalsize{$Avg.$} & \normalsize{$Std.$} & \normalsize{$Time$} & \normalsize{$Avg.$} & \normalsize{$Std.$} & \normalsize{$Time$} & \normalsize{$Avg.$} & \normalsize{$Std.$} & \normalsize{$Time$} \\
    \midrule[0.7pt]
    CVRP1K+G & 35.13 & 1.26 & 1.18s & 34.84 & 1.26 & 1.78s	& 34.69 & 1.25 & 2.37s & 34.63 & 1.25 & 2.91s & 34.58 & 1.26 & 3.47s \\
    CVRP2K+U & 60.32 & 1.18 & 3.94s & 59.85 & 1.18 & 5.50s & 59.70 & 1.17 & 7.07s & 59.61 & 1.16 & 8.51s & 59.58 & 1.17 & 10.03s \\
    \midrule[0.7pt]
    \multirow{2}{*}{RL-driven HLGP}  & \multicolumn{3}{c}{$K=6$} & \multicolumn{3}{c}{$K=7$} & \multicolumn{3}{c}{$K=8$} & \multicolumn{3}{c}{$K=9$} & \multicolumn{3}{c}{$K=10$} \\
    \cmidrule(r){2-4} \cmidrule(r){5-7} \cmidrule(r){8-10} \cmidrule(r){11-13} \cmidrule(r){14-16}
        & \normalsize{$Avg.$} & \normalsize{$Std.$} & \normalsize{$Time$} & \normalsize{$Avg.$} & \normalsize{$Std.$} & \normalsize{$Time$} & \normalsize{$Avg.$} & \normalsize{$Std.$} & \normalsize{$Time$} & \normalsize{$Avg.$} & \normalsize{$Std.$} & \normalsize{$Time$} & \normalsize{$Avg.$} & \normalsize{$Std.$} & \normalsize{$Time$} \\
    \midrule[0.7pt]
    CVRP1K+G & 34.56 & 1.25 & 4.10s & 34.51 & 1.24 & 4.57s	& 34.51 & 1.24 & 5.11s & 34.51 & 1.24 & 5.77s & 34.51 & 1.24 & 6.31s \\
    CVRP2K+U & 59.56 & 1.17 & 11.40s & 59.55 & 1.17 & 13.00s & 59.50 & 1.17 & 14.70s & 59.50 & 1.17 & 16.39s & 59.50 & 1.17 & 17.66s \\
    \bottomrule[1.2pt]
    \end{tabular}
    \label{tab:rl_K}
\end{table*}

\begin{table*}[hbt]
    \caption{Hyperparameter study for the number of levels in SL-driven HLGP (denoted as $K$). ``G" denotes the Gaussian distribution. ``U" denotes the Uniform distribution.}
    \setlength\tabcolsep{0.9pt}
    \renewcommand{\arraystretch}{0.9}
    \small
    \centering
    \begin{tabular}{cccccccccccccccc}
    \toprule[1.2pt]
       \multirow{2}{*}{SL-driven HLGP}  & \multicolumn{3}{c}{$K=1$} & \multicolumn{3}{c}{$K=2$} & \multicolumn{3}{c}{$K=3$} & \multicolumn{3}{c}{$K=4$} & \multicolumn{3}{c}{$K=5$} \\
    \cmidrule(r){2-4} \cmidrule(r){5-7} \cmidrule(r){8-10} \cmidrule(r){11-13} \cmidrule(r){14-16}
        & \normalsize{$Avg.$} & \normalsize{$Std.$} & \normalsize{$Time$} & \normalsize{$Avg.$} & \normalsize{$Std.$} & \normalsize{$Time$} & \normalsize{$Avg.$} & \normalsize{$Std.$} & \normalsize{$Time$} & \normalsize{$Avg.$} & \normalsize{$Std.$} & \normalsize{$Time$} & \normalsize{$Avg.$} & \normalsize{$Std.$} & \normalsize{$Time$} \\
    \midrule[0.7pt]
    CVRP1K+G & 32.65 & 1.21 & 4.38s & 32.60 & 1.21 & 5.06s	& 32.58 & 1.21 & 5.78s & 32.56 & 1.21 & 6.49s & 32.55 & 1.21 & 7.21s \\
    CVRP2K+U & 57.73 & 1.80 & 19.54s & 57.69 & 1.08 & 23.48s & 57.67 & 1.08 & 26.96s & 57.67 & 1.08 & 29.97s & 57.67 & 1.08 & 32.40s \\
    \midrule[0.7pt]
    \multirow{2}{*}{SL-driven HLGP}  & \multicolumn{3}{c}{$K=6$} & \multicolumn{3}{c}{$K=7$} & \multicolumn{3}{c}{$K=8$} & \multicolumn{3}{c}{$K=9$} & \multicolumn{3}{c}{$K=10$} \\
    \cmidrule(r){2-4} \cmidrule(r){5-7} \cmidrule(r){8-10} \cmidrule(r){11-13} \cmidrule(r){14-16}
        & \normalsize{$Avg.$} & \normalsize{$Std.$} & \normalsize{$Time$} & \normalsize{$Avg.$} & \normalsize{$Std.$} & \normalsize{$Time$} & \normalsize{$Avg.$} & \normalsize{$Std.$} & \normalsize{$Time$} & \normalsize{$Avg.$} & \normalsize{$Std.$} & \normalsize{$Time$} & \normalsize{$Avg.$} & \normalsize{$Std.$} & \normalsize{$Time$} \\
    \midrule[0.7pt]
    CVRP1K+G & 32.55 & 1.21 & 7.92s & 32.55 & 1.21 & 8.62s	& 32.55 & 1.21 & 9.34s & 32.55 & 1.21 & 10.05s & 32.54 & 1.21 & 10.76s \\
    CVRP2K+U & 57.66 & 1.08 & 34.96s & 57.66 & 1.08 & 37.60s & 57.66 & 1.08 & 40.00s & 57.66 & 1.08 & 42.37s & 57.66 & 1.08 & 44.77s \\
    \bottomrule[1.2pt]
    \end{tabular}
    \label{tab:sl_K}
\end{table*}

The SL-driven HLGP leverages the same variant Transformer architecture utilized in BQ~\cite{drakulic2024bq} for both the global and local partition policies to incrementally generate the partition solution. Although BQ~\cite{drakulic2024bq} was initially designed to serve as a neural solver for producing the CVRP solution directly, it can also be repurposed to generate the partition solution. This adaptability arises from the fact that each subtour in the CVRP solution can be extracted by rearranging the node sequence within the subgraph of the partition solution. Consequently, the decoded node sequence from BQ~\cite{drakulic2024bq} can be interpreted as the partition solution, regardless of the specific order of decoding. This justifies that any neural solver crafted for generating CVRP solutions can be repurposed as the partition policy. Subsequently, the same pretrained local permutation policy utilized in the RL-driven HLGP is applied to address the subproblems derived from partition policies. Similarly, the cost of the CVRP solution corresponds to the evaluation of the partition solution.

The training phase of the SL-driven HLGP consists of two stages. Notably, given that the two partition policies are initialized randomly, within our proposed multi-level hierarchical learning (HL) framework, incorporating beam search into the global and local partition policies at each level to generate partition solutions as labels for large-scale instances does not ensure that these solutions are of sufficiently high quality to effectively supervise the training of the involved partition policies. Additionally, there are no optimal solvers specifically tailored for the partition task in the context of CVRP. However, as previously mentioned, the CVRP solution can be regarded as a partition solution, irrespective of the node order, with the optimal CVRP solution being synonymous with the optimal partition solution. Thus, in the first stage, we employ curriculum learning to train both the global and local partition policies on CVRP instances comprising 100 nodes. The labels for the small-scale instances of 100 nodes are generated from the procedure employed in BQ~\cite{drakulic2024bq}, and we adhere to the training methodology outlined in BQ~\cite{drakulic2024bq} during this stage. In the second stage, we use CVRP instances with 1000 nodes to train the partition policies. For detailed training procedures in this stage, please refer to the Algorithm Pseudocodes. This stage comprises 20 epochs, each consisting of 500 iterations. Within each epoch, 100 problem instances with 1000 nodes are randomly sampled. The corresponding labels are generated by executing the involved policies with beam search at each level within our proposed multi-level HL framework (as illustrated in Figure 2(b) in the main body), with a beam size of 16. These labeled instances constitute the training dataset. During each iteration, batches of data from the training dataset are utilized to train the global partition policy, with a batch size set at 50. Additionally, we employ the Adam Optimizer with an initial learning rate of $1 \times 10^{-5}$. The learning rate is subject to decay by a factor of 0.9 every 5 update steps. We enforce a constraint to keep the maximum L2 norm of gradients below 1. The coefficients for the L2 regularization terms in the loss function are set at $1 \times 10^{-6}$, denoted as $\lambda_{G}=\lambda_{L}=1 \times 10^{-6}$, which aligns with common practices in Transformer-based neural solvers~\cite{kwon2020pomo,kim2022sym,zhou2023towards,gao2023towards}. The SL-driven HLGP is trained on 4 NVIDIA RTX 3090 GPUs and an INTEL(R) XEON(R) GOLD 5218R CPU@2.10GHz.

\subsection{Evaluation Details}

During the evaluation phase, our focus is on assessing the generalization capabilities of various models. Therefore, we utilize diverse datasets with varying scales and distributions. Each evaluation dataset can specify the number of customer nodes as 1000, 2000, 5000, 7000, or 10000. The customer nodes in each dataset are distributed according to a Uniform distribution, a Gaussian distribution, an explosion pattern, or a rotation pattern. Except for the dataset with 1000 customer nodes setting capacity as 200, the capacity for the other datasets is set to 300. Each dataset comprises 128 instances. The process of generating problem instances aligns with the methodologies outlined in~\cite{kool2018attention, zhou2023towards}. In addition, for comparison, the metrics include the average solution costs ($Avg.$), the standard deviation of solution costs ($Std.$), and average inference time ($Time$). 
For both RL-driven and SL-driven HLGP, the number of levels, denoted as $K$ is set as 5 across various evaluation settings. Additionally, the beam size in SL-driven HLGP for the evaluation is set as 16, 16, 8, 4, and 4 for the CVRP datasets with node counts of 1000, 2000, 5000, 7000, and 10000, respectively.

\subsection{Hyperparameter Studies}

We begin by examining the influence of the number of levels in the multi-level HL framework, denoted as $K$, on the performance of both RL-driven and SL-driven HLGP. In Table~\ref{tab:rl_K}, we assess RL-driven HLGP using the CVRP1K dataset with Gaussian distributed nodes (CVRP1K+G) and the CVRP2K dataset with uniformly distributed nodes (CVRP2K+U). Here, $K$ ranges from 1 to 10. It is evident that the average costs gradually converge to 34.51 when $K=7$ for the CVRP1K+G dataset and to 59.50 when $K=8$ for the CVRP2K+U dataset. Concurrently, the time consumption increases as the number of levels rises. Table~\ref{tab:sl_K} illustrates the performance variations of SL-driven HLGP on the CVRP1K+G and CVRP2K+U datasets as $K$ varies. Notably, for the CVRP1K+G dataset, the average cost has already converged to 32.55 when $K=5$, while for the CVRP2K+U dataset, the convergent average cost is 57.66 at $K=6$. Likewise, the time taken for computation escalates with the increase in the number of levels. In practice, to trade off the performance against efficiency, we select $K=5$ for both RL-driven and SL-driven HLGP.

We then investigate the impact of the coefficients of the entropy terms corresponding to the global and local partition policies, denoted as $\lambda_{G}$ and $\lambda_{L}$, in the loss function of RL-driven HLGP. Since the value of $\lambda_{G}$ directly influences the performance of the global partition policy, as shown in Table~\ref{tab:rl_lambda_G}, we analyze how the performance of the global partition policy varies across four datasets with node numbers ranging from 1000 to 7000 as $\lambda_{G}$ changes. Notably, setting $\lambda_{G}$ to 0.1 consistently yields the best performance across all datasets. In Table~\ref{tab:rl_lambda_L}, we display the performance of RL-driven HLGP on the same four datasets as $\lambda_{L}$ varies. The model with $\lambda_{L}=0.005$ excels on three out of four datasets. Therefore, we decide to set $\lambda_{G}=0.1$ and $\lambda_{L}=0.005$. It is worth noting that the coefficients of L2 regularization terms in the loss function are commonly set to $1 \times 10^{-6}$ in Transformer-based neural solvers~\cite{kwon2020pomo,kim2022sym,zhou2023towards,gao2023towards}. Therefore, we maintain this consistent setting for the coefficients of the L2 regularization terms for both the global and local partition policies, denoted as $\lambda_{G}=\lambda_{L}=1 \times 10^{-6}$. This configuration has proven effective in our SL-driven HLGP setup.

\begin{table}[t]
    \caption{Hyperparameter study for $\lambda_{G}$ in RL-driven HLGP on uniformly distributed CVRP instances.}
    \setlength\tabcolsep{0.9pt}
    \renewcommand{\arraystretch}{0.9}
    \small
    \centering
    \begin{tabular}{lcccccccc}
    \toprule[1.2pt]
       \multirow{2}{*}{$\lambda_{G}$}  & \multicolumn{2}{c}{CVRP1K} & \multicolumn{2}{c}{CVRP2K} & \multicolumn{2}{c}{CVRP5K} & \multicolumn{2}{c}{CVRP7K} \\
    \cmidrule(r){2-3} \cmidrule(r){4-5} \cmidrule(r){6-7} \cmidrule(r){8-9} 
        & \normalsize{$Avg.$} & \normalsize{$Std.$} & \normalsize{$Avg.$} & \normalsize{$Std.$} & \normalsize{$Avg.$} & \normalsize{$Std.$} & \normalsize{$Avg.$} & \normalsize{$Std.$} \\
    \midrule[0.7pt]
0.1 & 46.40 & 0.91 & 63.55 & 1.19 & 138.37 & 3.51 & 187.91 & 4.99 \\
0.05 & 46.48 & 0.82 & 63.68 & 1.21 & 138.68 & 3.57 & 188.65 & 5.09 \\
0.01 & 46.51 & 0.92 & 63.69 & 1.19 & 139.02 & 3.57 & 189.31 & 5.13 \\
0.005 & 46.54 & 0.92 & 63.77 & 1.20 & 138.74 & 3.54 & 188.71 & 5.05 \\
0.001 & 46.64 & 0.93 & 63.84 & 1.20 & 139.49 & 3.61 & 190.13 & 5.18 \\
0.0005 & 46.54 & 0.92 & 63.78 & 1.20 & 139.20 & 3.62 & 189.80 & 5.23 \\
0.0001 & 46.60 & 0.92 & 63.70 & 1.20 & 138.89 & 3.60 & 189.04 & 5.19 \\
0.0 & 46.55 & 0.94 & 63.75 & 1.21 & 139.16 & 3.58 & 189.30 & 5.14 \\
    \bottomrule[1.2pt]
    \end{tabular}
    \label{tab:rl_lambda_G}
\end{table}

\begin{table}[t]
    \caption{Hyperparameter study for $\lambda_{L}$ in RL-driven HLGP on uniformly distributed CVRP instances.}
    \setlength\tabcolsep{0.9pt}
    \renewcommand{\arraystretch}{0.9}
    \small
    \centering
    \begin{tabular}{lcccccccc}
    \toprule[1.2pt]
       \multirow{2}{*}{$\lambda_{L}$}  & \multicolumn{2}{c}{CVRP1K} & \multicolumn{2}{c}{CVRP2K} & \multicolumn{2}{c}{CVRP5K} & \multicolumn{2}{c}{CVRP7K} \\
    \cmidrule(r){2-3} \cmidrule(r){4-5} \cmidrule(r){6-7} \cmidrule(r){8-9} 
        & \normalsize{$Avg.$} & \normalsize{$Std.$} & \normalsize{$Avg.$} & \normalsize{$Std.$} & \normalsize{$Avg.$} & \normalsize{$Std.$} & \normalsize{$Avg.$} & \normalsize{$Std.$} \\
    \midrule[0.7pt]
0.1 & 43.92 & 0.82 & 59.81 & 1.11 & 128.31 & 3.04 & 173.78 & 4.36 \\
0.05 & 43.91 & 0.82 & 59.73 & 1.12 & 128.16 & 3.00 & 173.63 & 4.35 \\
0.01 & 43.91 & 0.81 & 59.80 & 1.12 & 128.52 & 3.01 & 174.02 & 4.32 \\
0.005 & 43.78 & 0.85 & 59.58 & 1.17 & 128.12 & 3.06 & 173.71 & 4.39 \\
0.001 & 44.01 & 0.81 & 60.01 & 1.13 & 128.38 & 3.02 & 173.40 & 4.22 \\
0.0005 & 44.05 & 0.82 & 60.14 & 1.11 & 128.38 & 3.03 & 173.35 & 4.23 \\
0.0001 & 44.47 & 0.94 & 60.14 & 1.25 & 133.39 & 3.41 & 182.30 & 4.78 \\
0.0 & 43.94 & 0.86 & 59.70 & 1.17 & 129.07 & 3.08 & 175.10 & 4.27 \\
    \bottomrule[1.2pt]
    \end{tabular}
    \label{tab:rl_lambda_L}
\end{table}

\subsection{Ablation Studies}

\begin{table*}[hbt]
    \caption{Ablation study for RL-driven HLGP.}
    \small
    \centering
    \begin{tabular}{lcccccccc}
    \toprule[1.2pt]
       \multirow{2}{*}{RL-driven HLGP}  & \multicolumn{2}{c}{glob.} & \multicolumn{2}{c}{glob.+subp.} & \multicolumn{2}{c}{glob.+loc.} & \multicolumn{2}{c}{glob.+loc.+subp.} \\
       \cmidrule(r){2-3} \cmidrule(r){4-5} \cmidrule(r){6-7} \cmidrule(r){8-9} 
        & \normalsize{$Avg.$} & \normalsize{$Std.$} & \normalsize{$Avg.$} & \normalsize{$Std.$} & \normalsize{$Avg.$} & \normalsize{$Std.$} & \normalsize{$Avg.$} & \normalsize{$Std.$} \\
    \midrule[0.7pt]
CVRP1K+G & 39.12 & 1.41 & 37.98 & 1.35 & 35.04 & 1.28 & 34.58 & 1.26 \\
CVRP2K+U & 64.38 & 1.25 & 63.55 & 1.19 & 60.33 & 1.14 & 59.58 & 1.17 \\
    \bottomrule[1.2pt]
    \end{tabular}
    \label{tab:rl_abl}
\end{table*}

\begin{table*}[hbt]
    \caption{Ablation study for SL-driven HLGP.}
    \small
    \centering
    \begin{tabular}{lcccccccc}
    \toprule[1.2pt]
       \multirow{2}{*}{SL-driven HLGP}  & \multicolumn{2}{c}{SS.+glob.} & \multicolumn{2}{c}{SS.+LS.+glob.} & \multicolumn{2}{c}{SS.+glob.+loc.} & \multicolumn{2}{c}{SS.+LS.+glob.+loc.} \\
       \cmidrule(r){2-3} \cmidrule(r){4-5} \cmidrule(r){6-7} \cmidrule(r){8-9} 
        & \normalsize{$Avg.$} & \normalsize{$Std.$} & \normalsize{$Avg.$} & \normalsize{$Std.$} & \normalsize{$Avg.$} & \normalsize{$Std.$} & \normalsize{$Avg.$} & \normalsize{$Std.$} \\
    \midrule[0.7pt]
CVRP1K+G & 33.28 & 1.26 & 32.73 & 1.21 & 32.88 & 1.23 & 32.55 & 1.21 \\
CVRP2K+U & 59.51 & 1.09 & 57.77 & 1.08 & 59.10 & 1.09 & 57.67 & 1.08 \\
    \bottomrule[1.2pt]
    \end{tabular}
    \label{tab:sl_abl}
\end{table*}

\begin{table*}[hbt]
    \caption{Time overheads analysis for HLGP.}
    \small
    \centering
    \begin{tabular}{lcccccc}
    \toprule[1.2pt]
       \multirow{2}{*}{RL-driven HLGP} & \multicolumn{3}{c}{global partition process} & \multicolumn{3}{c}{overall partition process} \\
       \cmidrule(r){2-4} \cmidrule(r){5-7}  
        & \normalsize{$Avg.$} & \normalsize{$Std.$} & \normalsize{$Time$} & \normalsize{$Avg.$} & \normalsize{$Std.$} & \normalsize{$Time$} \\
    \midrule[0.7pt]
CVRP1K+G & 37.98 & 1.35 & 0.58s & 34.58 & 1.26 & 3.47s \\
CVRP2K+U & 63.55 & 1.19 & 2.37s & 59.58 & 1.17 & 10.03s \\
    \midrule[0.7pt]
    \multirow{2}{*}{SL-driven HLGP} & \multicolumn{3}{c}{global partition process} & \multicolumn{3}{c}{overall partition process} \\
       \cmidrule(r){2-4} \cmidrule(r){5-7}  
        & \normalsize{$Avg.$} & \normalsize{$Std.$} & \normalsize{$Time$} & \normalsize{$Avg.$} & \normalsize{$Std.$} & \normalsize{$Time$} \\
    \midrule[0.7pt]
        CVRP1K+G & 32.73 & 1.21 & 3.69s & 32.55 & 1.21 & 7.21s \\
        CVRP2K+U & 57.77 & 1.08 & 15.11s & 57.67 & 1.08 & 32.40s \\
    \bottomrule[1.2pt]
    \end{tabular}
    \label{tab:overhead}
\end{table*}

In this section, we first evaluate the individual contributions of each module in RL-driven HLGP on the CVRP1K+G and CVRP2K+U datasets. In Table~\ref{tab:rl_abl}, ``glob." denotes that only the global partition policy is utilized, trained solely on randomly sampled CVRP1K instances. ``glob.+subp." indicates that the global partition policy is trained on both the random CVRP1K instances and the encountered subproblems during training. On the other hand, ``glob.+loc." and ``glob.+loc.+subp." signify that these models incorporate the local partition policy to progressively rectify misclusterings. It is evident from Table~\ref{tab:rl_abl} that both ``subp" and ``loc." modules contribute to performance enhancement. However, the ``loc." module notably makes a more substantial contribution to the final performance on both evaluation datasets. Additionally, Figure~\ref{fig:rl_curves} presents the training and validation curves of the global partition policy of RL-driven HLGP. It is clear that incorporating the encountered subproblems during training for the global partition policy can expedite the training process.

In Table~\ref{tab:sl_abl}, we explore the individual contributions of each module in SL-driven HLGP on the CVRP1K+G and CVRP2K+U datasets. "SS.+glob." signifies that only the global partition policy trained on small-scale instances (100 nodes) is used for evaluation. "SS.+LL.+glob." indicates that the global partition policy is additionally trained on large-scale instances (1000 nodes). Similarly, "SS.+glob.+loc." and "SS.+LS.+glob.+loc." represent models incorporating the local partition policy. The table indicates that when evaluating the CVRP1K+G dataset, the "LS." and "loc." modules roughly make the equivalent contributions to the performance. However, on the larger-scale dataset (CVRP2K+U), the "LS." module clearly outweighs the "loc." module in importance. This suggests that our proposed multi-level HL framework, coupled with beam search for each involved policy, can indeed yield dependable high-quality solutions serving as labels for supervised training, thereby enhancing the generalization capability of the partition policy.

\begin{figure}[t]
\centering
\subfigure[]{\label{fig:rl_curve_1}\includegraphics[width=0.85\columnwidth]{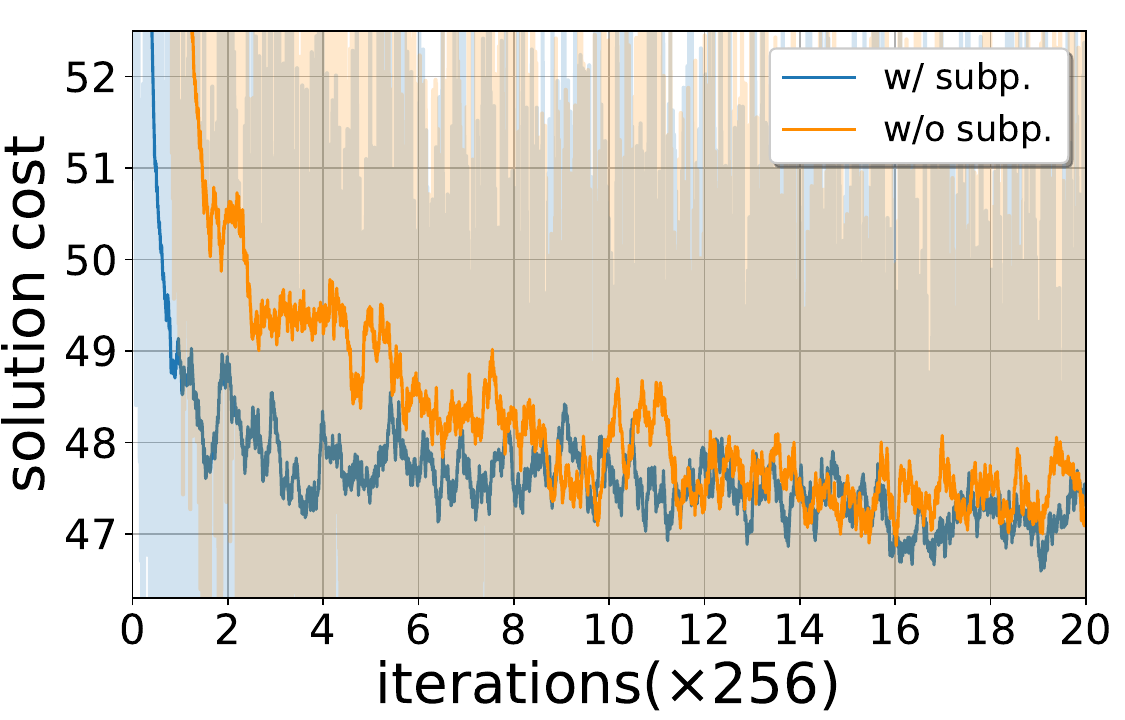}}
\subfigure[]{\label{fig:rl_curve_2}\includegraphics[width=0.85\columnwidth]{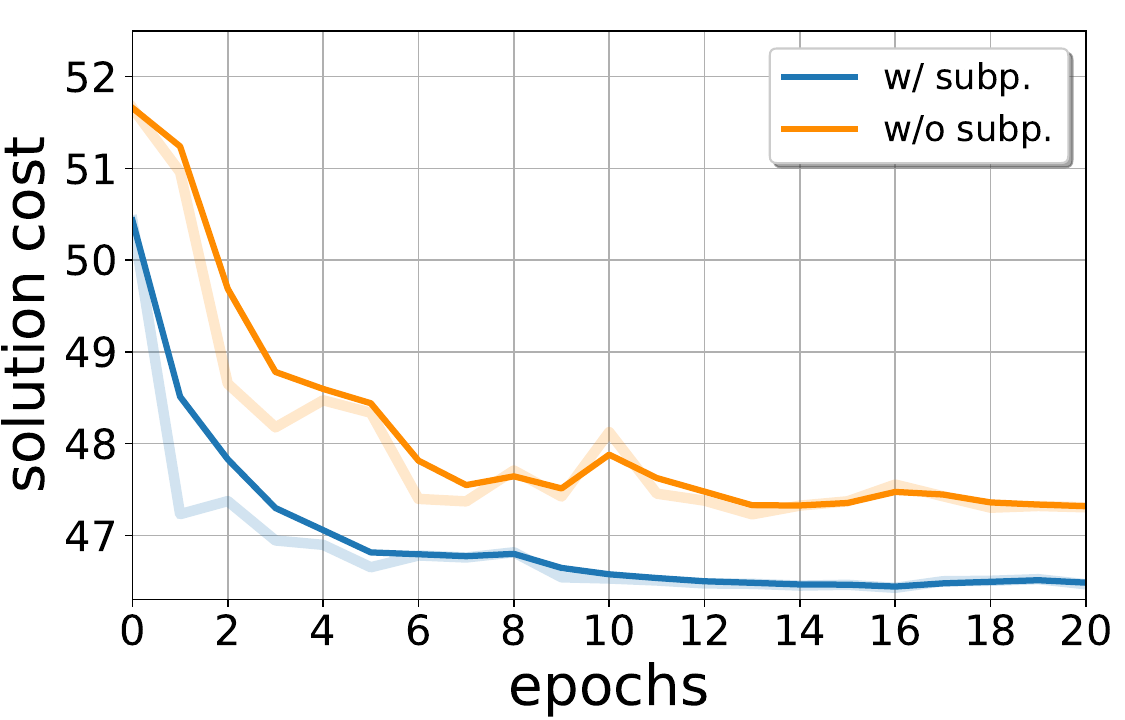}}
\caption{The training curve (a) and the validation curve (b) of the global partition policy in RL-driven HLGP.
}
\label{fig:rl_curves}
\end{figure}

\subsection{Time Overheads Analysis}
In this section, we analyze the time overhead contributions of the global partition process and the local partition process for both RL-driven HLGP and SL-driven HLGP, as illustrated in Table~\ref{tab:overhead}. In the RL-driven HLGP, it is evident that the local partition process contributes significantly more to the average time overhead compared to the global partition process. This disparity arises due to the local partition process typically involving multiple levels of partitioning. Conversely, in the SL-driven HLGP, we observe a similar average time overhead between the global and local partition processes. This similarity can be attributed to the nature of the global partition process in the SL-driven HLGP. In this process, with each step, a new encountered subproblem is derived for the current node selection, necessitating each encountered subproblem to be processed by the policy network. As the number of encountered subproblems is linearly proportional to the graph size, the time overhead of the global partition process increases accordingly. In contrast, in the global partition process of the RL-driven HLGP, only an encountered subproblem is formed and fed to the policy network when a new subgraph is constructed. Consequently, the time overhead associated with the global partition process is less significant compared to that of the local partition process.

\subsection{Visualization Results}

In Figure~\ref{fig:rl_visual} and Figure~\ref{fig:sl_visual}, we present the visualized CVRP solutions generated by RL-driven and SL-driven HLGP, respectively, on the CVRP1K datasets with distributions including Uniform, Gaussian, Explosion, and Rotation.

\subsection{Implementation details of baselines}
In this section, we provide the implementation details of the baseline methods, including both both non-learning and learning-based approaches, selected for comparison with our proposed approaches.

\textbf{LKH3.} We adhere to the settings outlined in Omni-POMO~\cite{zhou2023towards} to reproduce the results of LKH3. For solving each CVRP instance with varying scales and distributions, we set the maximum number of trials to 10000.

\textbf{HGS.} In order to decrease the runtime, we adjust the number of iterations of HGS for the CVRP datasets with varying scales. Specifically, we set the number of iterations to 20000, 5000, 2000, 1500, and 1000 for the CVRP instances with node counts of 1000, 2000, 5000, 7000, and 10000, respectively.

\textbf{AM.} We follow the implementation guidelines of AM as outlined by~\citet{ye2024glop}, where the checkpoint trained on CVRP100 is adapted for comparison purposes. In the evaluation stage, to enhance the quality of solutions generated by AM, we employ a sampling decoding strategy with 1280 solutions for each instance. The temperature of the softmax function in the output layer is set to 0.1 to prevent solutions from diverging towards lower-quality outcomes.

\textbf{POMO.} To benchmark against POMO, we generalize the checkpoint trained on CVRP100 to generate results across different CVRP datasets. The POMO size is adjusted to align with the number of nodes in each CVRP instance. Furthermore, data augmentation is implemented for each CVRP instance with an augmentation factor of 8.

\textbf{Sym-POMO.} We stick to the original implementations of Sym-POMO, where the POMO size is adjusted to match the number of nodes, and each CVRP instance is augmented by a factor of 8. The checkpoint trained on CVRP100 is generalized for comparisons.

\textbf{AMDKD.} AMDKD aims to improve the generalization capacity of neural solvers on CVRP datasets with scale and distribution shifts through knowledge distillation. We adhere to the original evaluation settings with the model trained via knowledge distillation on CVRP100 datasets with various distributions. The default POMO size and augmentation factor are employed in our evaluations.

\textbf{Omni-POMO.} Omni-POMO aims to transfer neural solvers to problem instances with diverse scales and distributions through meta-learning techniques. Therefore, we strictly follow the evaluation configurations by utilizing the provided model trained via MAML with 250000 epochs to replicate the results of Omni-POMO across various CVRP datasets. The POMO size is configured to match the number of nodes in each CVRP instance. Additionally, data augmentation is applied to every CVRP instance with an augmentation factor of 8.

\textbf{ELG-POMO.} ELO-POMO integrates both a global neural solver and a local neural solver to leverage insensitive local topological features, thereby enhancing the overall system's generalization capability. We generalize the checkpoint trained on CVRP100 with 250000 epochs for comparison purposes. Throughout the evaluation phase, we maintain the default values of the POMO size and augmentation factors as recommended in the original implementations.

\textbf{INViT.} INViT utilizes multiple encoders to capture the features of subgraphs of varying scales within each CVRP instance, enabling the utilization of local topological features for the generalization improvement. In our study, we generalize the checkpoint trained on CVRP100 to evaluate the generalization ability against other baselines. During the evaluation phase, given that each instance in the evaluation dataset consists of at least 1000 nodes, the beam size is accordingly set to 16 for each evaluation CVRP dataset, as recommended in the original implementation. All other settings strictly conform to the original implementations.

\textbf{LEHD.} LEHD incorporates iterative local revision (known as RRC) by utilizing a SL-trained neural solver to improve the generalization capability of the constructive method. We leverage the officially provided checkpoint to compare its generalization ability with other methods. During the evaluation phase, as all baseline methods are evaluated on large-scale instances with a minimum of 1000 nodes, to balance performance and efficiency, the number of RRC is set to 200, 50, 5, and 1 for CVRP datasets with node counts of 1000, 2000, 5000, and 7000. For the CVRP10K dataset, the greedy model is directly deployed for comparison purposes.

\textbf{BQ.} BQ enhances generalization ability by leveraging the inherent recursive property of the CVRP instance. We utilize the officially provided checkpoint to generalize its application to various CVRP datasets. During the evaluation phase, to achieve a trade-off between performance and efficiency, the beam size is set to 16, 16, 8, 4, and 4 for the CVRP datasets with node counts of 1000, 2000, 5000, 7000, and 10000, respectively. Other settings remain aligned with the original implementations.

\textbf{L2I.} L2I is an iterative approach that merges an RL-based policy with predefined local operators to iteratively enhance a given solution. We directly apply the provided official checkpoint of the RL-based policy and the predefined set of local operators to diverse CVRP datasets for generalization comparisons. Additionally, to prevent unexpected increases in makespan due to excessive iterations, we limit the maximum number of rollout steps for the RL-based policy to 40, 10, 6, 4, and 2 for the CVRP datasets with node counts of 1000, 2000, 5000, 7000, and 10000, respectively.

\textbf{NLNS.} NLNS seamlessly integrates heuristic destroy operators and a collection of learning-based repair policies to iteratively enhance a given solution. We closely follow the official implementations of NLNS for generalization comparisons, with the exception of the setting for maximum running time per instance. Since NLNS is intended to generalize to large-scale instances, the maximum running time in our study is extended to approximately 1200ms per instance for various CVRP datasets.

\textbf{DACT.} DACT, as an iterative method, redesigns the solution representation for the RL-based policy to improve the given solution. We thus generalize the official checkpoint trained on CVRP100 to various CVRP datasets for comparisons. Following the guidelines in the official implementation, to tailor DACT to larger CVRP instances, smaller values are recommended for the number of perturbations and the dummy rate. As a result, we set the dummy rate to 0.05 and the number of perturbations to 2. Furthermore, to prevent unexpected increases in makespan resulting from an excessive number of iterations, the maximum rollout steps of the RL-based policy are set to 400.

\textbf{L2D.} L2D is replicated by adhering to the publicly available open-source implementation across various CVRP datasets. It is impressive to observe L2D's strong performance on CVRP datasets exhibiting distribution and scale shifts. This success can be attributed to L2D's heavy reliance on near-optimal solutions from numerous neighboring subgraphs, derived from classical heuristic solvers, to guide both clustering subgraphs and resolving subproblems.

\textbf{GLOP.} We rigorously follow the official implementations of GLOP to benchmark it against other baseline methods. Moreover, for the selected baseline methods that align with those in GLOP, we rely on GLOP's recommendations to reproduce these methods, thereby ensuring that our replicated results are at least on par with those documented in GLOP.


\section{Algorithm Pseudocodes}
The algorithm pseudocodes for both RL-driven and SL-driven HLGP are illustrated in Algorithm~\ref{alg:rl} and Algorithm~\ref{alg:sl}, respectively.

\begin{algorithm}[h]
    \caption{RL-driven HLGP}
    \label{alg:rl}
    \begin{algorithmic}[1] 
    \REQUIRE Problem instance distribution $p_{I}$, permutation policy $\pi_{\mathrm{perm}}$
    \ENSURE global partition policy $\pi_{\theta_{G}}$, local partition policy $\pi_{\theta_{L}}$
        \FOR{$n=0$ to $N$}
        \STATE Sample an instance: $I=(G, D, N_{\mathrm{max}}) \sim p_{I}$
        \STATE Initialize $\mathcal{C}^{(0)}$: $\mathcal{C}^{(0)} = \{\}$
        \STATE \textbf{while} $G \neq \emptyset$ \textbf{do}
        \STATE \quad \texttt{\# sample a partition solution}
        \STATE \quad $\hat{\mathcal{C}}^{(0)} = \{\hat{c}_{1}^{(0)}, \ldots,  \hat{c}_{N_{c}}^{(0)}\} \sim \pi_{\theta_{G}^{n}}(\cdot|I)$
        \STATE \quad \texttt{\# construct the subproblem}
        \STATE \quad $N_{\mathrm{max}} \leftarrow N_{\mathrm{max}} - 1$, $I \leftarrow (\cup_{i=2}^{N_{c}}\hat{c}_{i}^{(0)}, D, N_{\mathrm{max}})$
        \STATE \quad Update $\mathcal{C}^{(0)}$: $\mathcal{C}^{(0)} \leftarrow \mathcal{C}^{(0)} \cup \{ \hat{c}_{1}^{(0)} \}$
        \STATE \quad \texttt{\# train global partition policy}
        \STATE \quad CVRP solution and Cost: $\mathcal{T} \sim \pi_{\mathrm{perm}}(\cdot|\hat{\mathcal{C}}^{(0)})$, $e(\mathcal{T})$
        \STATE \quad Update parameter: $\theta_{G}^{n} \leftarrow
        \mathrm{AdamOpt}(\theta_{G}^{n}, \nabla_{\theta_{G}}\hat{J}(\theta_{G}^{n}, \theta_{L}^{n}))$
        \STATE \textbf{end while}
        \FOR{$k=1$ to K}
        \STATE Construct subproblems: $I_{j}^{(k-1)}, 1 \leq j \leq \lfloor \frac{N_{c}}{2}\rfloor$
        \STATE \texttt{\# sample partition solutions }
        \STATE $\mathcal{C}_{j}^{(k-1)} \sim \pi_{\theta_{L}^{n}}(\cdot|I_{j}^{(k-1)}), 1 \leq j \leq \lfloor \frac{N_{c}}{2}\rfloor$
        \STATE \texttt{\# train local partition policy}
        \STATE CVRP solutions and Costs: $\mathcal{T}_{j}\sim\pi_{\mathrm{perm}}(\cdot|\mathcal{C}_{j}^{(k-1)}), e(\mathcal{T}_{j})$
        \STATE Update parameter: $\theta_{L}^{n} \leftarrow
        \mathrm{AdamOpt}(\theta_{L}^{n}, \nabla_{\theta_{L}}\hat{J}(\theta_{G}^{n}, \theta_{L}^{n}))$
        \ENDFOR
        \STATE $\theta_{G}^{n+1} \leftarrow \theta_{G}^{n}$; $\theta_{L}^{n+1} \leftarrow \theta_{L}^{n}$
        \ENDFOR
    \end{algorithmic}
\end{algorithm}

\begin{algorithm}[h]
    \caption{SL-driven HLGP}
    \label{alg:sl}
    \begin{algorithmic}[1] 
    \REQUIRE Problem instance distribution $p_{I}$, permutation policy $\pi_{\mathrm{perm}}$
    \ENSURE global partition policy $\pi_{\theta_{G}}$, local partition policy $\pi_{\theta_{L}}$
        \FOR{$N=0$ to $N$}
        \STATE Sample an instance: $I=(G, D, N_{\mathrm{max}}) \sim p_{I}$
        \STATE \texttt{\# generate the label}
        \STATE $\bar{\mathcal{C}}^{(0)} \sim \mathrm{Beam Search}
        (I, \pi_{\theta_{G}^{n}}, \pi_{\mathrm{perm}})$
        \FOR{$k=1$ to K}
        \STATE Construct subproblems $I_{j}^{(k-1)}, 1 \leq j \leq \lfloor \frac{N_{G}}{2}\rfloor$
        \STATE \texttt{\# solve each subproblem}
        \STATE $\bar{\mathcal{C}}_{j}^{(k-1)} \sim \mathrm{Beam Search}(I_{j}^{(k-1)}, \pi_{\theta_{L}^{n}}, \pi_{\mathrm{perm}})$
        \STATE Construct $\bar{\mathcal{C}}^{(K)}$ from the set of $\bar{\mathcal{C}}_{j}^{(k-1)}$.
        \ENDFOR
        \STATE $\bar{\mathcal{C}} \leftarrow \bar{\mathcal{C}}^{(K)}$
        \STATE \texttt{\# train global partition policy}
        \STATE Generate labeled instances: $(I_{N_{v}(t)}, \bar{\mathcal{C}}[t]), t \geq 1$
        \STATE Update parameter: $\theta_{G}^{n} \leftarrow
        \mathrm{AdamOpt}(\theta_{G}^{n}, \nabla_{\theta_{G}}\hat{J}(\theta_{G}^{n}, \theta_{L}^{n}))$
        \STATE \texttt{\# train local partition policy}
        \STATE Generate labeled instances: $(I_{i}, \bar{C}_{i}), i \geq 1$
        \STATE Update parameter: $\theta_{L}^{n} \leftarrow
        \mathrm{AdamOpt}(\theta_{L}^{n}, \nabla_{\theta_{L}}\hat{J}(\theta_{G}^{n}, \theta_{L}^{n}))$
        \STATE $\theta_{G}^{n+1} \leftarrow \theta_{G}^{n}$; $\theta_{L}^{n+1} \leftarrow \theta_{L}^{n}$
        \ENDFOR
    \end{algorithmic}
\end{algorithm}

\newpage

\section{Proofs}

\subsection{Proof of Theorem 1}
\begin{theorem}
\label{thm:gplc}
    The objective in solving an original CVRP instance $I$ is to identify a (permutation) policy $\pi(\mathcal{T}|I) \in \Delta(\mathbb{S}_{\mathcal{T}})$ so as to minimize the expected cost $\mathbb{E}_{\mathcal{T} \sim \pi}[e(\mathcal{T})]$. If $\pi_{\mathrm{perm}}^{\ast} \in \Delta(\mathbb{S}_{\mathcal{T}})$ is optimal for each subproblem $(c_{i}, D, 1)$, then the original objective can be reframed as identifying an optimal partition policy $\pi_{\mathrm{part}}^{\ast} \in \Delta(\mathbb{S}_{\mathcal{C}})$ to minimize the expected cost, expressed as:
    \begin{equation}
    \label{equ:gplc}
    \min_{\pi_{\mathrm{part}}} \; \mathbb{E}_{\mathcal{C} \sim \pi_{\mathrm{part}}} [\sum_{i=1}^{N_{c}} \mathbb{E}_{\tau_{i} \sim \pi_{\mathrm{perm}}^{\ast}} (e(\tau_{i})) ],
    \end{equation}
    The partition policy $\pi_{\mathrm{part}}$ and the permutation policy $\pi_{\mathrm{perm}}$ are formulated respectively as follows:
    \begin{equation}
    \label{equ:part_perm_pi}
    \begin{split}
    &\pi_{\mathrm{part}}(\mathcal{C}|I)=\prod_{n=1}^{N_{\mathrm{sol}}}\pi_{\mathrm{part}}(\mathcal{C}[n]|\mathcal{C}[0:n-1], I), \\
    &\pi_{\mathrm{perm}}(\mathcal{T}|\mathcal{C}) = \prod_{i=1}^{N_{c}} \pi_{\mathrm{perm}}(\tau_{i}|c_{i})
    \end{split}
    \end{equation}
    where $N_{\mathrm{sol}}$ denotes the length of partition solution.
\end{theorem}

\begin{proof}
    For a CVRP instance $I$, the feasible CVRP solution $\mathcal{T}$ comprises $N_{\tau}$ subtours $\tau_{i}, 1 \leq i \leq N_{\tau}$. If a partition solution $\mathcal{C}=\{c_{1}, \ldots, c_{N_{c}}\}$ corresponds to $\mathcal{T}$, then $\tau_{i}$ can be obtained by rearranging the nodes in $c_{i}$, and $N_{\tau} = N_{c}$. This observation indicates that a given CVRP solution $\mathcal{T}$ can be represented by the corresponding partition solution and the order of the nodes in each subgraph, denoted as $(\mathcal{C}, \mathcal{T}) = \{c_{1}, \ldots, c_{N_{c}}, \tau_{1}, \ldots, \tau_{N_{c}}\}$. Therefore, we can derive the following expressions for the primary objective.
    \begin{equation}
    \label{equ:thm1_1}
    \begin{split}
        &\min_{\pi} \;\mathbb{E}_{\mathcal{T}} [e(\tau)] \\
    = &\min_{\pi} \; \sum_{\mathcal{T}} \pi(\mathcal{T}|I) e(\mathcal{T}) \\
    = &\min_{\pi} \; \sum_{(\mathcal{C}, \mathcal{T})} \pi((\mathcal{C}, \mathcal{T})|I) e(\mathcal{T}) \\
    = &\min_{\pi_{\mathrm{part}}, \pi_{\mathrm{perm}}} \sum_{\mathcal{C}}\pi_{\mathrm{part}}(\mathcal{C}|I) \sum_{\mathcal{T}}\prod_{i=1}^{N_{c}}\pi_{\mathrm{perm}}(\tau_{i}|c_{i}) \sum_{i}^{N_{c}}e(\tau_{i}).
    \end{split}
    \end{equation}
    We can simplify the expectation term associated with $\pi_{\mathrm{perm}}$ further as follows:
    \begin{equation}
    \label{equ:thm1_2}
    \begin{split}
\sum_{\mathcal{T}}\prod_{i=1}^{N_{c}}\pi_{\mathrm{perm}}(\tau_{i}|c_{i}) \sum_{i}^{N_{c}}e(\tau_{i}) =  \sum_{i=1}^{N_{c}} \sum_{\tau_{i}} \pi_{\mathrm{perm}}(\tau_{i}|c_{i})e(\tau_{i}).
    \end{split}
    \end{equation}
Thus, we can plug Equation~\ref{equ:thm1_2} into Equation~\ref{equ:thm1_1} to derive the following objective function:
\begin{equation}
\label{equ:thm1_3}
\begin{split}
    &\min_{\pi_{\mathrm{part}}, \pi_{\mathrm{perm}}} \sum_{\mathcal{C}}\pi_{\mathrm{part}}(\mathcal{C}|I) \sum_{\mathcal{T}}\prod_{i=1}^{N_{c}}\pi_{\mathrm{perm}}(\tau_{i}|c_{i}) \sum_{i}^{N_{c}}e(\tau_{i}) \\
    = & \min_{\pi_{\mathrm{part}}, \pi_{\mathrm{perm}}} \sum_{\mathcal{C}}\pi_{\mathrm{part}}(\mathcal{C}|I) \sum_{i=1}^{N_{c}} \sum_{\tau_{i}} \pi_{\mathrm{perm}}(\tau_{i}|c_{i})e(\tau_{i}) \\
    = & \min_{\pi_{\mathrm{part}}} \sum_{\mathcal{C}}\pi_{\mathrm{part}}(\mathcal{C}|I) \sum_{i}^{N_{c}}\min_{\pi_{\mathrm{perm}}} \sum_{\tau_{i}} \pi_{\mathrm{perm}}(\tau_{i}|c_{i})e(\tau_{i}) \\
\end{split}
\end{equation}
If we denote the optimal permutation policy as $\pi_{\mathrm{perm}}^{\ast}$, then the primary objective can be expressed as:
\begin{equation}
    \label{equ:thm1_4}
        \min_{\pi_{\mathrm{part}}} \; \mathbb{E}_{\mathcal{C} \sim \pi_{\mathrm{part}}} [\sum_{i=1}^{N_{c}} \mathbb{E}_{\tau_{i} \sim \pi_{\mathrm{perm}}^{\ast}} (e(\tau_{i})) ].
    \end{equation}
\end{proof}

\subsection{Proof of Theorem 2}
\begin{theorem}
\label{thm:rl_obj}
Let $g(c_{i})$ denote $\mathbb{E}_{\tau_{i} \sim \pi_{\mathrm{perm}}^{\ast}(\cdot|c_{i})}(e(\tau_{i}))$. It is clear that $f(\mathcal{C}) = \sum_{i=1}^{N_{c}} g(c_{i})$ acts as a feasible cost function. Then, the optimization problem in the multi-level HL framework can be transformed equivalently as follows:
\begin{equation}
\label{equ:trans_obj}
\begin{split}
     \min_{\pi_{\mathrm{Gpart}}, \pi_{\mathrm{Lpart}}} &\mathbb{E}_{\mathcal{C}^{(0)}} [f(\mathcal{C}^{(0)})] +  \mathbb{E}_{\mathcal{C}^{(0)}} \mathbb{E}_{\mathcal{C}^{(1)}}[f(\mathcal{C}^{(1)}) - f(\mathcal{C}^{(0)})]+ \\
    &  \cdots + \mathbb{E}_{\mathcal{C}^{(0)}} \mathbb{E}_{\mathcal{C}^{(1)}} \cdots \mathbb{E}_{\mathcal{C}^{(K)}} [f(\mathcal{C}^{(K)}) - f(\mathcal{C}^{(K-1)})].
\end{split}
\end{equation}
The evaluation for $\mathcal{C}^{(k)}, k \geq 1$, can further be derived as:
\begin{align}
\label{equ:eval_C}
    &f(\mathcal{C}^{(k)}) - f(\mathcal{C}^{(k-1)}) = \sum_{j=1}^{\lfloor\frac{N_{c}}{2}\rfloor}
    [h(\mathcal{C}^{(k)}, k, m) - h(\mathcal{C}^{(k-1)}, k, m)]; \nonumber \\
    &h(\mathcal{C}^{(k)}, k, m) = g(c_{(m+k-1)\%N_{c}+1}^{(k)}) + g(c_{(m+k)\%N_{c}+1}^{(k)}),
\end{align}
where $m=2(j-1)$.
\end{theorem}

\begin{proof}
The objective function of the multi-level HL framework can be equivalently transformed as follows:
\begin{equation}
\label{equ:thm2_1}
\begin{split}
    &\min_{\pi_{\mathrm{Gpart}}, \pi_{\mathrm{Lpart}}} \mathbb{E}_{\mathcal{C}^{(0)}} \mathbb{E}_{\mathcal{C}^{(1)}} \cdots \mathbb{E}_{\mathcal{C}^{(K)}} [f(\mathcal{C}^{(K)})] \\
= & \min_{\pi_{\mathrm{Gpart}}, \pi_{\mathrm{Lpart}}} \mathbb{E}_{\mathcal{C}^{(0)}} \mathbb{E}_{\mathcal{C}^{(1)}} \cdots \mathbb{E}_{\mathcal{C}^{(K)}} [f(\mathcal{C}^{(K)}) - f(\mathcal{C}^{(K-1)}) \\
+ & f(\mathcal{C}^{(K-1)}) - f(\mathcal{C}^{(K-2)}) + \ldots 
+ f(\mathcal{C}^{(1)}) - f(\mathcal{C}^{(0)}) + f(\mathcal{C}^{(0)})]
\end{split}
\end{equation}
The RHS of Equation~\ref{equ:thm2_1} can be further decomposed as follows:
\begin{equation}
\label{equ:thm2_2}
\begin{split}
&\mathbb{E}_{\mathcal{C}^{(0)}} \mathbb{E}_{\mathcal{C}^{(1)}} \cdots \mathbb{E}_{\mathcal{C}^{(K)}} [ f(\mathcal{C}^{(K)}) - f(\mathcal{C}^{(K-1)}) ]  + \\
 &\mathbb{E}_{\mathcal{C}^{(0)}} \mathbb{E}_{\mathcal{C}^{(1)}} \cdots \mathbb{E}_{\mathcal{C}^{(K-1)}} [f(\mathcal{C}^{(K-1)}) - f(\mathcal{C}^{(K-2)}) + \ldots \\
&+ f(\mathcal{C}^{(1)}) - f(\mathcal{C}^{(0)}) + f(\mathcal{C}^{(0)})]
\end{split}
\end{equation}
Due to the decomposition demonstrated in Equation~\ref{equ:thm2_2}, the original objective can be iteratively simplified as:
\begin{equation}
\label{equ:thm2_3}
\begin{split}
     \min_{\pi_{\mathrm{Gpart}}, \pi_{\mathrm{Lpart}}} &\mathbb{E}_{\mathcal{C}^{(0)}} [f(\mathcal{G}^{(0)})] +  \mathbb{E}_{\mathcal{C}^{(0)}} \mathbb{E}_{\mathcal{C}^{(1)}}[f(\mathcal{C}^{(1)}) - f(\mathcal{C}^{(0)})]+ \\
    &  \cdots + \mathbb{E}_{\mathcal{C}^{(0)}} \mathbb{E}_{\mathcal{C}^{(1)}} \cdots \mathbb{E}_{\mathcal{C}^{(K)}} [f(\mathcal{C}^{(K)}) - f(\mathcal{C}^{(K-1)})].
\end{split}
\end{equation}
Then for any $k \geq 1$, the evaluation for $\mathcal{C}^{(k)}$ can be written as:
\begin{equation}
\label{equ:thm2_4}
\begin{split}
    &f(\mathcal{C}^{(k)}) - f(\mathcal{C}^{(k-1)}) \\
    = &\sum_{i=1}^{N_{c}} g(c_{i}^{(k)}) - \sum_{i=1}^{N_{c}} g(c_{i}^{(k-1)}) \\
    = &\sum_{j=1}^{\lfloor \frac{N_{c}}{2} \rfloor} [g(c_{(k+m-1)\%N_{c}+1}^{(k)}) + g(c_{(k+m)\%N_{c}+1}^{(k)})] \\
    - &\sum_{j=1}^{\lfloor \frac{N_{c}}{2} \rfloor}[g(c_{(k+m-1)\%N_{c}+1}^{(k-1)}) + g(c_{(k+m)\%N_{c}+1}^{(k-1)}) ],
\end{split}
\end{equation}
where $m=2(j-1)$.
\end{proof}

\subsection{Proof of Theorem 3}
\begin{theorem}
\label{thm:sl_obj}
Given $f(\mathcal{C})=-\mathds{1}(\mathcal{C} = \bar{\mathcal{C}})$, the optimization objective in the HLGP framework for a problem instance $I$ is to identify $\pi_{\theta_{G}}$ and $\pi_{\theta_{L}}$ so as to minimize:
\begin{equation}
    L(\theta_{G}, \theta_{L}, \bar{\mathcal{C}}) = 
    - \log\pi_{\theta_{G}}(\bar{\mathcal{C}}|I) - \sum_{i=1}^{N_{c}} \log\pi_{\theta_{L}}(\bar{\mathcal{C}}_{i}|I_{i}),
\end{equation}
where $I_{i} = (G_{i}, D, 2)$ denotes the subproblem, with $G_{i} = \bar{c}_{i} \cup \bar{c}_{i\%N_{c}+1}$, and $\bar{\mathcal{C}}_{i} = \{\bar{c}_{i}, \bar{c}_{i\%N_{c}+1} \}$ represents the corresponding label.
\end{theorem}

\begin{proof}
Given that $f(\mathcal{C})=-\mathds{1}(\mathcal{C} = \bar{\mathcal{C}})$ serves as the feasible cost function, the original objective function can be written as:
\begin{equation}
\label{equ:thm3_1}
\begin{split}
& \mathbb{E}_{\mathcal{C}^{(0)}} \mathbb{E}_{\mathcal{C}^{(1)}} \cdots \mathbb{E}_{\mathcal{C}^{(K)}} [\mathds{1}(\mathcal{C}^{(K)} = \bar{\mathcal{C}})] \\
= & \frac{1}{K+1} \mathbb{E}_{\mathcal{C}^{(0)}} \mathbb{E}_{\mathcal{C}^{(1)}} \cdots \mathbb{E}_{\mathcal{C}^{(K)}} [\mathds{1}(\mathcal{C}^{(K)} = \bar{\mathcal{C}}) + \ldots + \mathds{1}(\mathcal{C}^{(0)} = \bar{\mathcal{C}})]
\end{split}
\end{equation}
Subsequently, we can apply the same recursive decomposition as demonstrated in Equation~\ref{equ:thm2_2} and Equation~\ref{equ:thm2_3} to Equation~\ref{equ:thm3_1} as follows:
\begin{equation}
\label{equ:thm3_2}
\begin{split}
&\mathbb{E}_{\mathcal{C}^{(0)}} \mathbb{E}_{\mathcal{C}^{(1)}} \cdots \mathbb{E}_{\mathcal{C}^{(K)}} [\mathds{1}(\mathcal{C}^{(K)} = \bar{\mathcal{C}}) + \ldots + \mathds{1}(\mathcal{C}^{(0)} = \bar{\mathcal{C}})] \\ =
&\mathbb{E}_{\mathcal{C}^{(0)}} \mathbb{E}_{\mathcal{C}^{(1)}} \cdots \mathbb{E}_{\mathcal{C}^{(K)}} [\mathds{1}(\mathcal{C}^{(K)} = \bar{\mathcal{C}})]  + \dots + \mathbb{E}_{\mathcal{C}^{(0)}} [\mathds{1}(\mathcal{C}^{(0)} = \bar{\mathcal{C}})]
\end{split}
\end{equation}
Each term on the RHS of Equation~\ref{equ:thm3_2} corresponds to a distinct objective. The global partition policy at level $k=0$ is designed to maximize the probability of the solution $\bar{\mathcal{C}}$ for a specific problem instance $I$. On the other hand, the local partition policy at level $k>0$ aims to maximize the probability of the solution $\bar{\mathcal{C}}_{i}$ for a given subproblem $I_{i}$ as defined earlier. In practical scenarios, maximizing the log-probability objective can be utilized without affecting the theoretical optimal policies.

\end{proof}

\subsection{Proof of Proposition 1}

\begin{proposition}
\label{prop:mlmdp}
In the multi-level MDP framework, at the global partition level, for $t\geq1$, the state $x_{t}^{(0)}\in\mathbb{X}^{(0)}$ comprises problem instance $I$ and the partial partition solution $\mathcal{C}^{(0)}[0:t-1]$ ($\mathcal{C}^{(0)}[0] = \emptyset$). The initial distribution $\mu^{(0)}$ aligns with the problem instance distribution $p_{I}$. The action $u_{t}^{(0)}\in\mathbb{U}^{(0)}$ involves selecting a node denoted as $\mathcal{C}^{(0)}[t]$, from unvisited customer nodes and the depot node. Let $i_{t}$ index subgraphs such that at timestep $t$, the agent is constructing $i_{t}$-th subgraph $c_{i_{t}}^{(0)}$. If the subgraph $c_{i_{t}}^{(0)}$ is created, then the reward $r_{t}^{(0)}$ is set as $-g(c_{i_{t}}^{(0)})$; otherwise, it remains at $0$. The global partition policy, parameterized by $\theta_{G}$, is thus specified as $\pi_{\theta_{G}}(u_{t}^{(0)}|x_{t}^{(0)})$. 

At each local partition level $k\geq1$, the local partition policy is tasked with solving the sequence of subproblems obtained from $\mathcal{C}^{(k-1)}$. In this context, we use $j_{t}$ as an index for subproblems, indicating that the $j_{t}$-th subproblem denoted as $I_{j_{t}}^{(k-1)}$, is currently being addressed but remains incomplete at timestep $t$. The state $x_{t}^{(k)}\in\mathbb{X}^{(k)}$ consists of the subproblem sequence and the partial solution of $I_{j_{t}}^{(k-1)}$. The initial state distribution $\mu^{(k)}$ corresponds to the distribution of the subproblem sequence. The action $u_{t}^{(k)}\in\mathbb{U}^{(k)}$ involves selecting a node for solving $I_{j_{t}}^{(k-1)}$. When $I_{j_{t}}^{(k-1)}$ is successfully solved, the index $j_{t}$ will proceed to the next subproblem, and the reward $r_{t}^{(k)}$ is set as $-(h(\mathcal{C}^{(k)}, k, m) - h(\mathcal{C}^{(k-1)}, k, m))$ (where $m=2(j_{t}-1)$). Otherwise, the reward remains at $0$. Thus, the local partition policy parameterized by $\theta_{L}$, is defined as $\pi_{\theta_{L}}(u_{t}^{(k)}|x_{t}^{(k)})$. The objective is to maximize the sum of expected returns across levels, as defined below:
\begin{equation}
\label{equ:rl_obj}
    J(\theta_{G}, \theta_{L}) = \mathbb{E}_{\omega^{(0)}} [\sum_{t=1}^{T^{(0)}}r_{t}^{(0)}] + \cdots+
\mathbb{E}_{\omega^{(0)}}\cdots\mathbb{E}_{\omega^{(K)}} [\sum_{t=1}^{T^{(K)}}r_{t}^{(K)}],
\end{equation}
where $T^{(k)}$ and $\omega^{(k)}$ denote the horizon and the trajectory at level $k$.
\end{proposition}

\begin{proof}
We aim to prove that, at any partition level $k \geq 0$, the definitions of state and action, as presented in Proposition~\ref{prop:mlmdp}, are essential for preserving the Markovian property of the corresponding MDP. At the global partition level, the partial partition solution evolves from $\mathcal{C}^{0}[0:t-1]$ to $\mathcal{C}^{0}[0:t]$ by selecting an unvisited node $\mathcal{C}^{0}[t]$. It is evident that the dynamics remains Markovian at this level. At the local partition level $k \geq 1$, the state consists of the subproblem sequence and the partial solution of the subproblem $I_{j_{t}}^{k-1}$ addressed at time step $t$. When solving $I_{j_{t}}^{k-1}$, executing the action $u_{t}^{k}$ augments the partial solution of $I_{j_{t}}^{k-1}$ with an unvisited node for the subsequent partial solution of $I_{j_{t}}^{k-1}$. Apart from $I_{j_{t}}^{k-1}$, the remaining subproblems remain unaltered. Upon the completion of $I_{j_{t}}^{k-1}$, the partial solution of $I_{j_{t+1}}^{k-1}$ begins to be addressed. Consequently, the dynamics of each local partition level retain their Markovian nature.
\end{proof}

\section{Related Works}

\textbf{Constructive Methods.} The learning-based constructive method aims to develop an efficient and (near-)optimal end-to-end neural solver for combinatorial optimization problems (COPs). Among them, the PointerNet~\cite{vinyals2015pointer} and the Transformer~\cite{kool2018attention} are two commonly employed architectures trained using either SL~\cite{vinyals2015pointer} or RL~\cite{nazari2018reinforcement, kool2018attention}, to progressively infer complete solutions with the aid of the autoregressive mechanism. Additionally, some methods consider inherent properties in VRPs, such as multiple optima~\cite{kwon2020pomo} and symmetry~\cite{kim2022sym}, to enhance the quality of solutions. However, the delayed rewards in the training of RL policies lead to the high GPU memory demands for gradient backpropagation. Thus, SL-driven policies, such as BQ~\cite{drakulic2024bq}, LEHD~\cite{luo2024neural} and SIL~\cite{luo2024self}, have resurfaced to alleviate training difficulties and moderately improve generalization. Moreover, some approaches employ techniques such as meta-learning~\cite{son2023meta, zhou2023towards, manchanda2022generalization, qiu2022dimes}, knowledge distillation~\cite{bi2022learning}, or ensemble learning~\cite{gao2023towards, jiang2024ensemble, grinsztajn2023winner} to enhance the generalization of neural solvers. However, these constructive methods might still experience performance deterioration when encountering substantial distribution or scale shifts.

\noindent \textbf{Iterative Methods.} In comparison to constructive methods, iterative methods offer the benefit of consistently improving a given solution until convergence. Both L2I~\cite{lu2019learning} and NeuRewriter~\cite{chen2019learning} utilize RL policies to choose from local improvement operators to refine the given initial solutions. Likewise,~\citet{hottung2020neural} interleaves the use of heuristic destroy operators and a set of learning-based repair policies to generate a new solution. Moreover, DACT~\cite{ma2021learning} focuses on the expressive representation of solution encodings provided to the RL policy. Additionally, both NeuralLKH~\cite{xin2021neurolkh} and Neural k-Opt~\cite{ma2024learning} utilize the RL policy to substitute the heuristic rule for edge exchanges in k-opt algorithms. However, these iterative methods trade efficiency for improved performance and still rely heavily on well-crafted rules.

\noindent \textbf{Divide-and-Conquer Methods.} The divide-and-conquer paradigm can leverage local topological features that are insensitive to distribution or scale shifts to mitigate performance deterioration. ~\citet{fu2021generalize}, ~\citet{kim2021learning} and~\citet{cheng2023select} attempt to transfer a standard neural solver for larger instances by sampling multiple small-scale subgraphs using heuristic rules. In contrast, both L2D~\cite{li2021learning} and RGB~\cite{zong2022rbg} learn a policy to choose among heuristically constructed subgraphs for iterative enhancement. However, the used heuristic rules may lead to solutions being trapped in local optima. Unlike the above methods, TAM~\cite{hou2023generalize}, GLOP~\cite{ye2024glop} and UDC~\cite{zheng2024udc} opt to use a learning-based policy to globally partition the entire instance into subproblems that are solved by the pretrained local construction policy. In addition,~\citet{pan2023h} resort to a hierarchical RL model where a local policy solves subproblems assigned by a jointly trained global policy. However, these neural partition policies may suffer performance degradation due to compounded errors during the partition process.

\section{Toy Examples}
The toy examples for the overall HLGP framework, the RL-driven HLGP training framework and the SL-driven HLGP training framework
are illustrated in Figure~\ref{fig:hlgp_toy_example}, Figure~\ref{fig:rl_hlgp_toy_example}, and Figure~\ref{fig:sl_hlgp_toy_example}, respectively.

\begin{figure*}[t]
    \centering
    \includegraphics[width=1\textwidth]{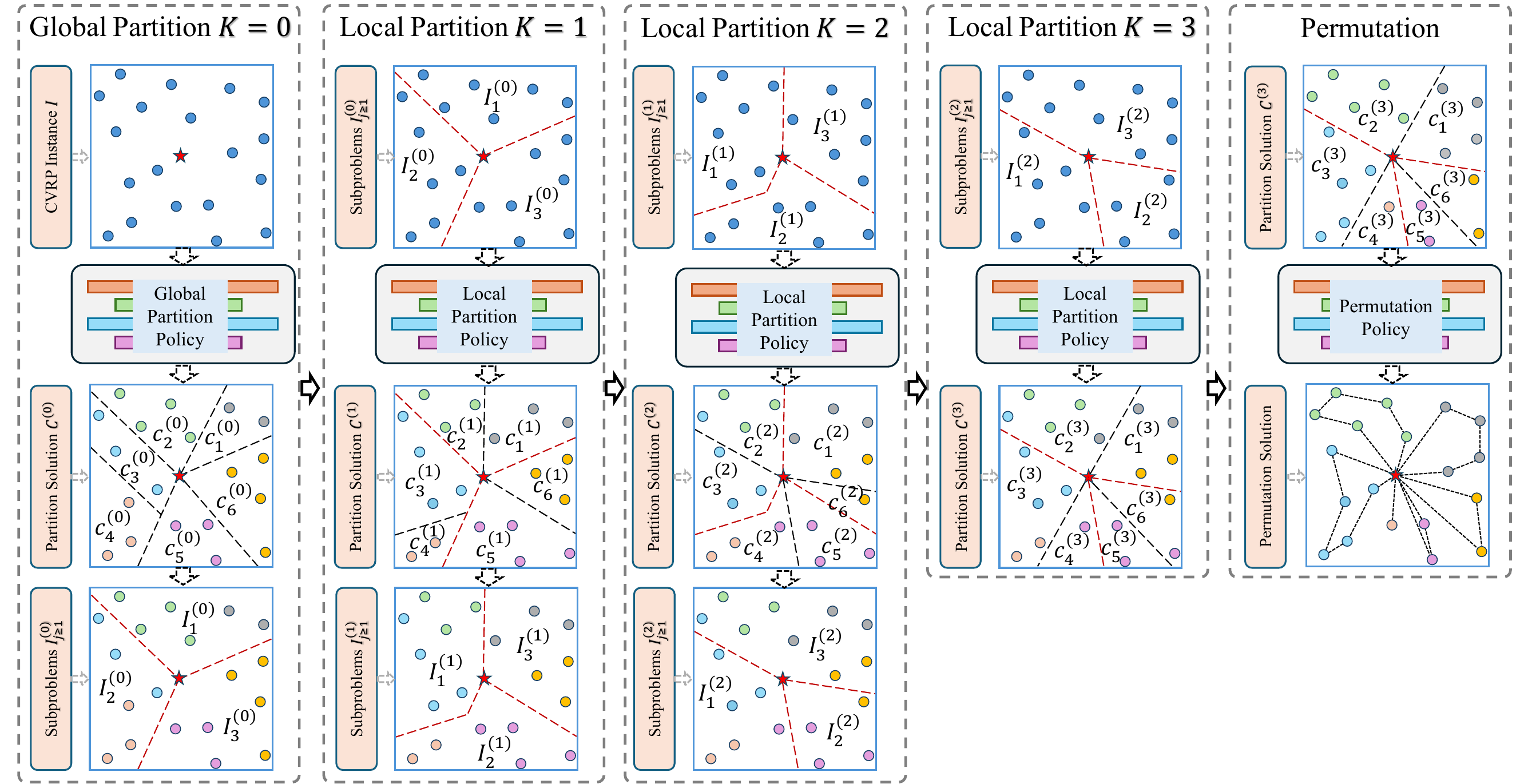}
    \caption{The toy example of the overall HLGP framework.
    }
    \label{fig:hlgp_toy_example}
\end{figure*}

\begin{figure*}[t]
    \centering
    \includegraphics[width=1\textwidth]{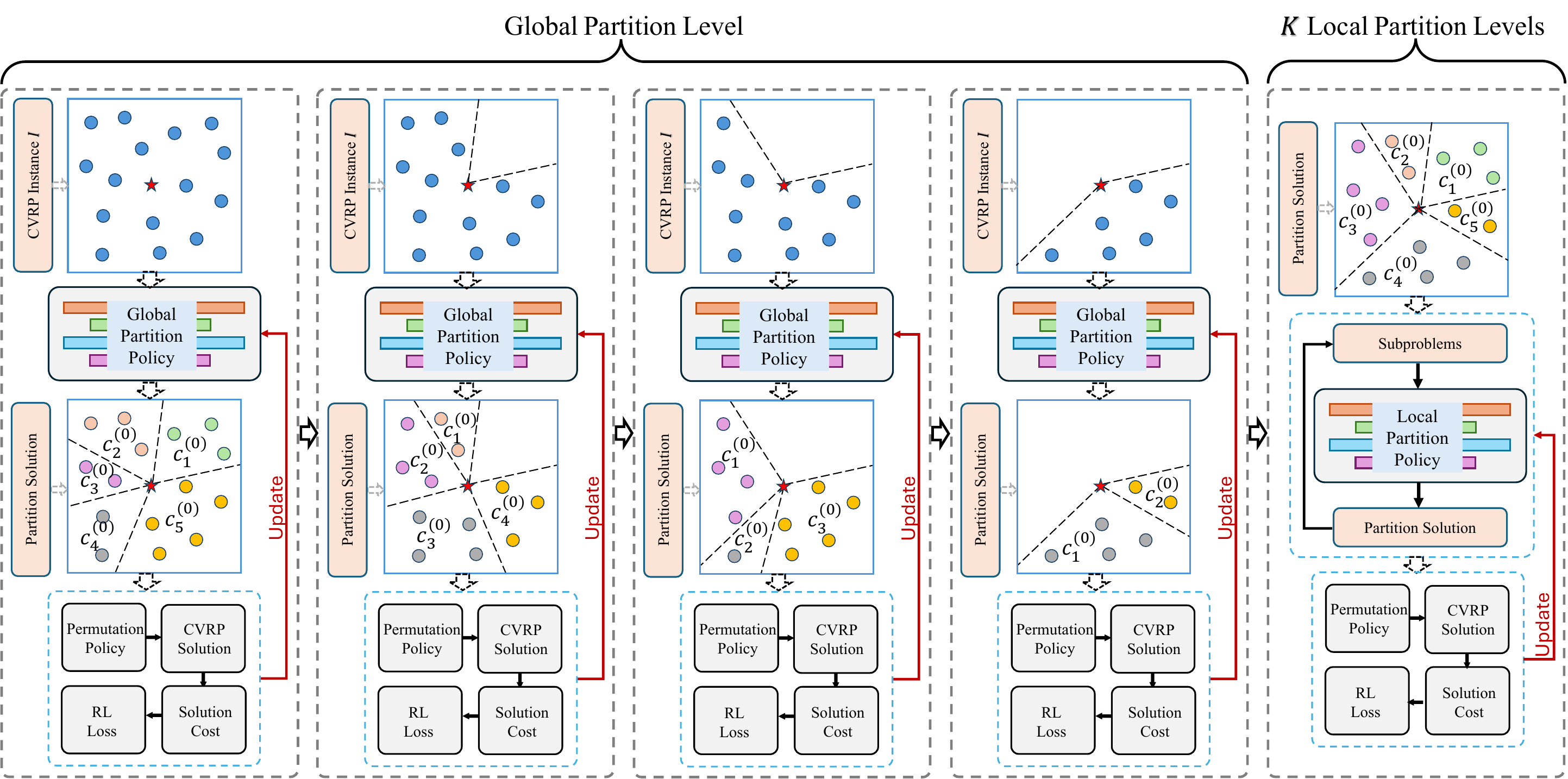}
    \caption{The toy example of the RL-driven HLGP training framework.
    }
    \label{fig:rl_hlgp_toy_example}
\end{figure*}

\begin{figure*}[t]
    \centering
    \includegraphics[width=1\textwidth]{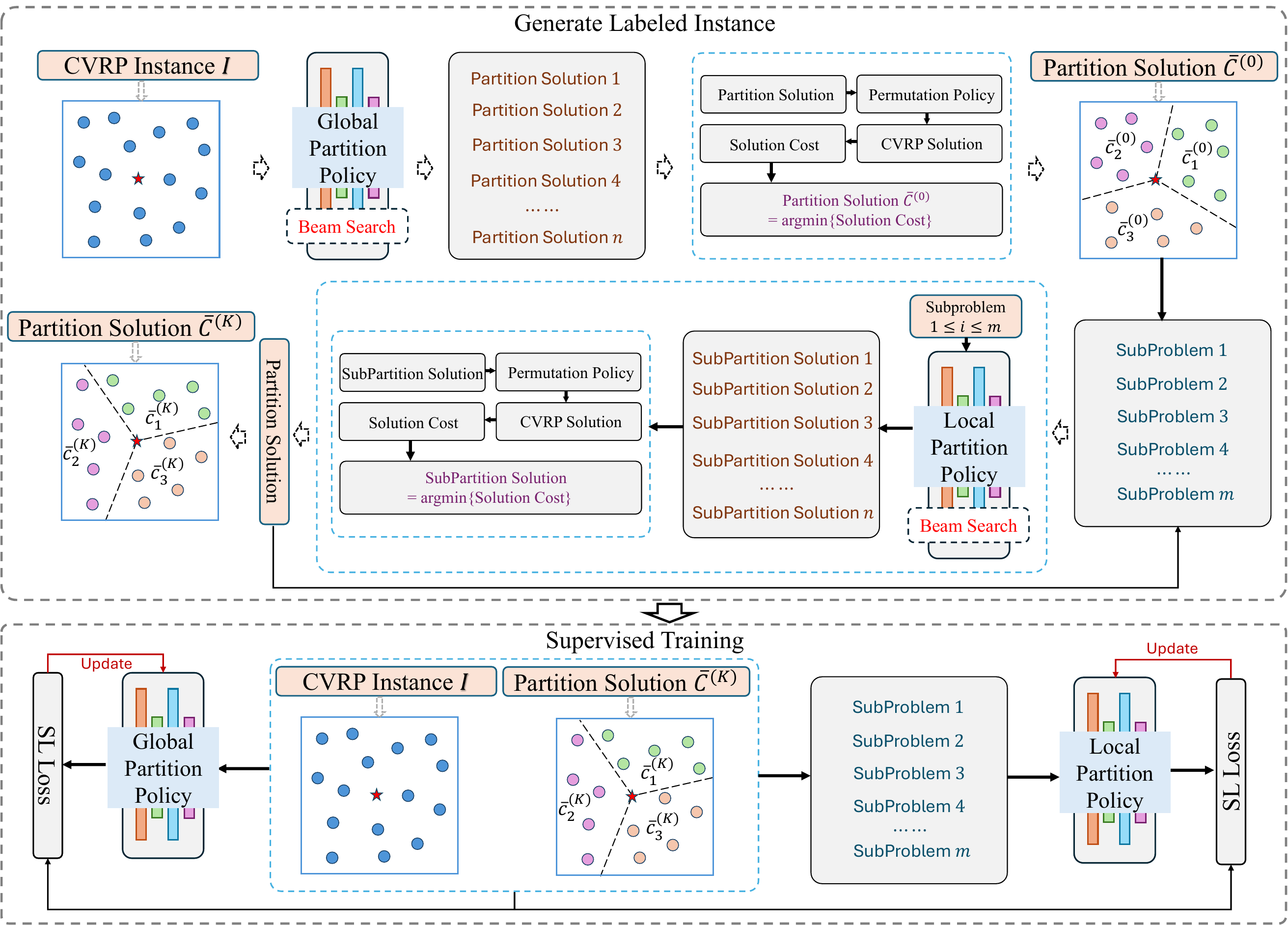}
    \caption{The toy example of the SL-driven HLGP training framework.
    }
    \label{fig:sl_hlgp_toy_example}
\end{figure*}

\newpage

\begin{figure*}[t]
\centering
\rotatebox{90}{~~~~~~~~~~~~\scriptsize{CVRP1000+Uniform}}
\subfigure[\shortstack{Cost=59.37 \\ glob.}]{\label{fig:u_rl_1}\includegraphics[width=0.23\textwidth]{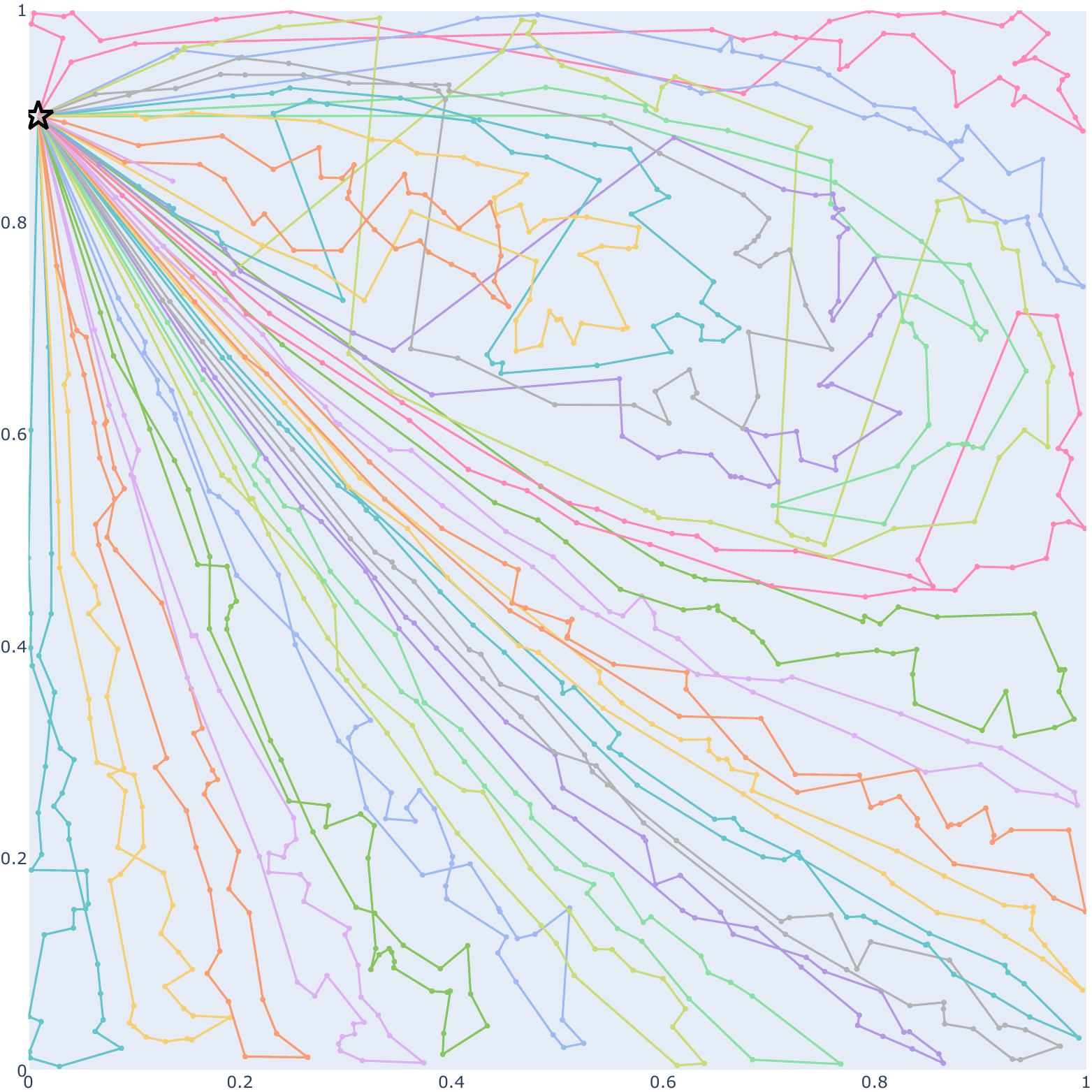}}
\subfigure[\shortstack{Cost=53.99 \\ glob.+loc.}]{\label{fig:u_rl_2}\includegraphics[width=0.23\textwidth]{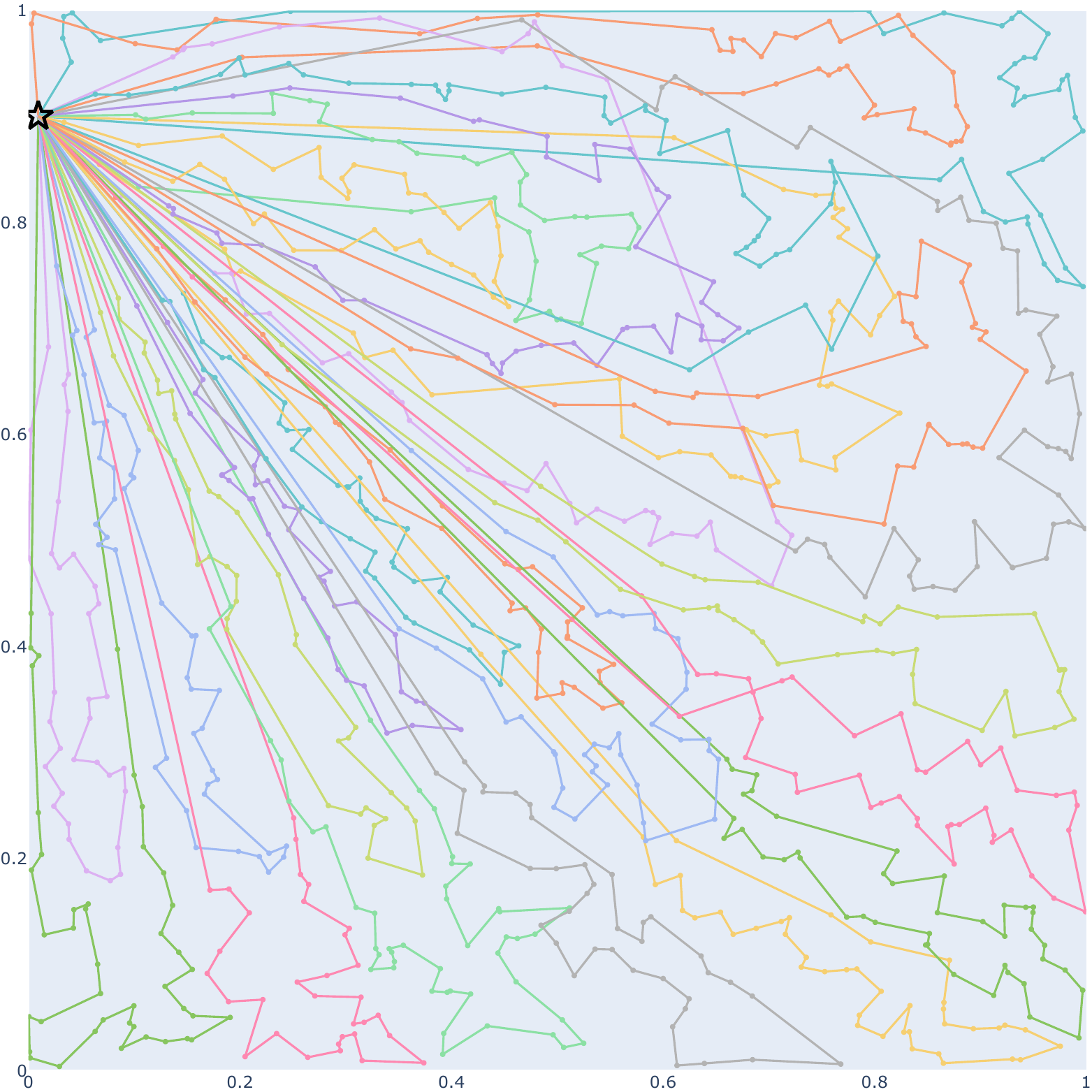}}
\subfigure[\shortstack{Cost=56.49 \\ glob.+subp.}]
{\label{fig:u_rl_3}\includegraphics[width=0.23\textwidth]{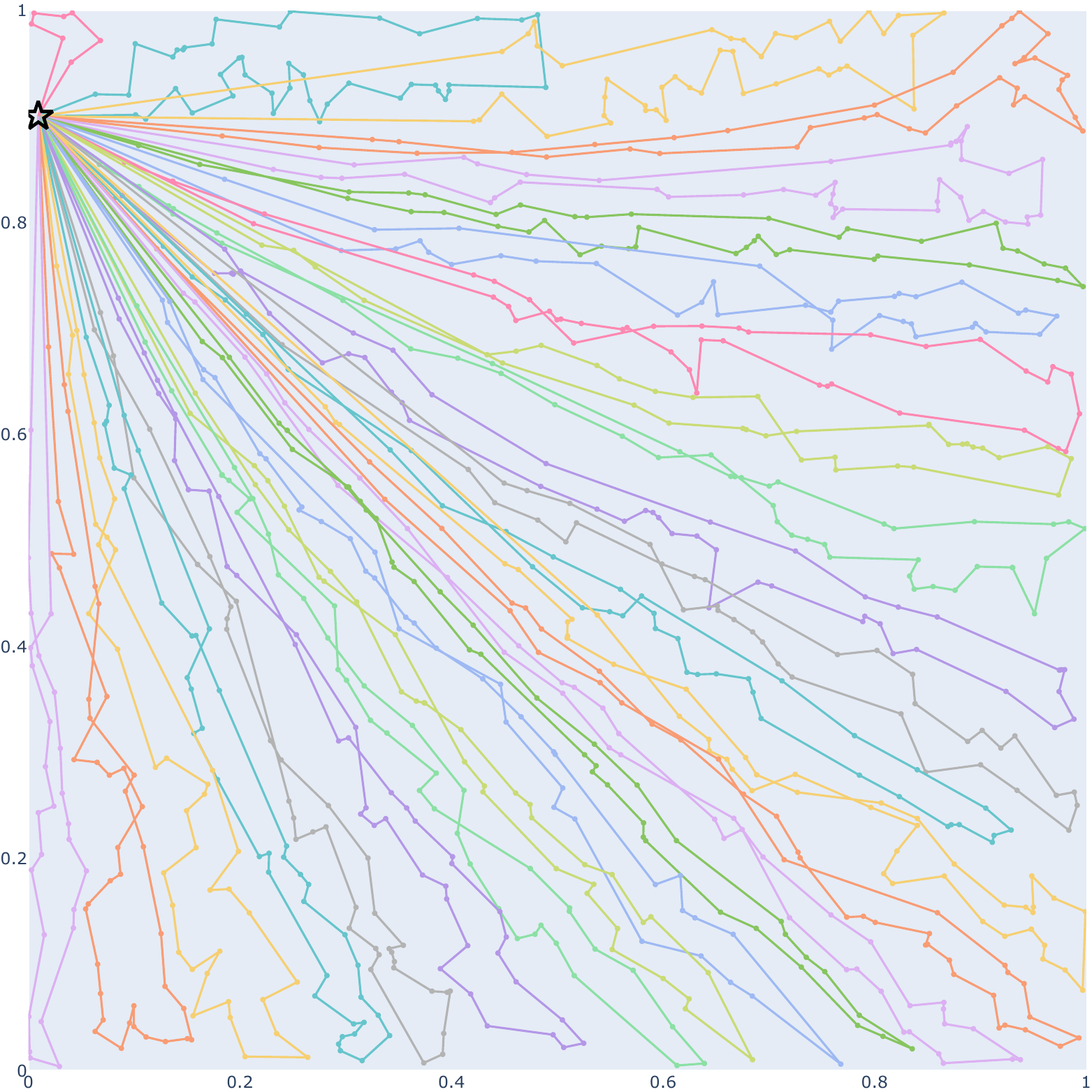}}
\subfigure[\shortstack{Cost=52.73 \\ glob.+loc.+subp.}]
{\label{fig:u_rl_4}\includegraphics[width=0.23\textwidth]{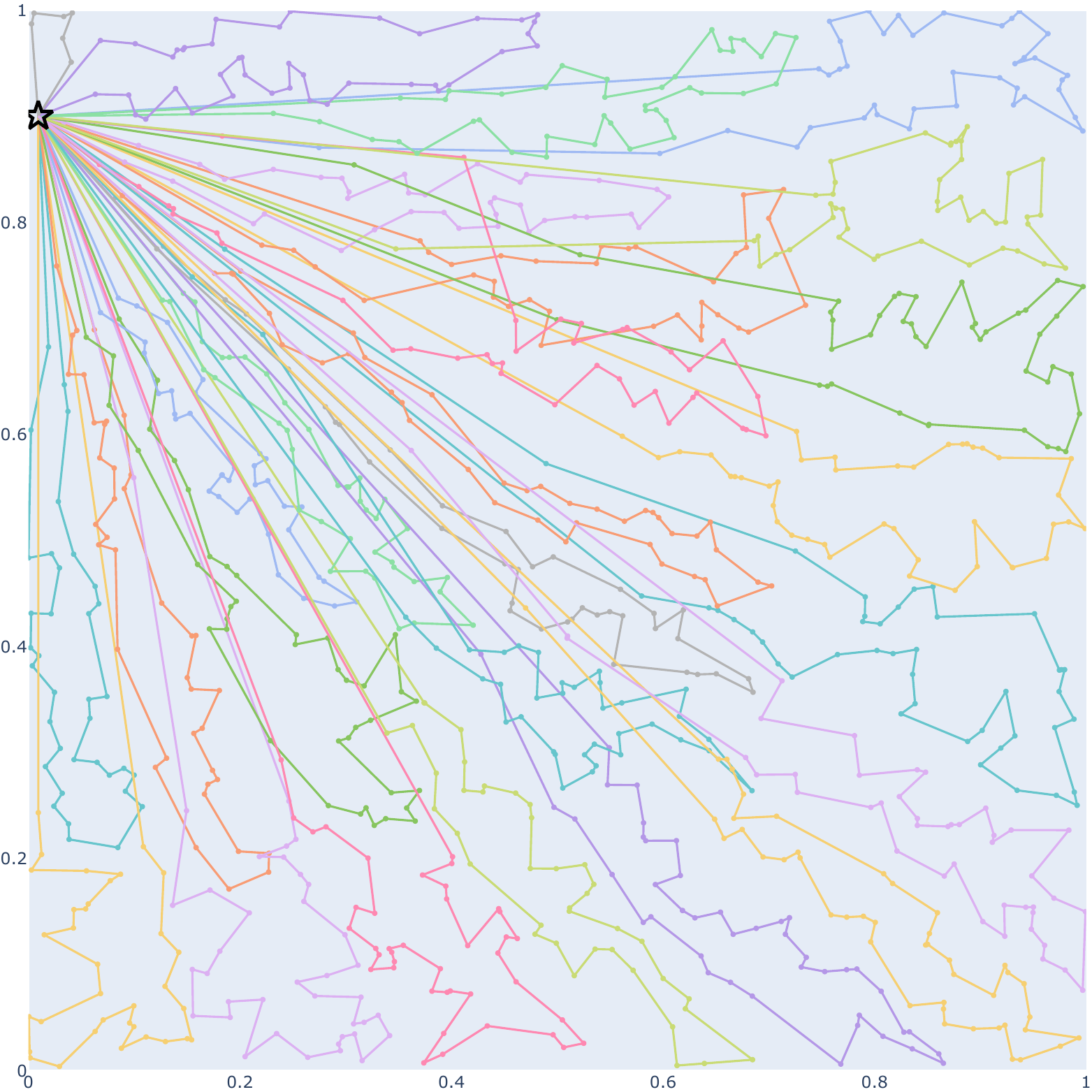}}

\rotatebox{90}{~~~~~~~~~~~~\scriptsize{CVRP1000+Gaussian}}
\subfigure[\shortstack{Cost=51.61 \\ glob.}]{\label{fig:g_rl_1}\includegraphics[width=0.23\textwidth]{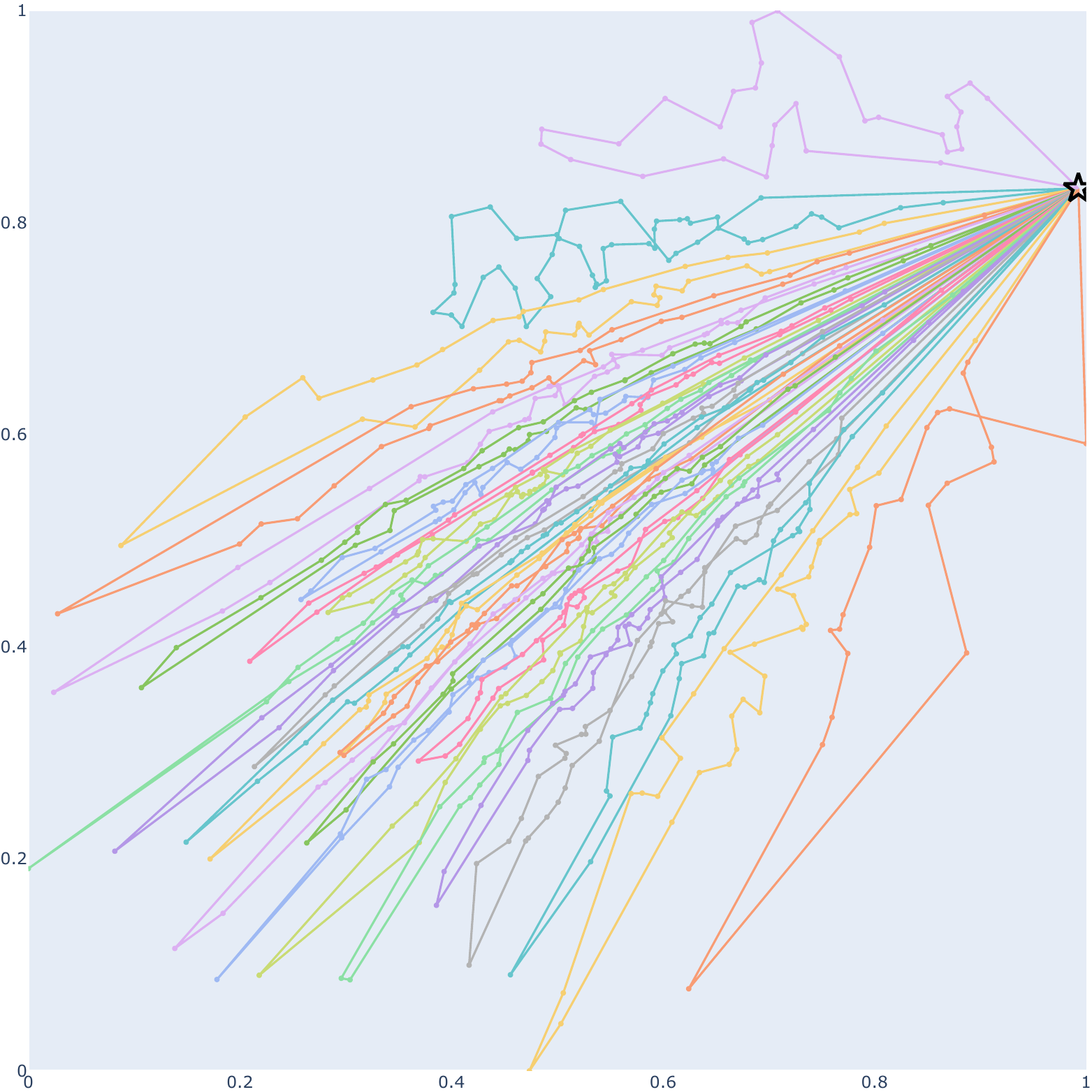}}
\subfigure[\shortstack{Cost=43.59 \\ glob.+loc.}]{\label{fig:g_rl_2}\includegraphics[width=0.23\textwidth]{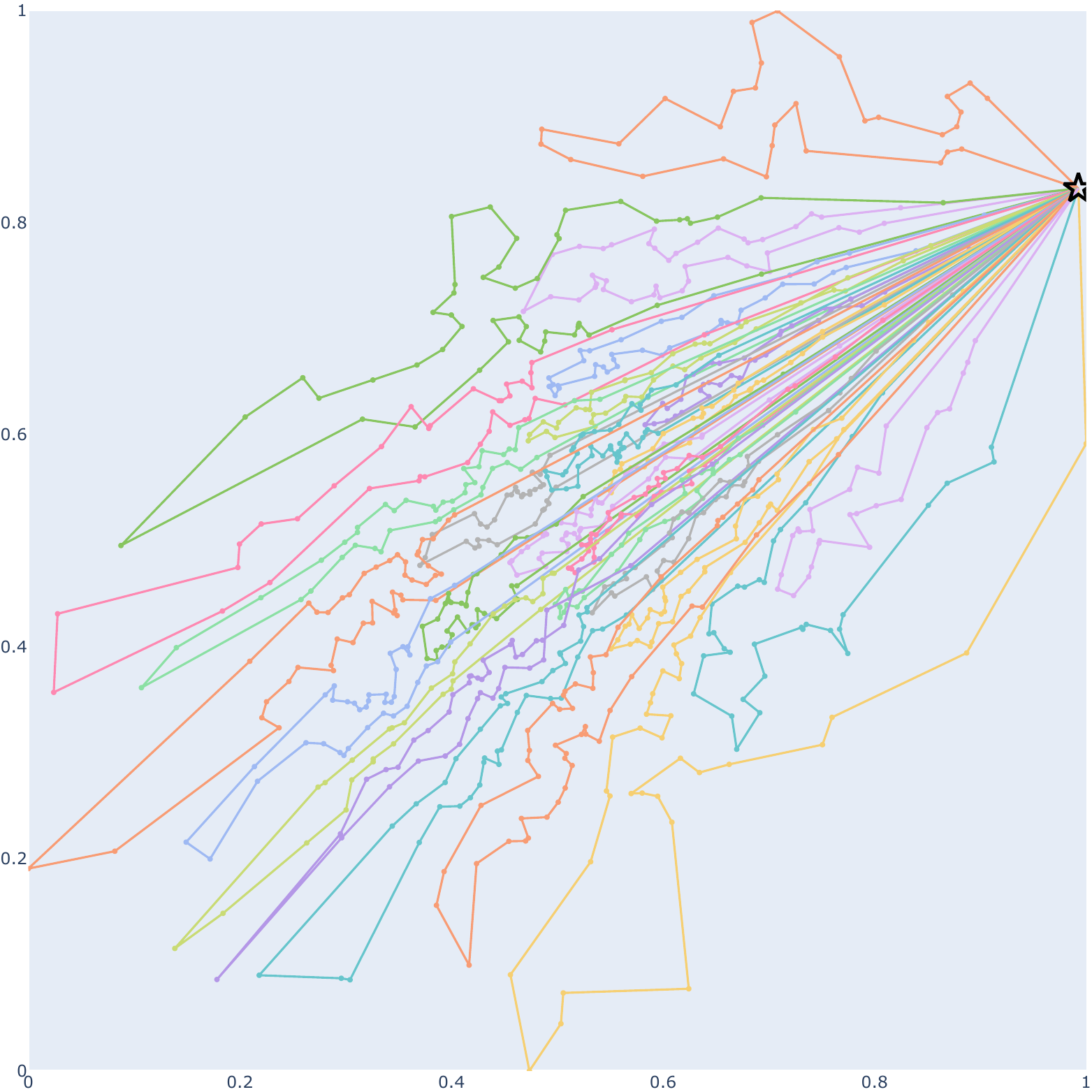}}
\subfigure[\shortstack{Cost=49.08 \\ glob.+subp.}]
{\label{fig:g_rl_3}\includegraphics[width=0.23\textwidth]{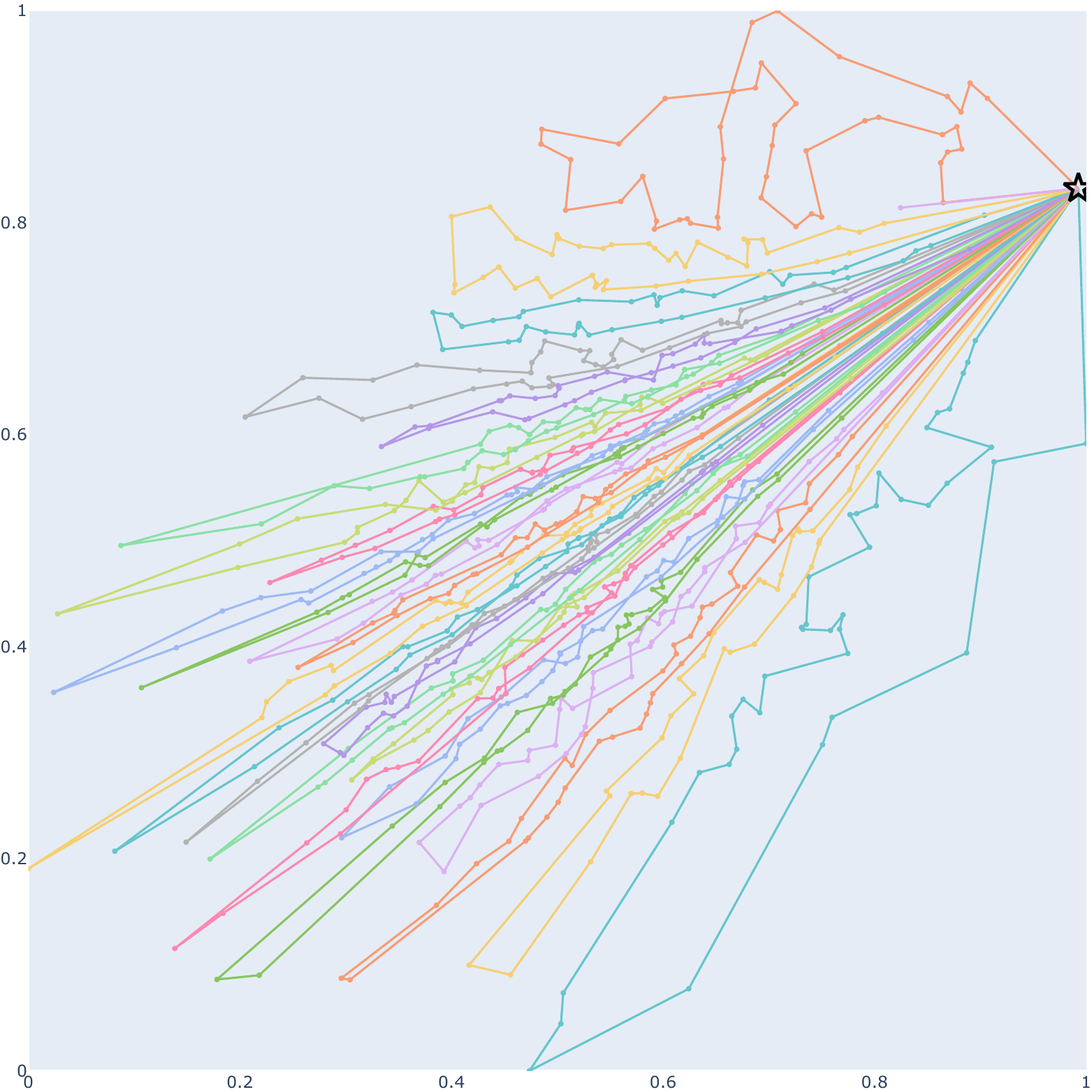}}
\subfigure[\shortstack{Cost=42.18 \\ glob.+loc.+subp.}]
{\label{fig:g_rl_4}\includegraphics[width=0.23\textwidth]{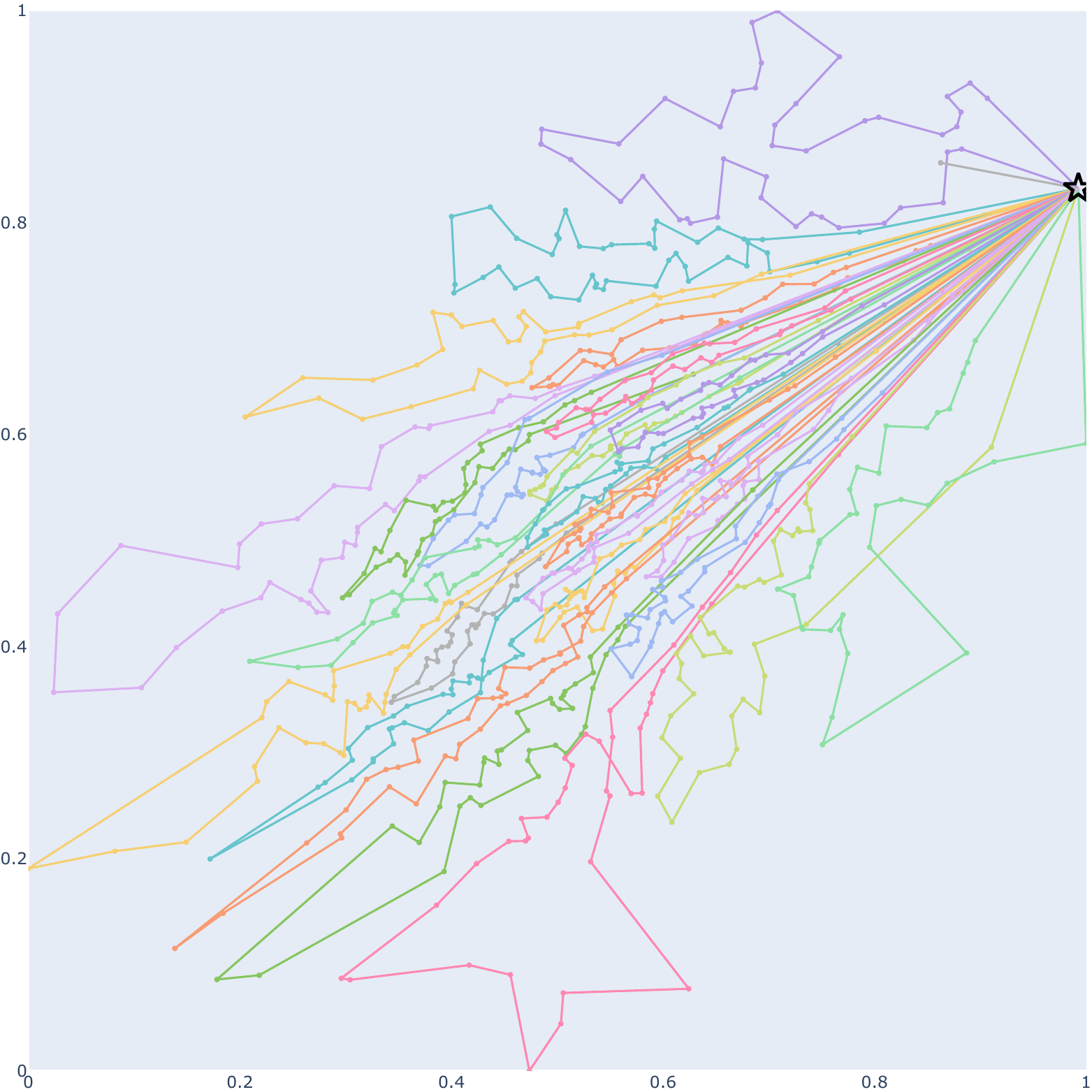}}

\rotatebox{90}{~~~~~~~~~~~~\scriptsize{CVRP1000+Explosion}}
\subfigure[\shortstack{Cost=46.55 \\ glob.}]{\label{fig:e_rl_1}\includegraphics[width=0.23\textwidth]{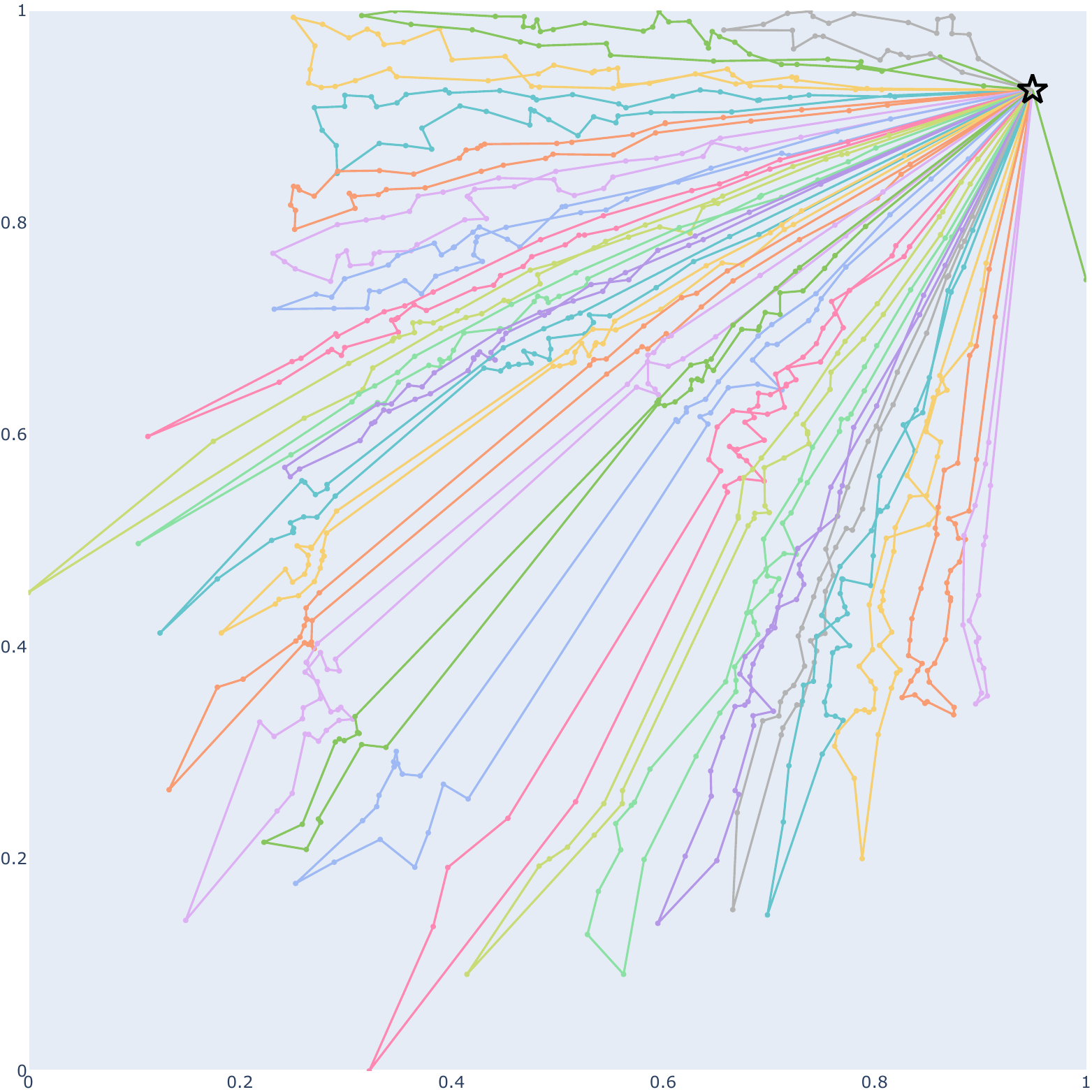}}
\subfigure[\shortstack{Cost=40.90 \\ glob.+loc.}]{\label{fig:e_rl_2}\includegraphics[width=0.23\textwidth]{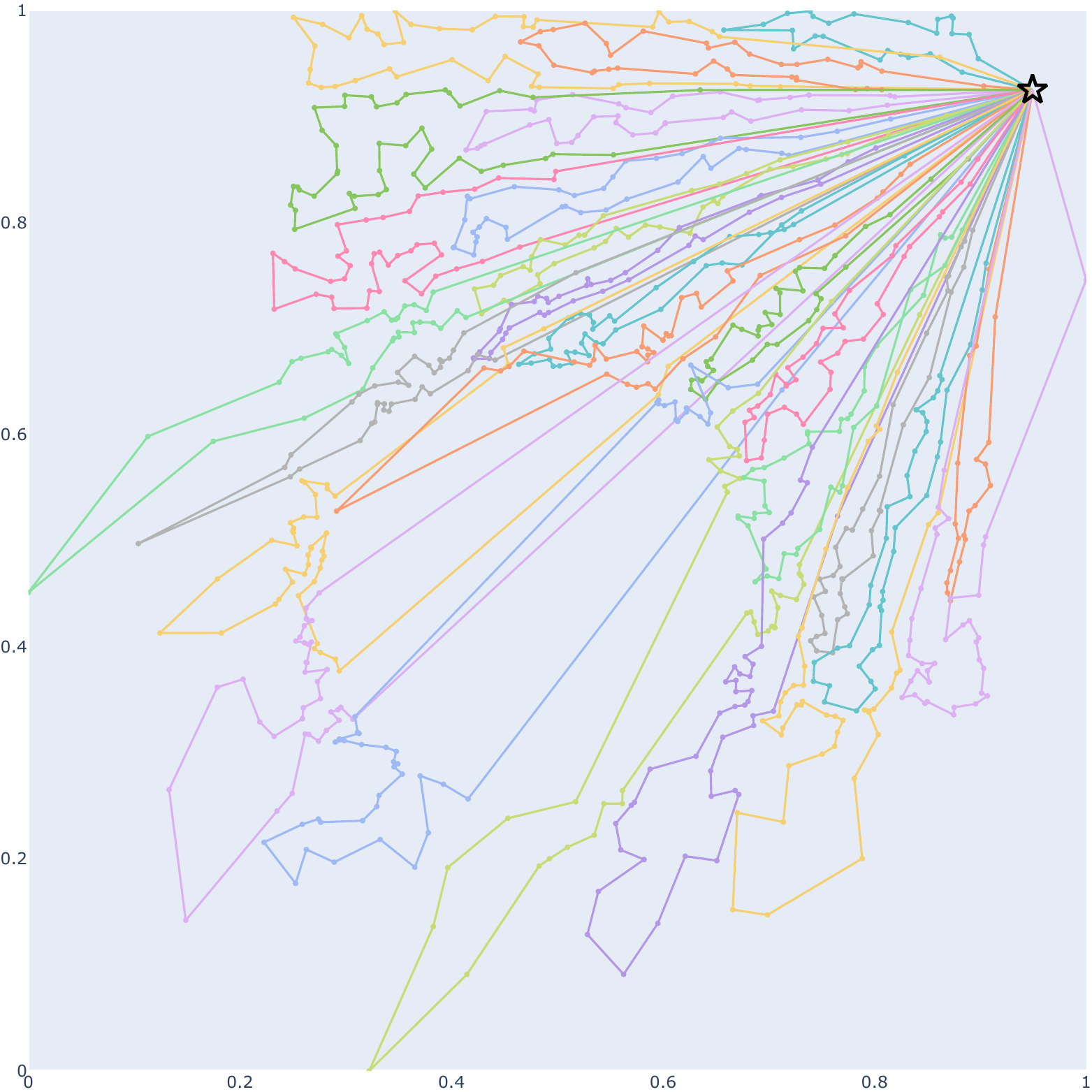}}
\subfigure[\shortstack{Cost=43.75 \\ glob.+subp.}]
{\label{fig:e_rl_3}\includegraphics[width=0.23\textwidth]{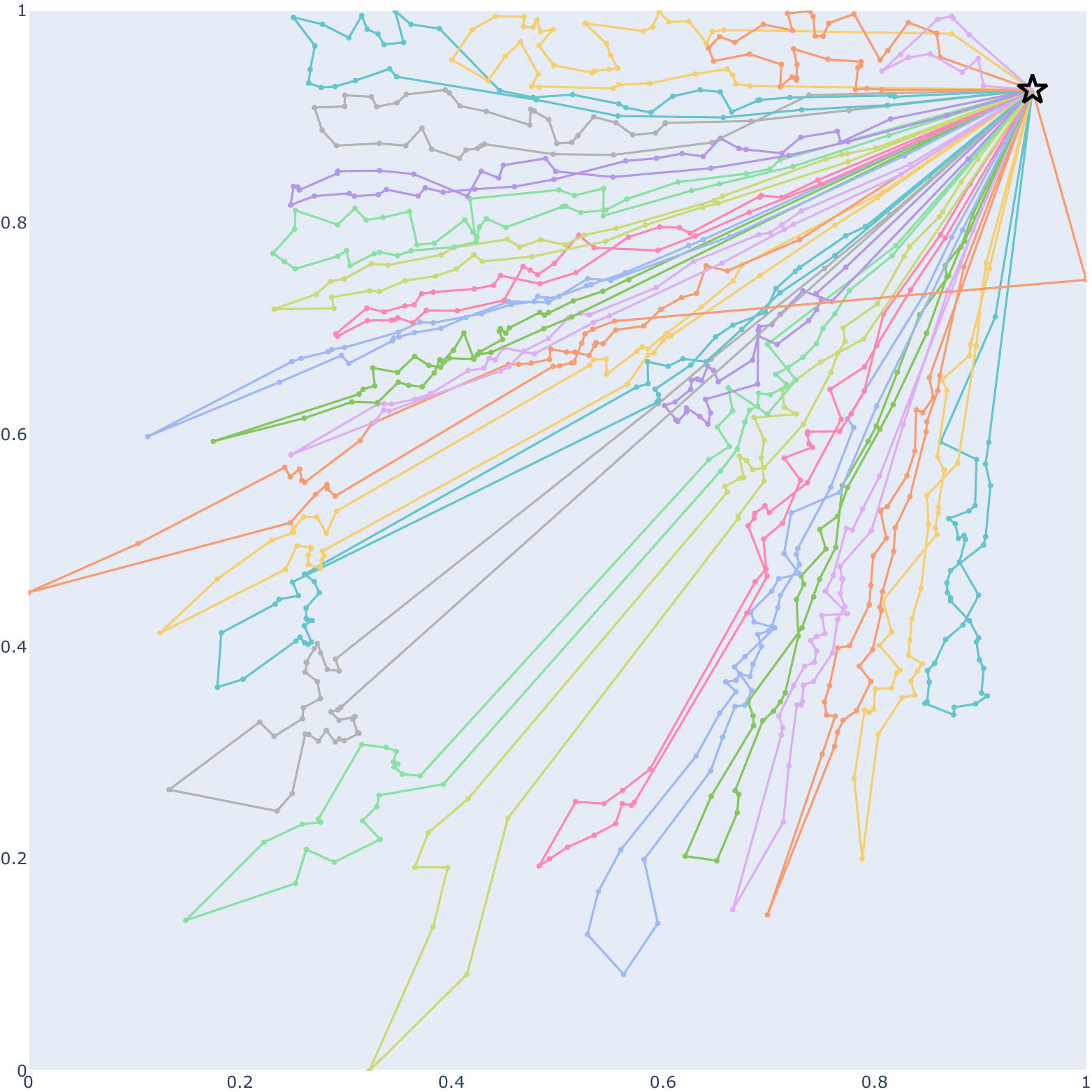}}
\subfigure[\shortstack{Cost=39.73 \\ glob.+loc.+subp.}]
{\label{fig:e_rl_4}\includegraphics[width=0.23\textwidth]{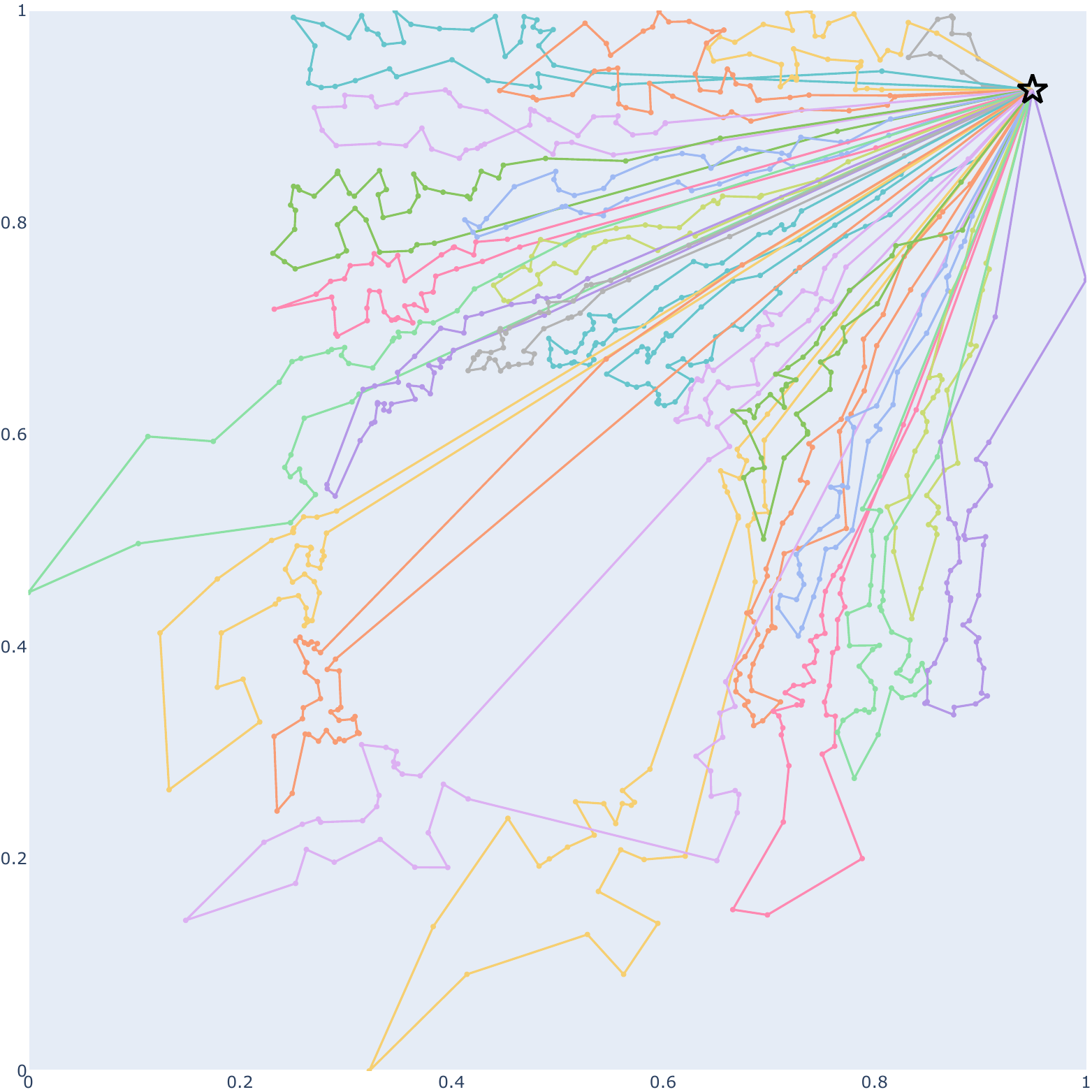}}

\rotatebox{90}{~~~~~~~~~~~~\scriptsize{CVRP1000+Rotation}}
\subfigure[\shortstack{Cost=51.28 \\ glob.}]{\label{fig:r_rl_1}\includegraphics[width=0.23\textwidth]{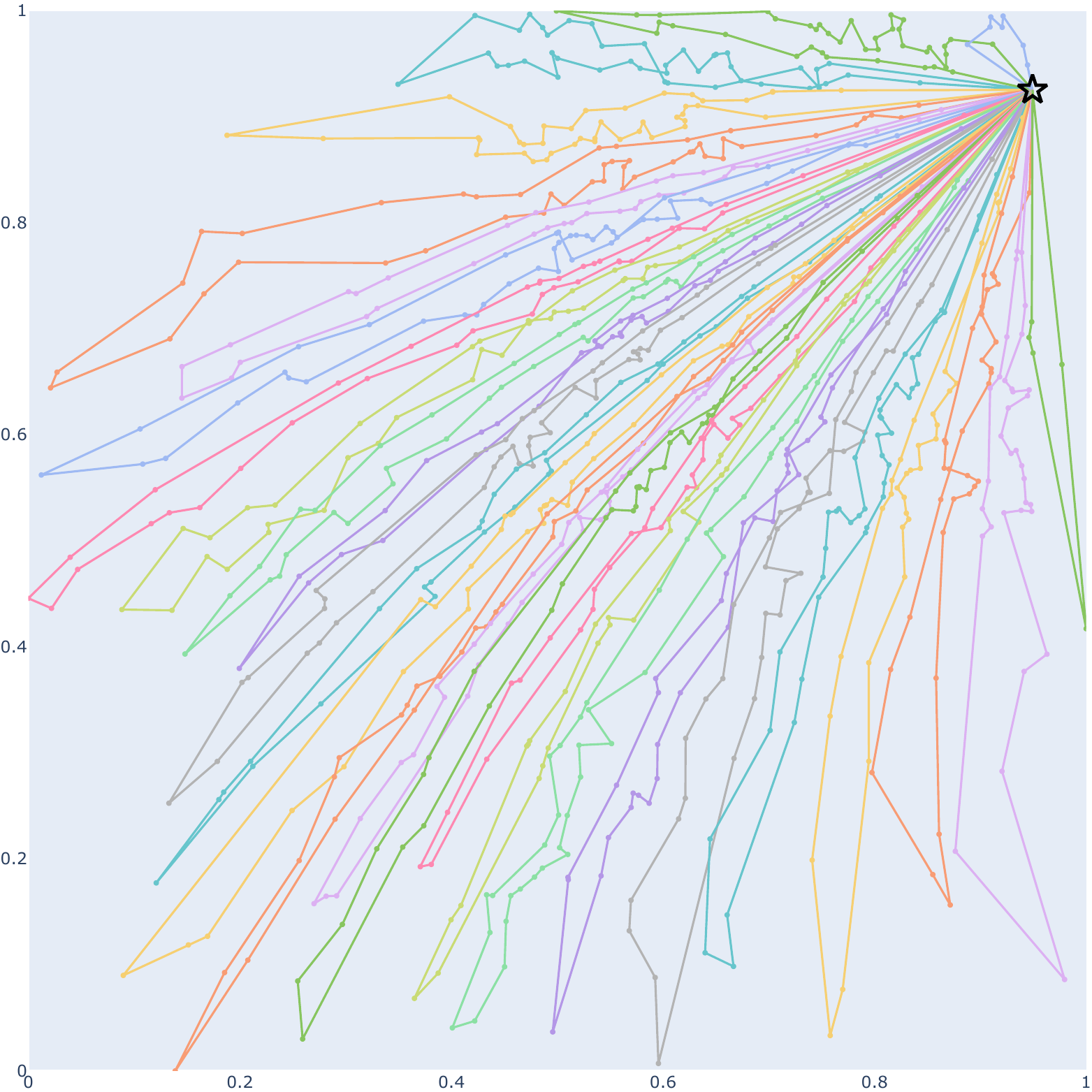}}
\subfigure[\shortstack{Cost=42.27 \\ glob.+loc.}]{\label{fig:r_rl_2}\includegraphics[width=0.23\textwidth]{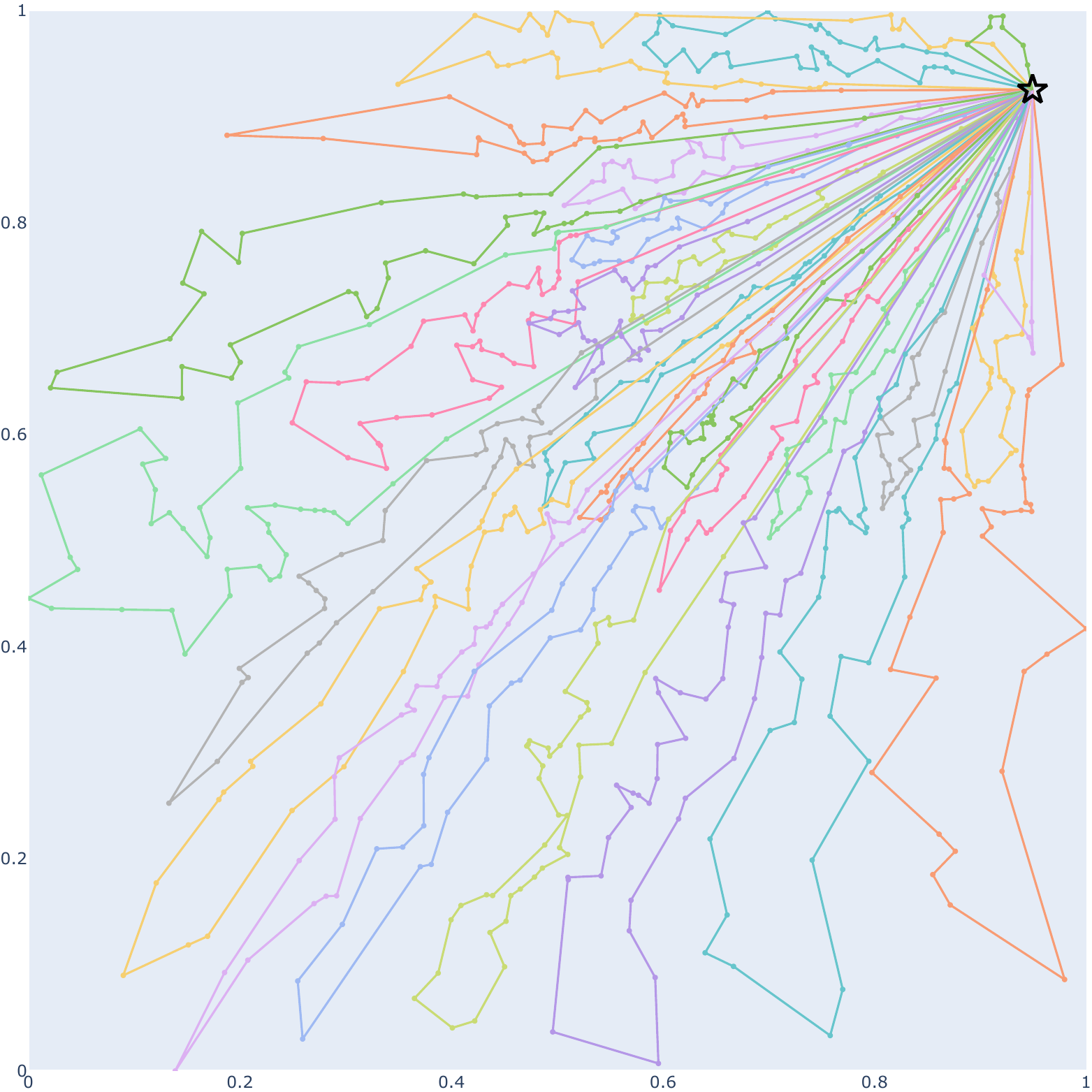}}
\subfigure[\shortstack{Cost=48.65 \\ glob.+subp.}]
{\label{fig:r_rl_3}\includegraphics[width=0.23\textwidth]{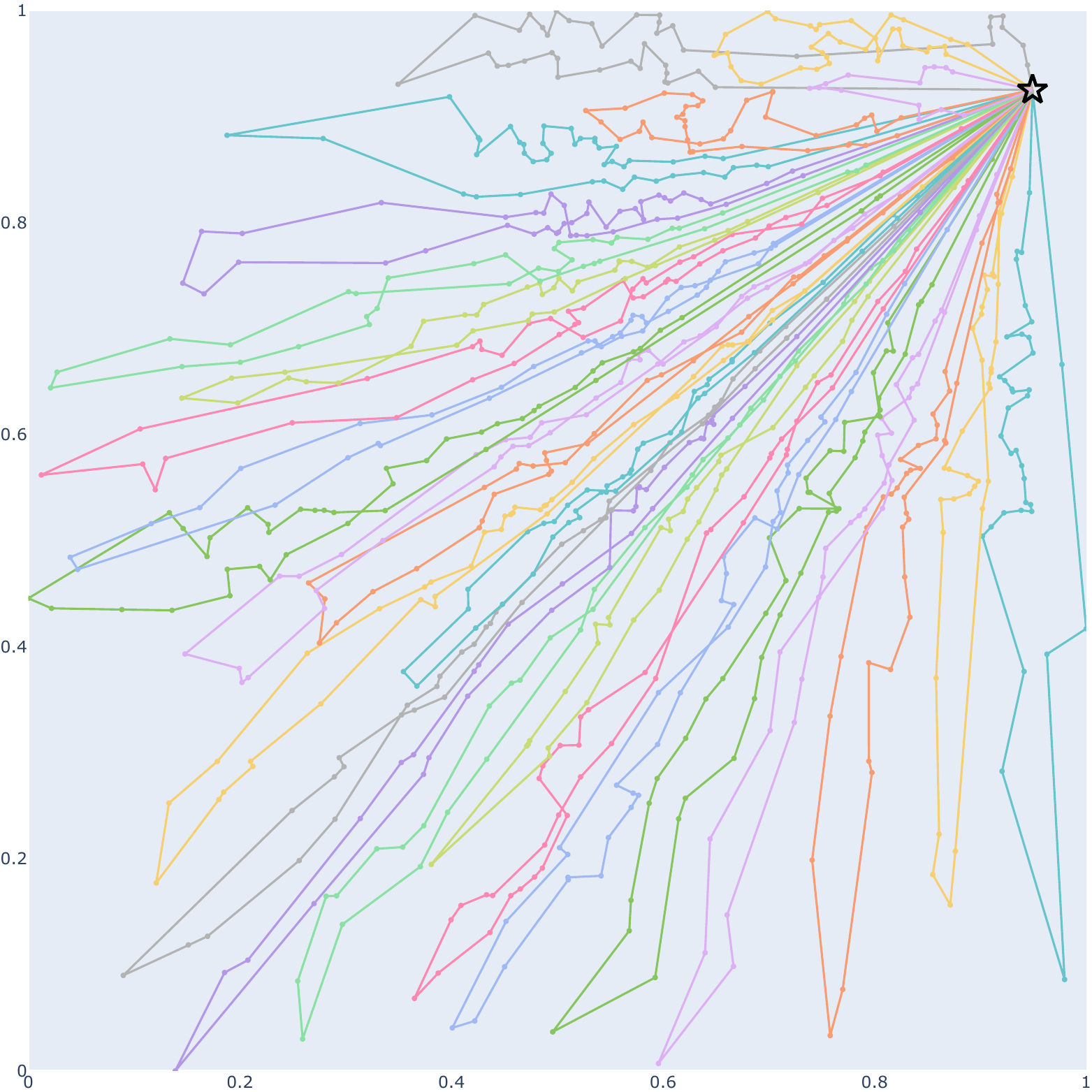}}
\subfigure[\shortstack{Cost=40.91 \\ glob.+loc.+subp.}]
{\label{fig:r_rl_4}\includegraphics[width=0.23\textwidth]{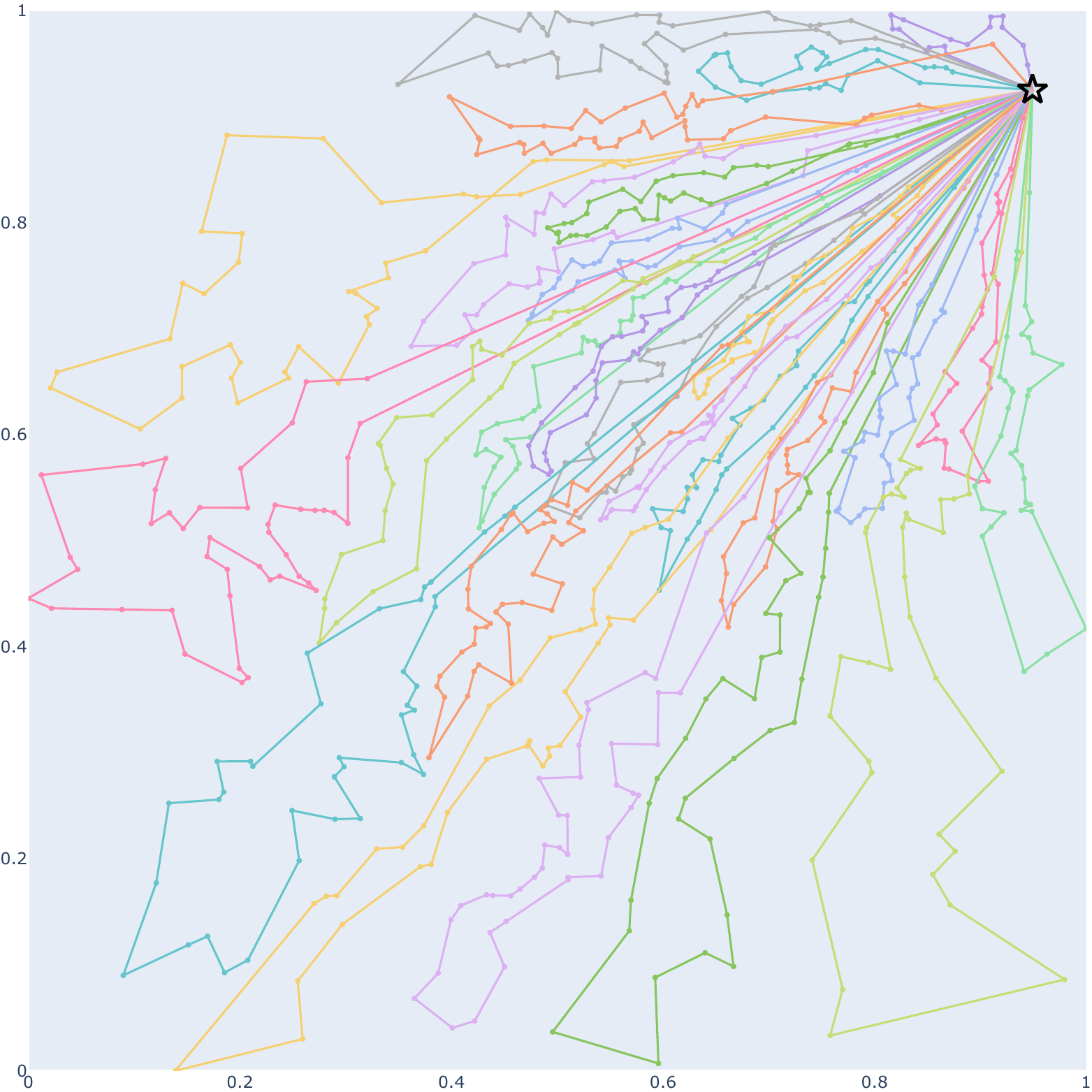}}

\caption{visualization of RL-driven HLGP routes.}
\label{fig:rl_visual}
\end{figure*}

\begin{figure*}[t]
\centering
\rotatebox{90}{~~~~~~~~~~~~\scriptsize{CVRP1000+Uniform}}
\subfigure[\shortstack{Cost=49.70 \\ SS.+glob.}]{\label{fig:u_sl_1}\includegraphics[width=0.23\textwidth]{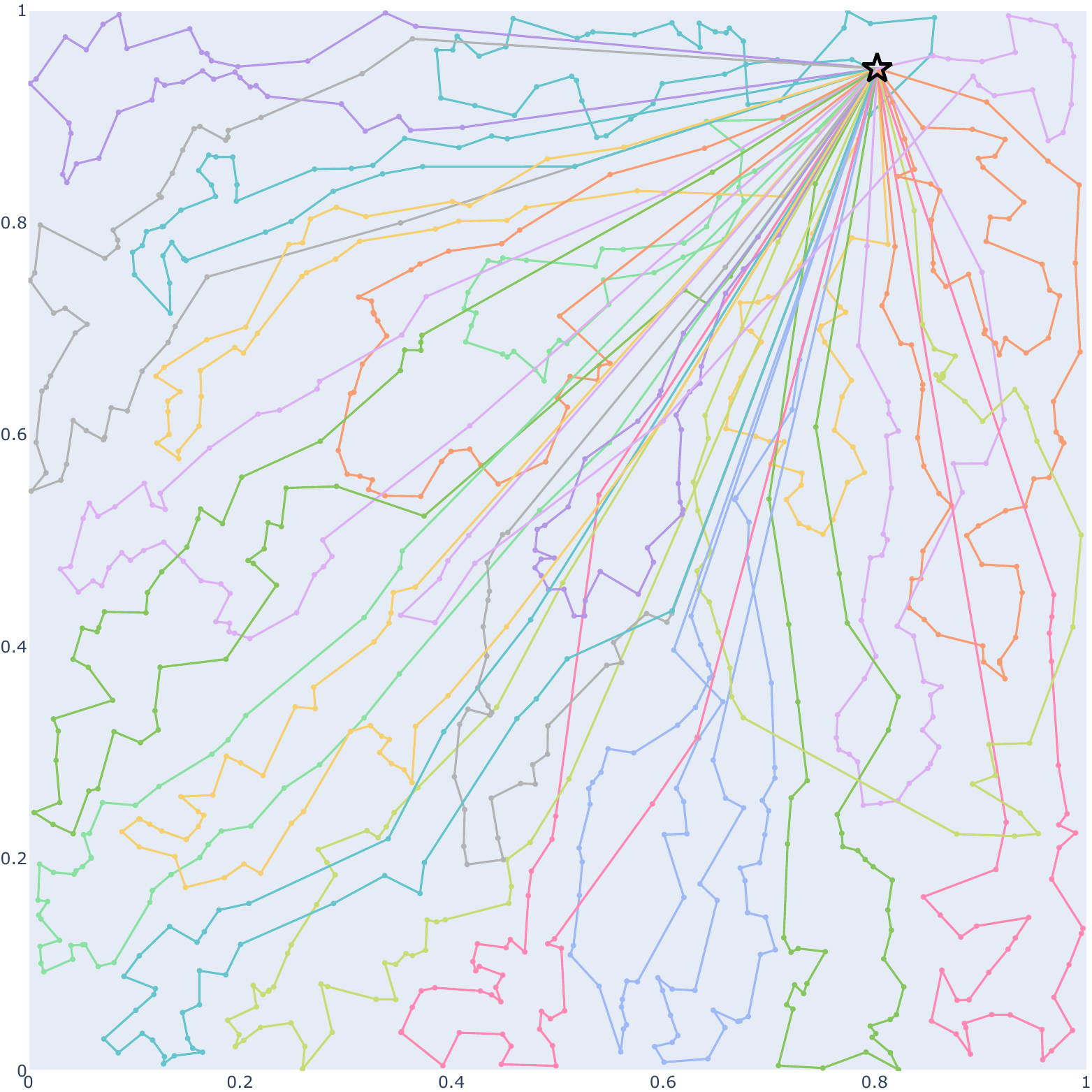}}
\subfigure[\shortstack{Cost=48.70 \\ SS.+glob.+loc.}]{\label{fig:u_sl_2}\includegraphics[width=0.23\textwidth]{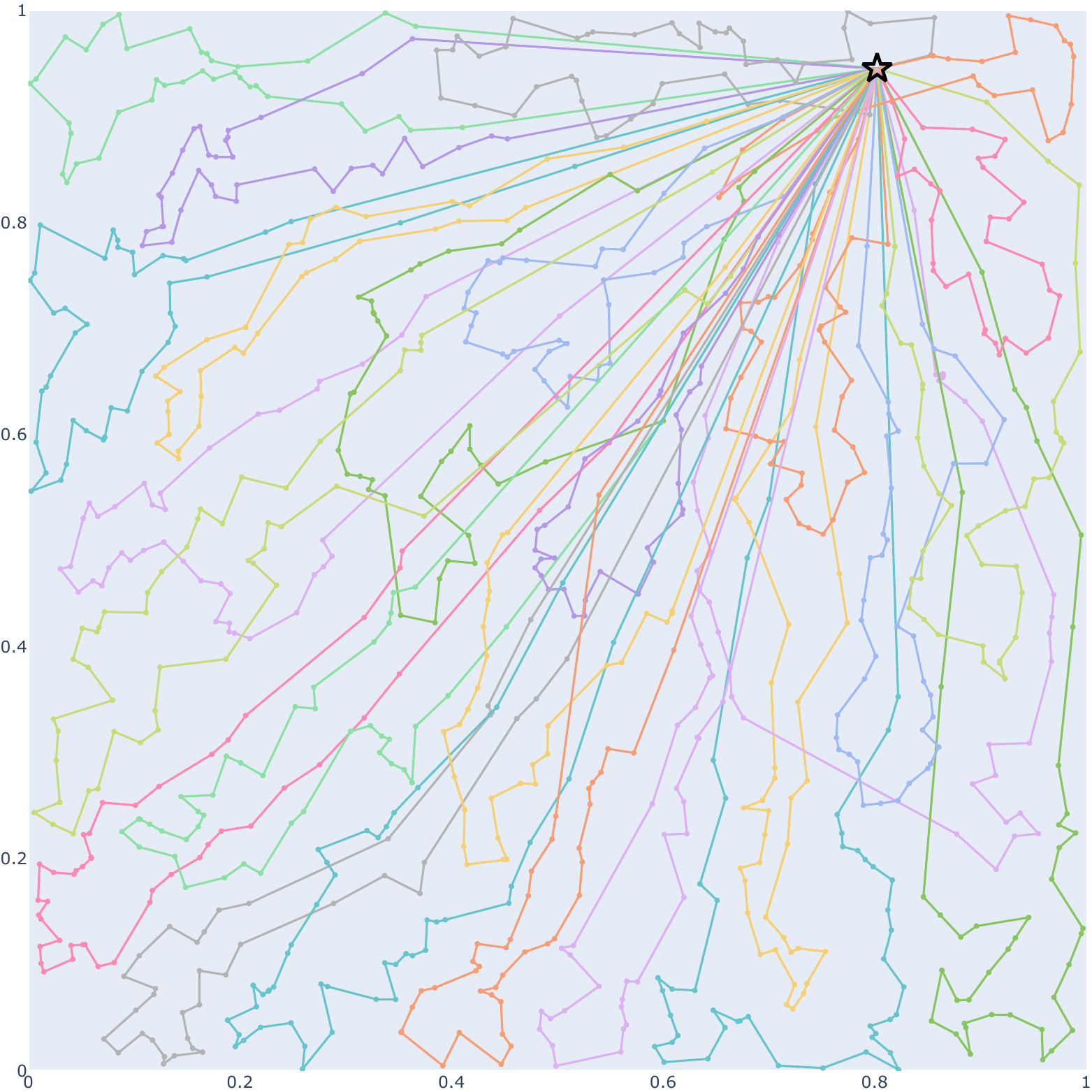}}
\subfigure[\shortstack{Cost=47.99 \\ SS.+LS.+glob.}]
{\label{fig:u_sl_3}\includegraphics[width=0.23\textwidth]{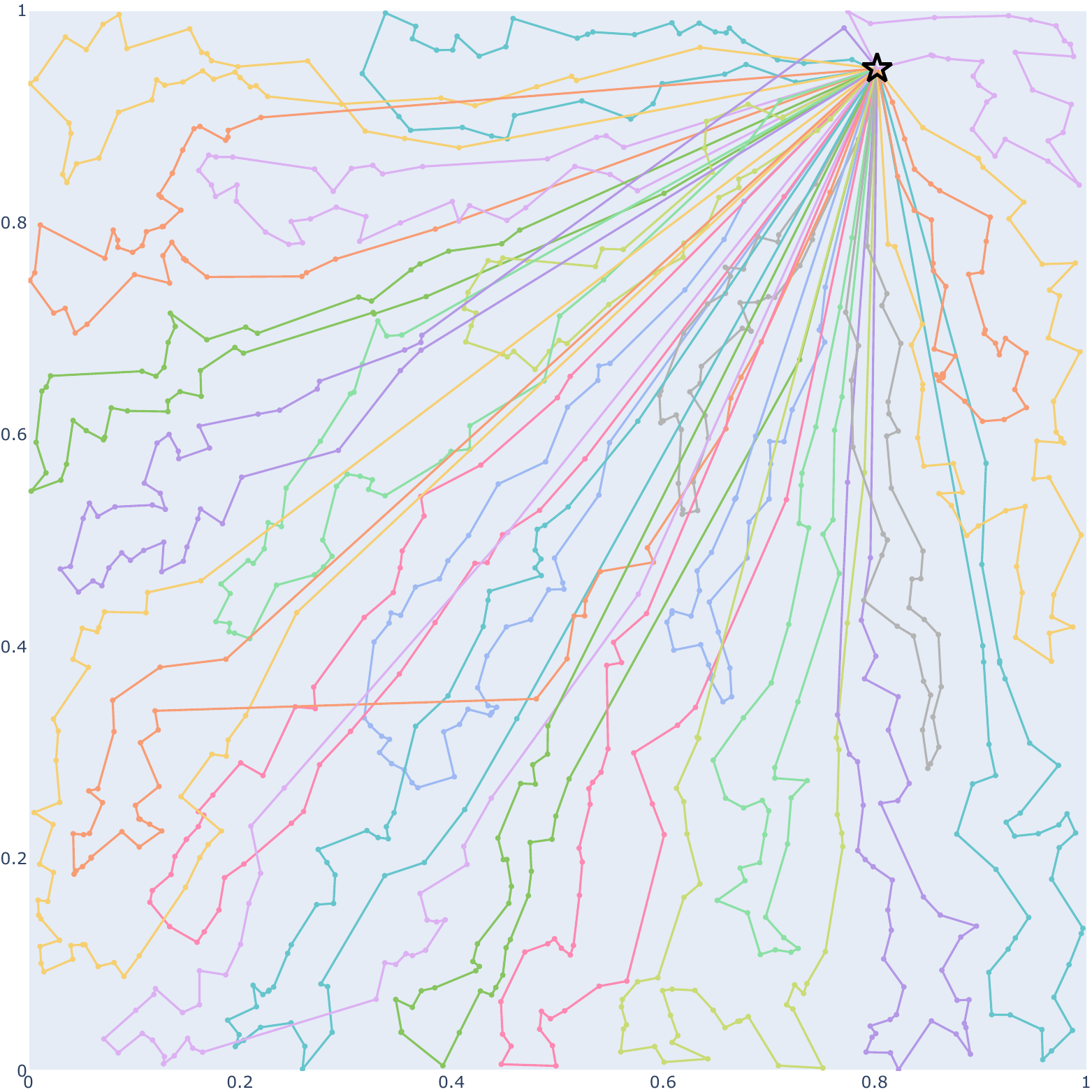}}
\subfigure[\shortstack{Cost=47.79 \\ SS.+LS.+glob.+loc.}]
{\label{fig:u_sl_4}\includegraphics[width=0.23\textwidth]{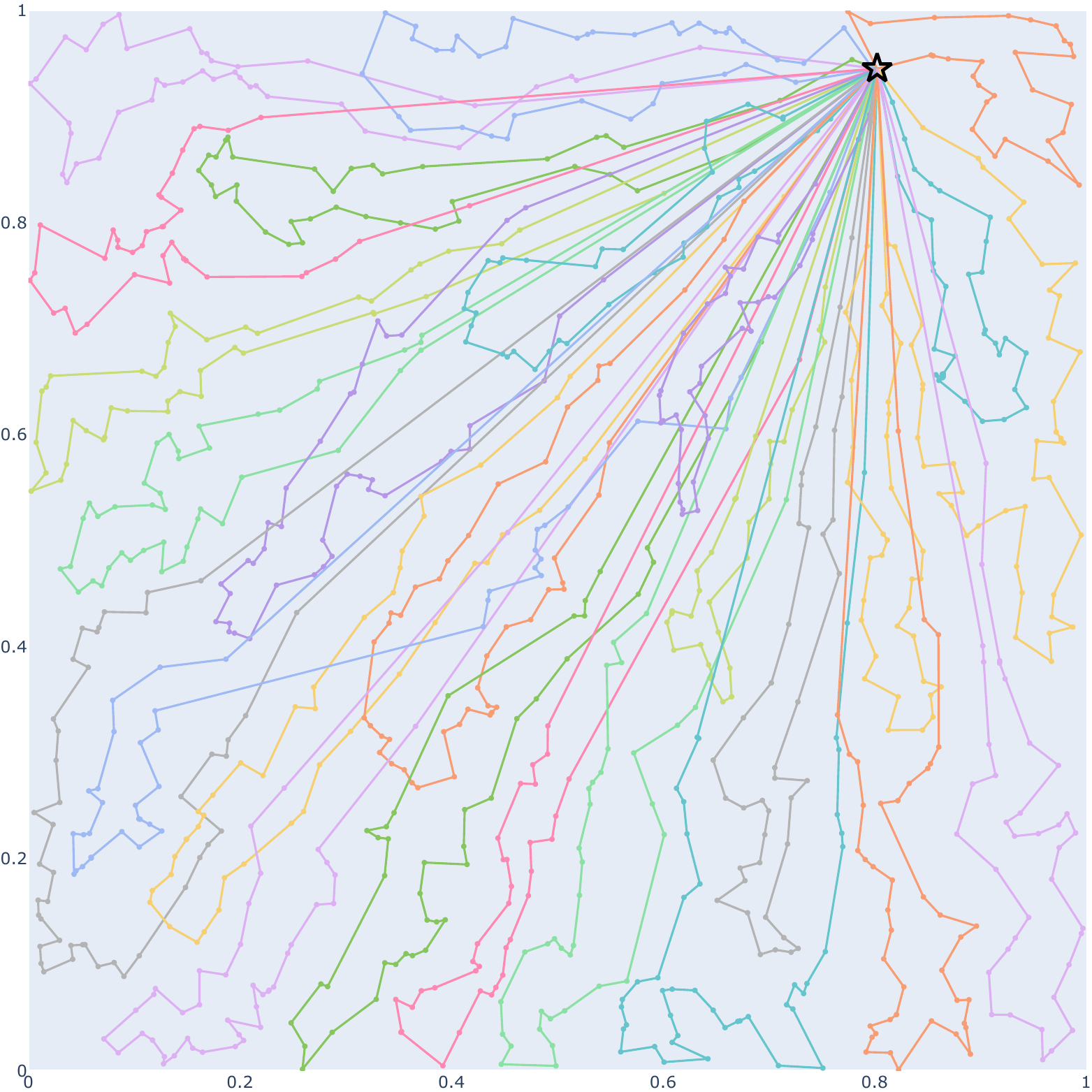}}

\rotatebox{90}{~~~~~~~~~~~~\scriptsize{CVRP1000+Gaussian}}
\subfigure[\shortstack{Cost=33.90 \\ SS.+glob.}]{\label{fig:g_sl_1}\includegraphics[width=0.23\textwidth]{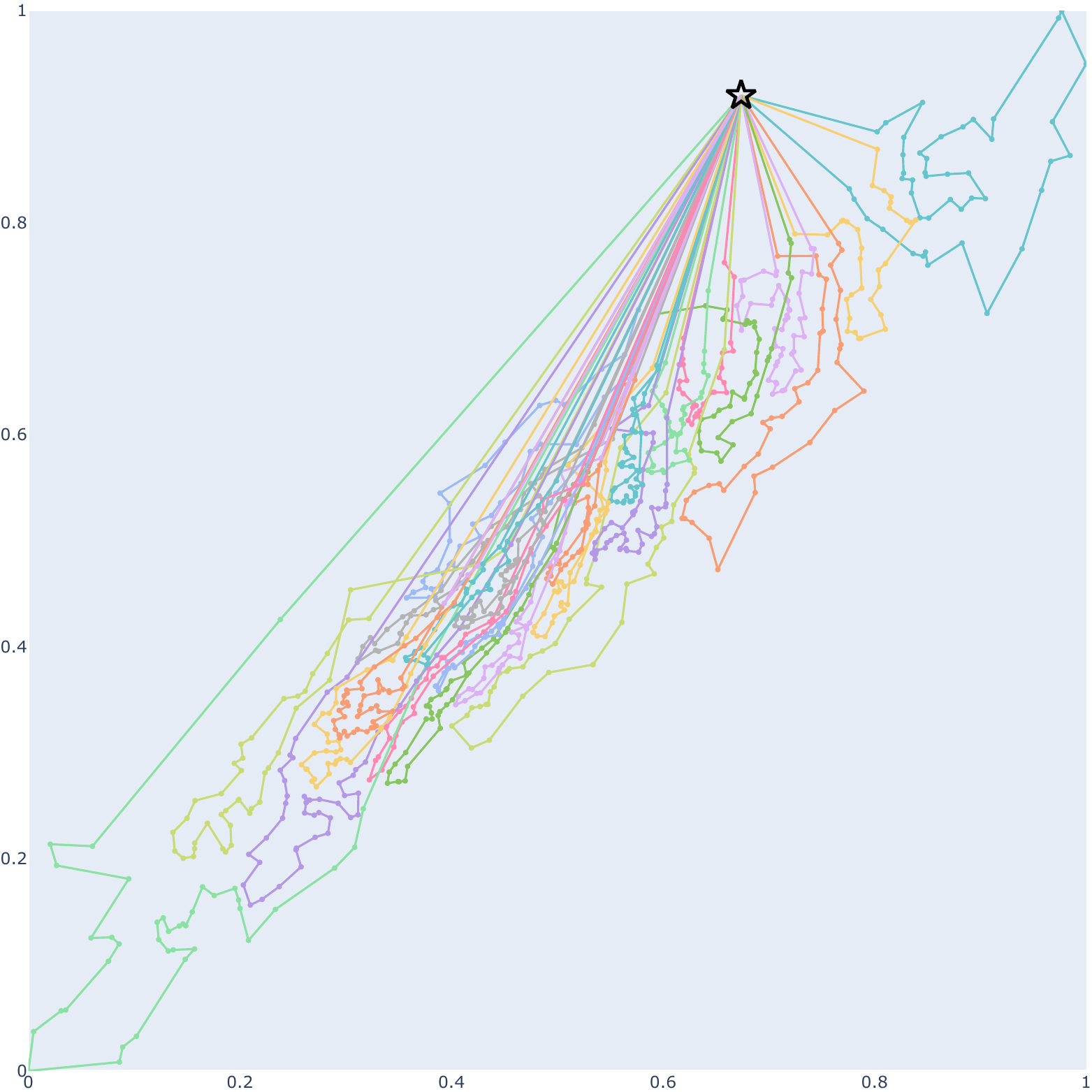}}
\subfigure[\shortstack{Cost=33.09 \\ SS.+glob.+loc.}]{\label{fig:g_sl_2}\includegraphics[width=0.23\textwidth]{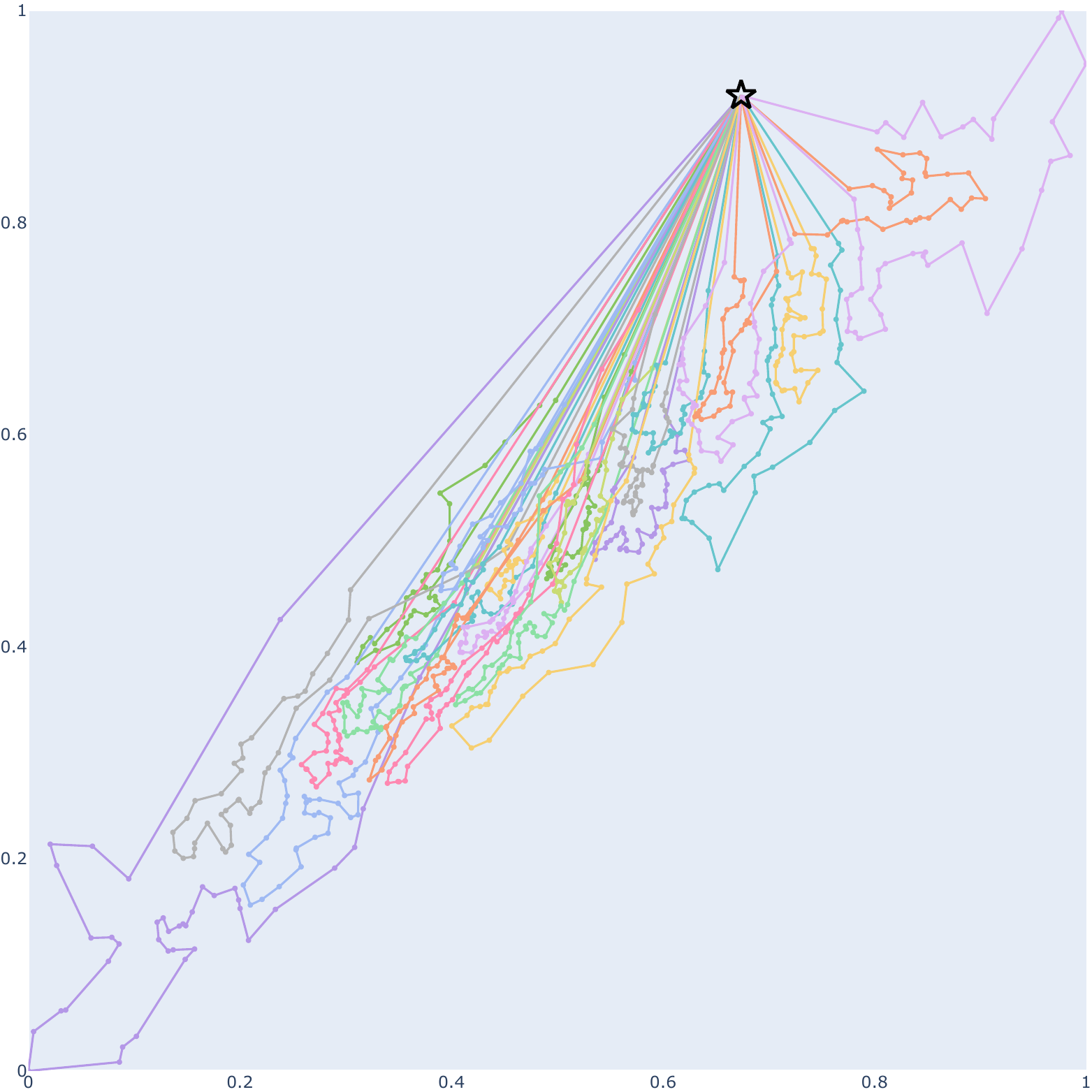}}
\subfigure[\shortstack{Cost=32.25 \\ SS.+LS.+glob.}]
{\label{fig:g_sl_3}\includegraphics[width=0.23\textwidth]{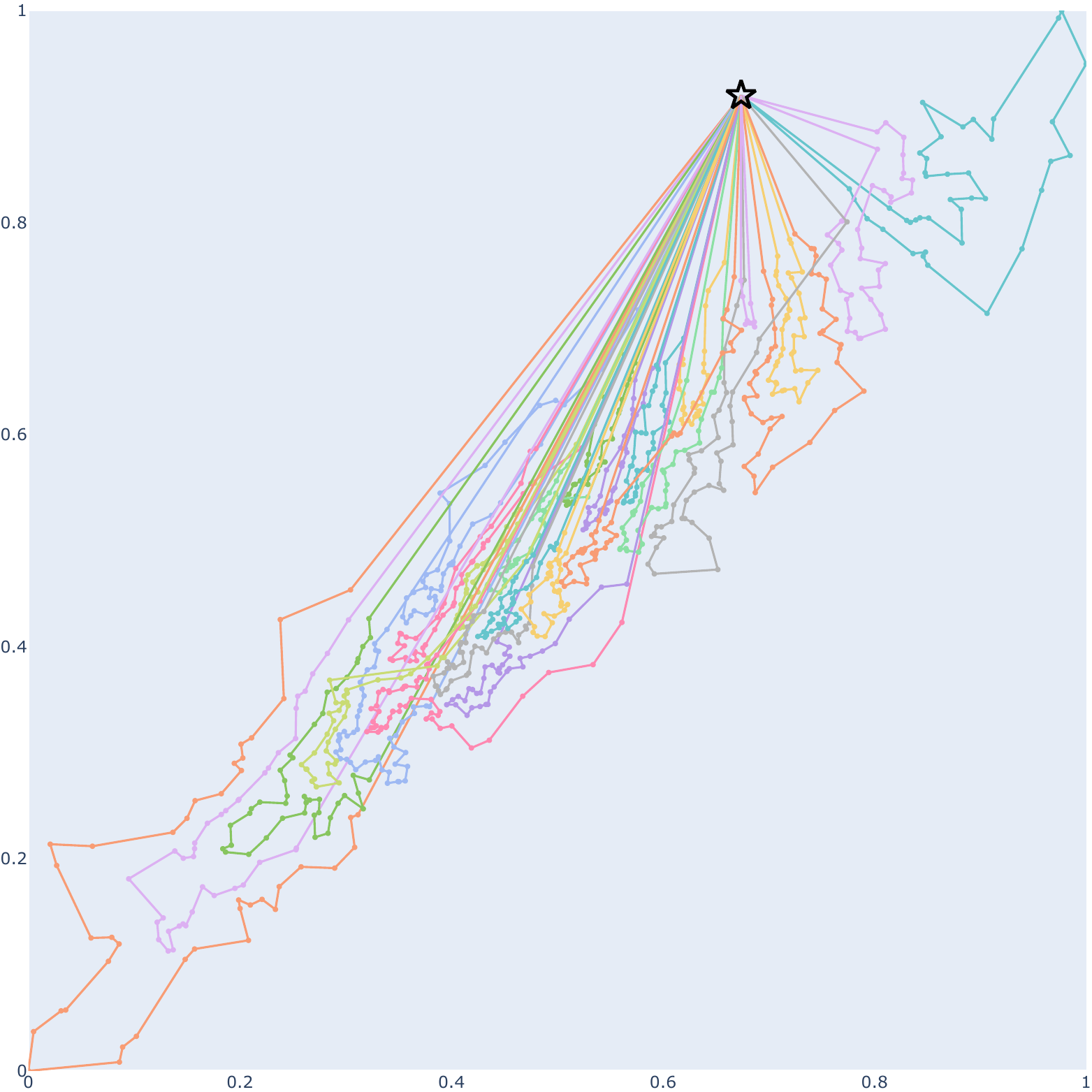}}
\subfigure[\shortstack{Cost=31.85 \\ SS.+LS.+glob.+loc.}]
{\label{fig:g_sl_4}\includegraphics[width=0.23\textwidth]{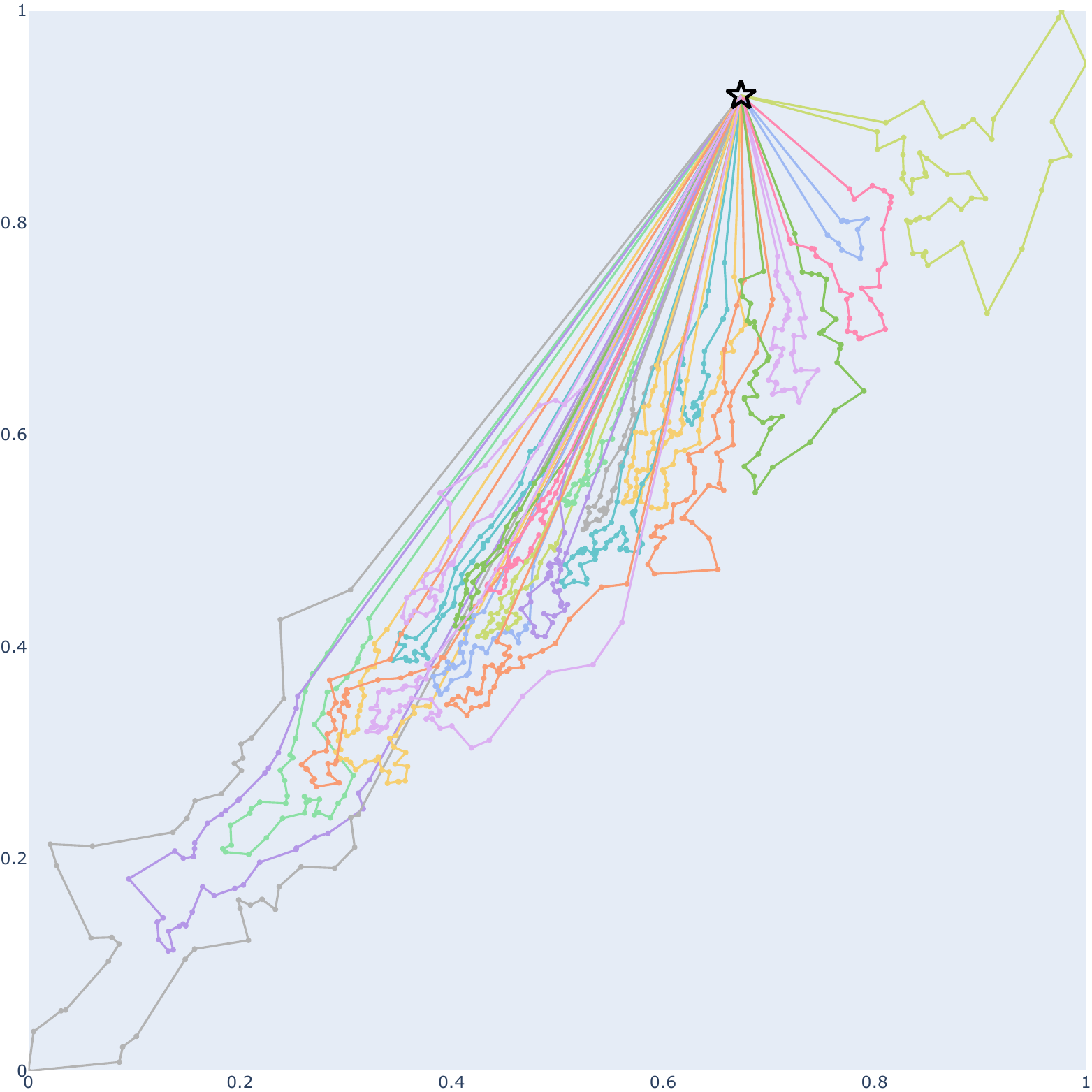}}

\rotatebox{90}{~~~~~~~~~~~~\scriptsize{CVRP1000+Explosion}}
\subfigure[\shortstack{Cost=58.14 \\ SS.+glob.}]{\label{fig:e_sl_1}\includegraphics[width=0.23\textwidth]{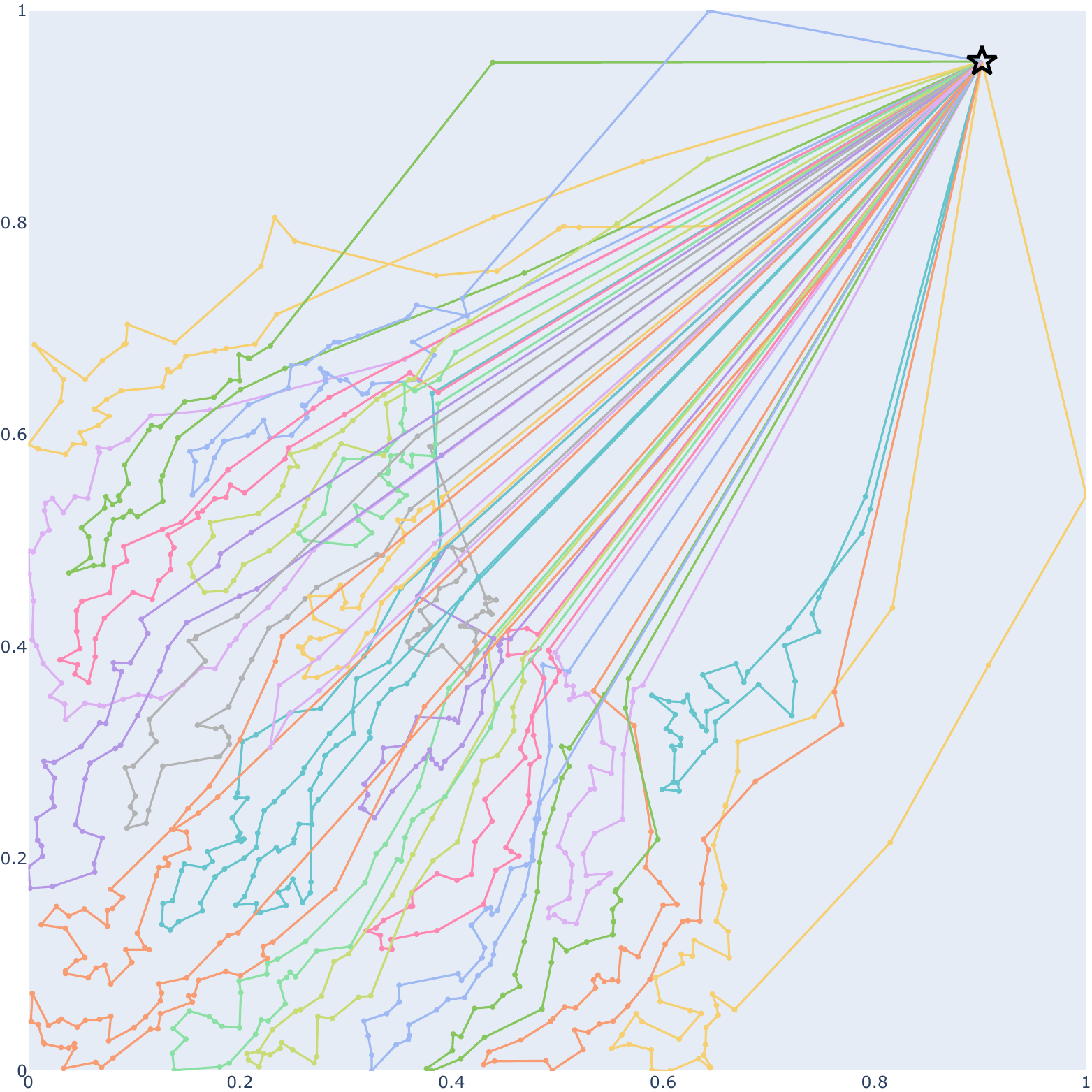}}
\subfigure[\shortstack{Cost=56.97 \\ SS.+glob.+loc.}]{\label{fig:e_sl_2}\includegraphics[width=0.23\textwidth]{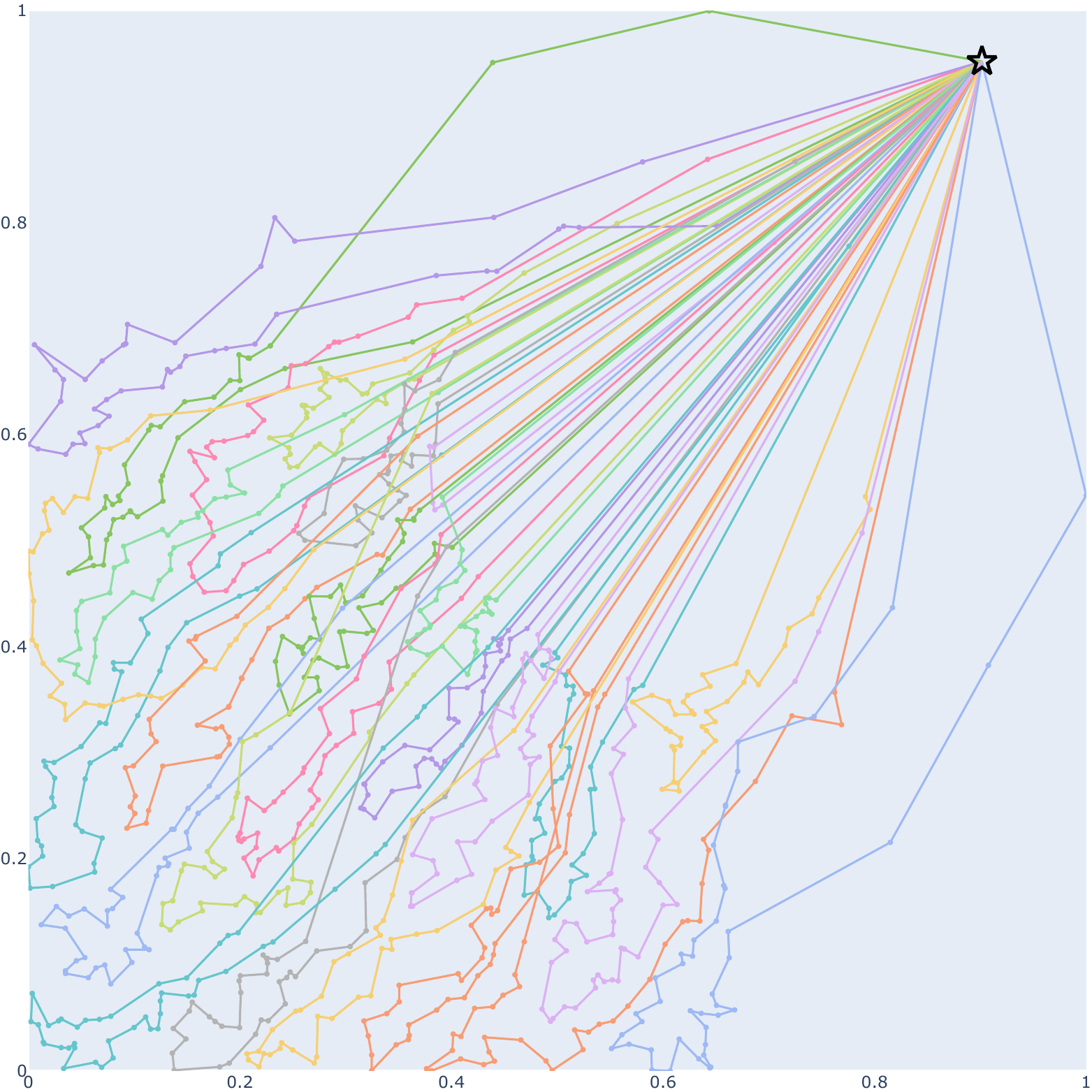}}
\subfigure[\shortstack{Cost=56.20 \\ SS.+LS.+glob.}]
{\label{fig:e_sl_3}\includegraphics[width=0.23\textwidth]{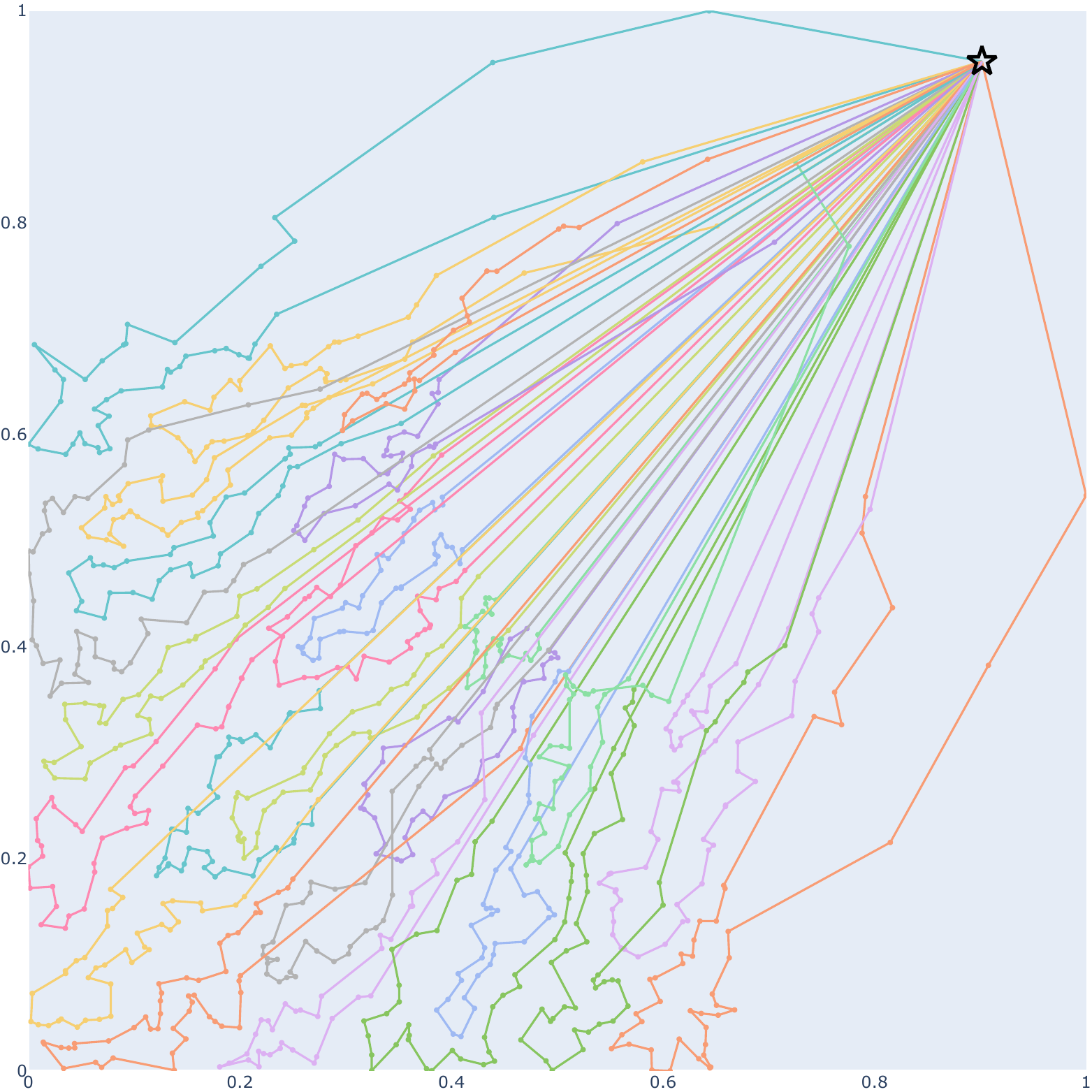}}
\subfigure[\shortstack{Cost=56.01 \\ SS.+LS.+glob.+loc.}]
{\label{fig:e_sl_4}\includegraphics[width=0.23\textwidth]{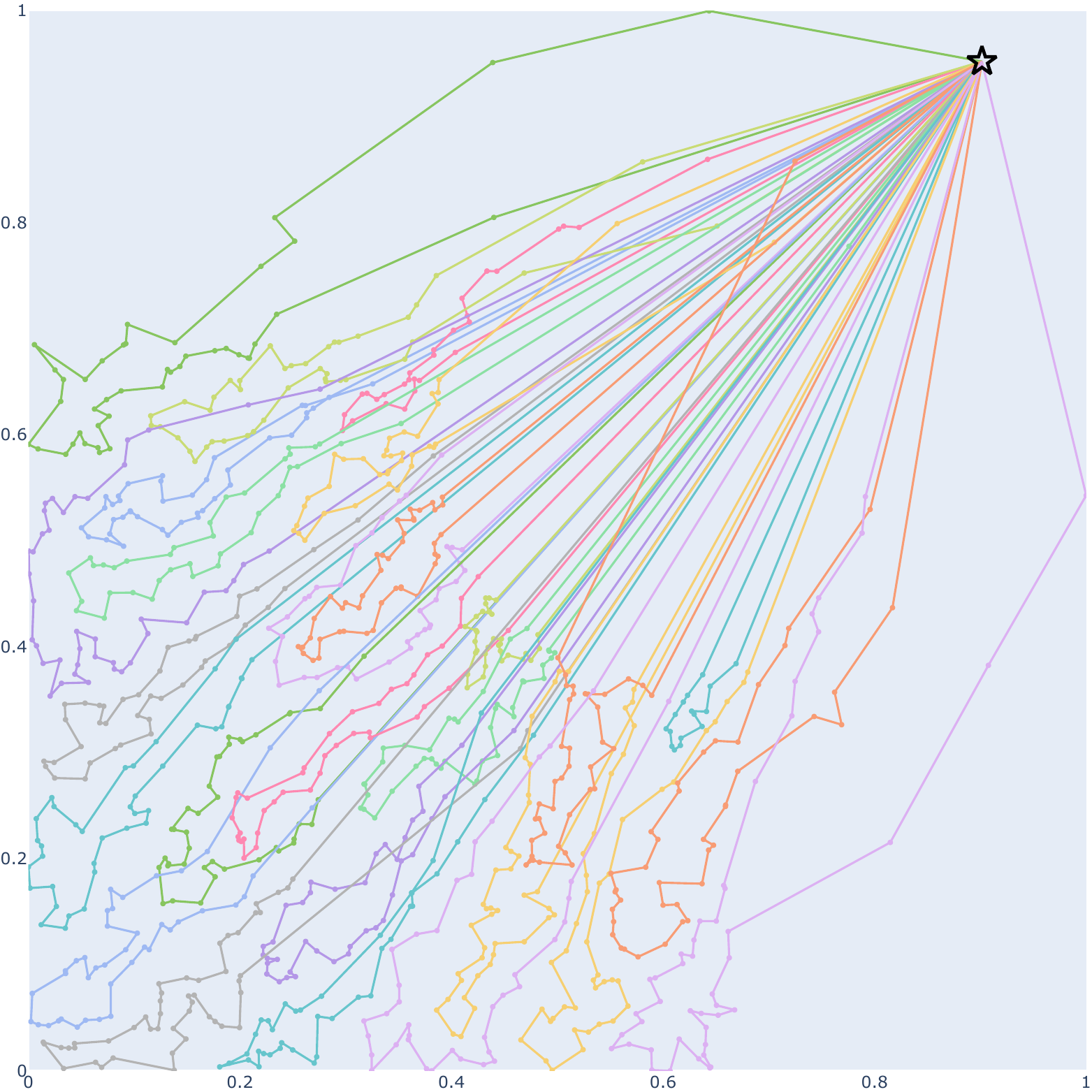}}

\rotatebox{90}{~~~~~~~~~~~~\scriptsize{CVRP1000+Rotation}}
\subfigure[\shortstack{Cost=47.27 \\ SS.+glob.}]{\label{fig:r_sl_1}\includegraphics[width=0.23\textwidth]{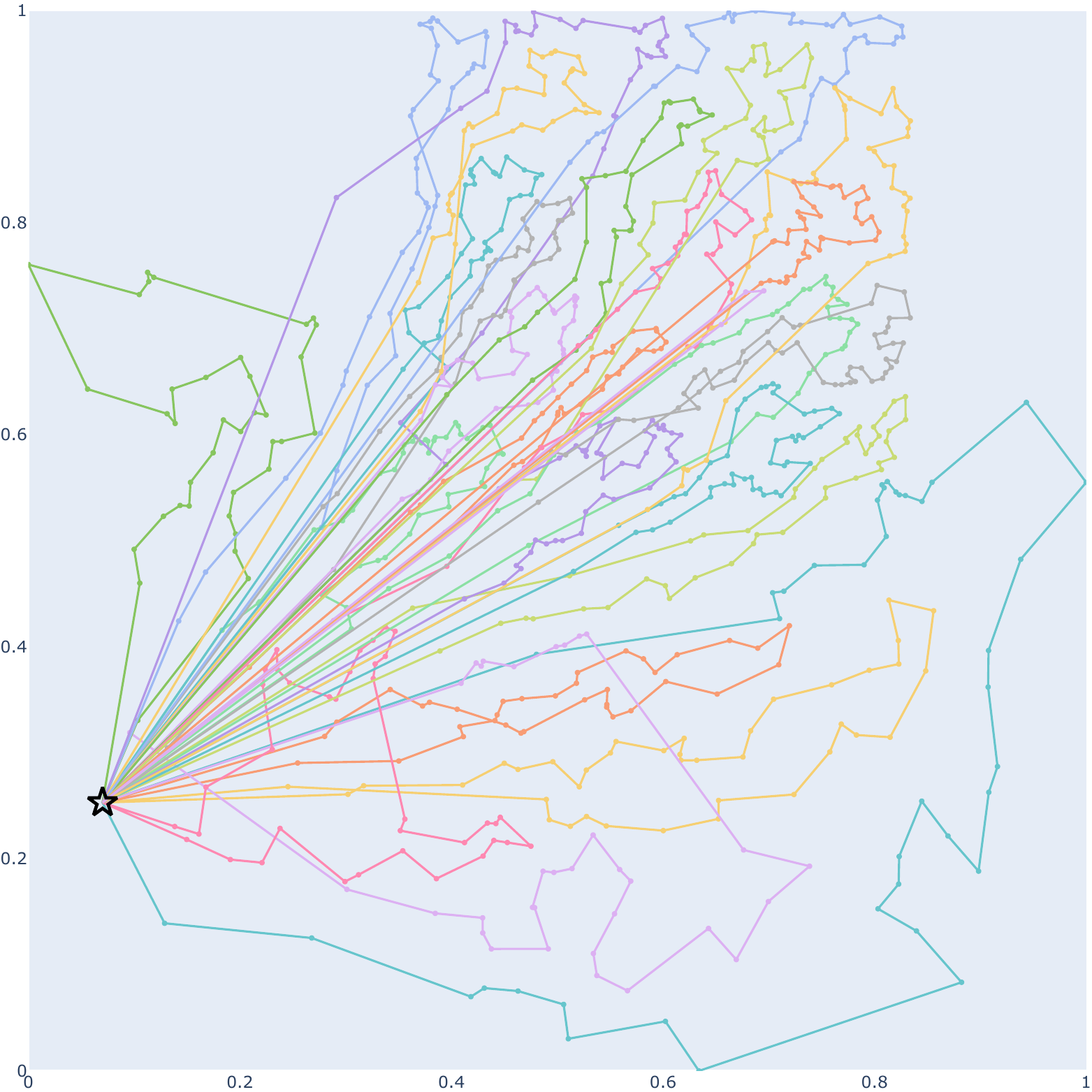}}
\subfigure[\shortstack{Cost=46.56 \\ SS.+glob.+loc.}]{\label{fig:r_sl_2}\includegraphics[width=0.23\textwidth]{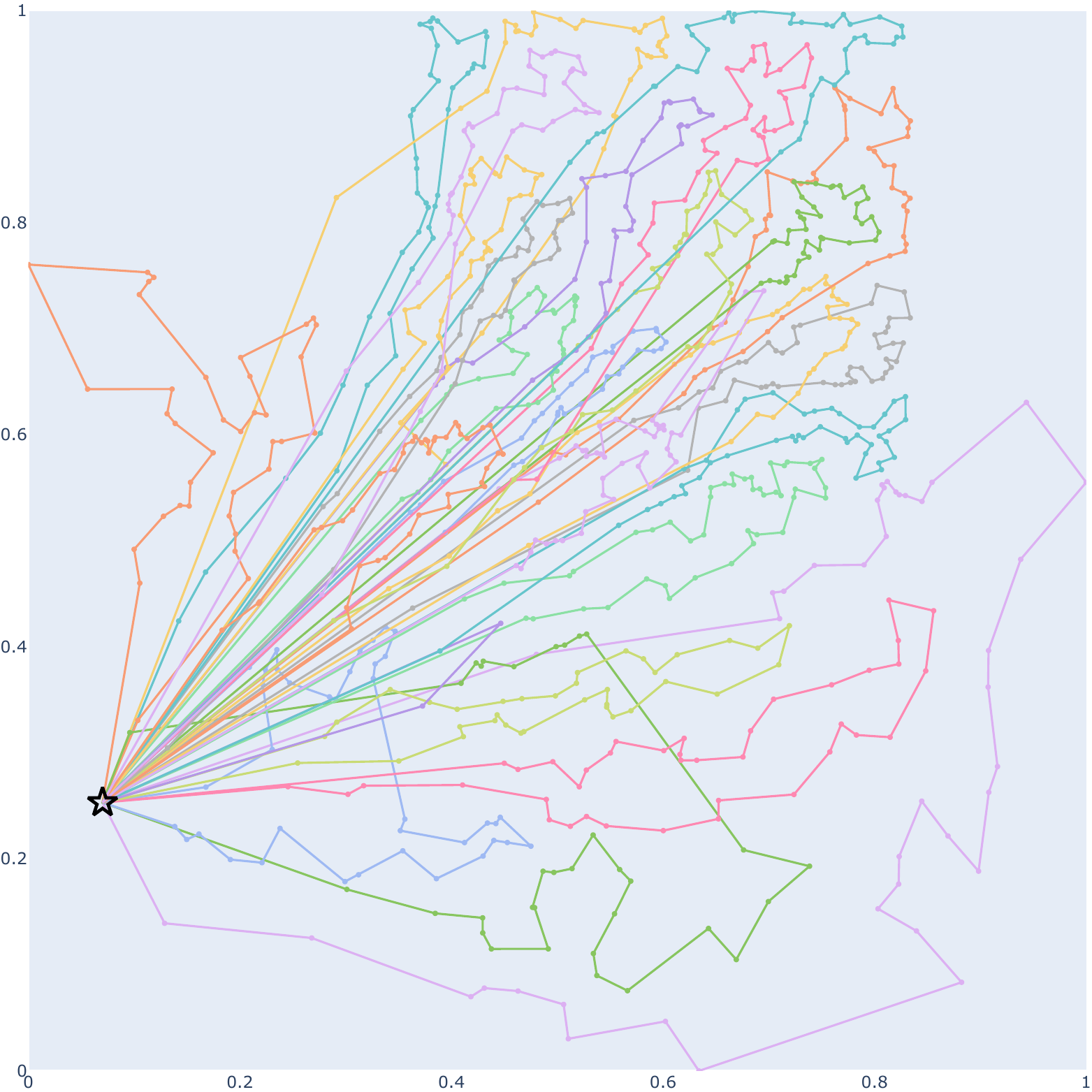}}
\subfigure[\shortstack{Cost=45.49 \\ SS.+LS.+glob.}]
{\label{fig:r_sl_3}\includegraphics[width=0.23\textwidth]{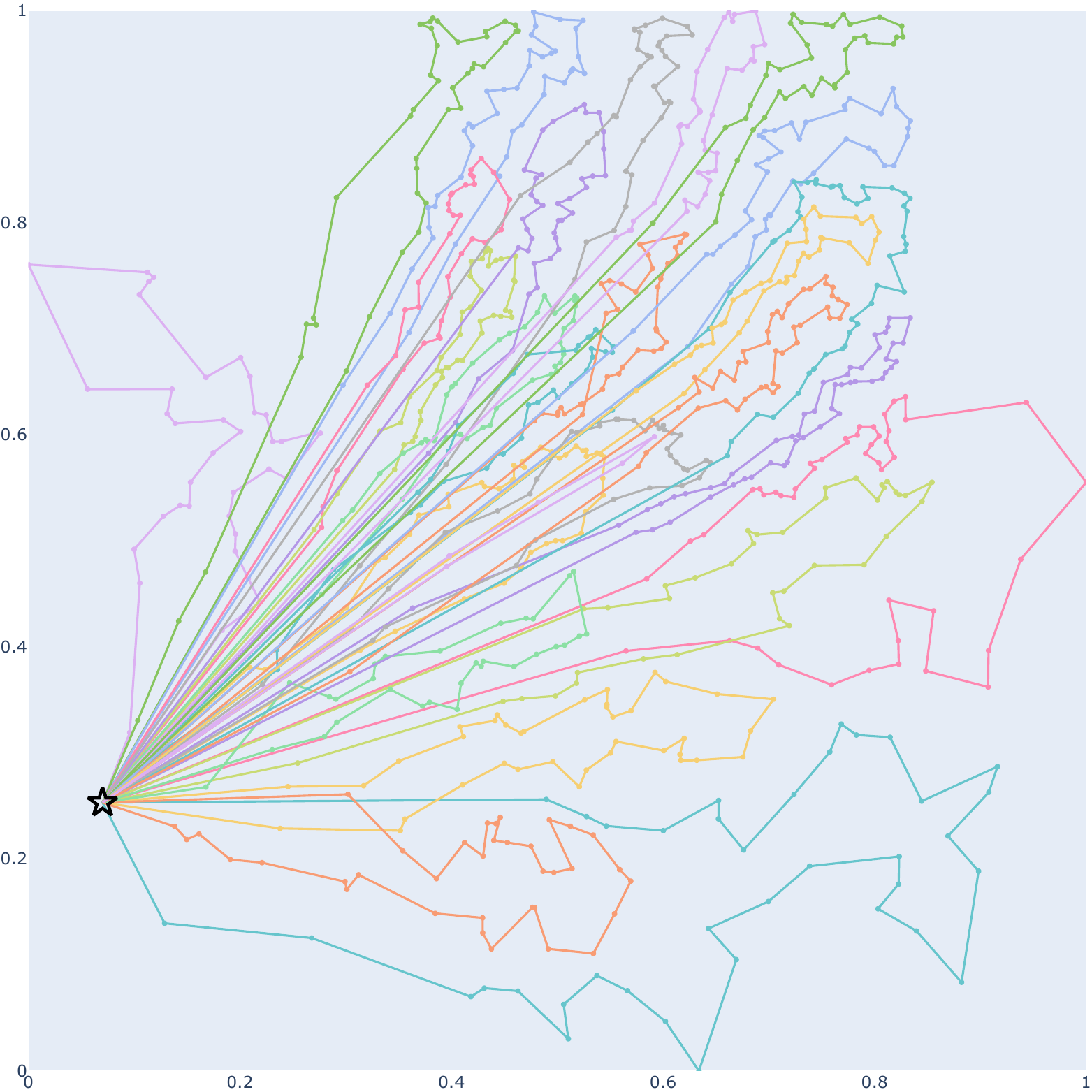}}
\subfigure[\shortstack{Cost=44.86 \\ SS.+LS.+glob.+loc.}]
{\label{fig:r_sl_4}\includegraphics[width=0.23\textwidth]{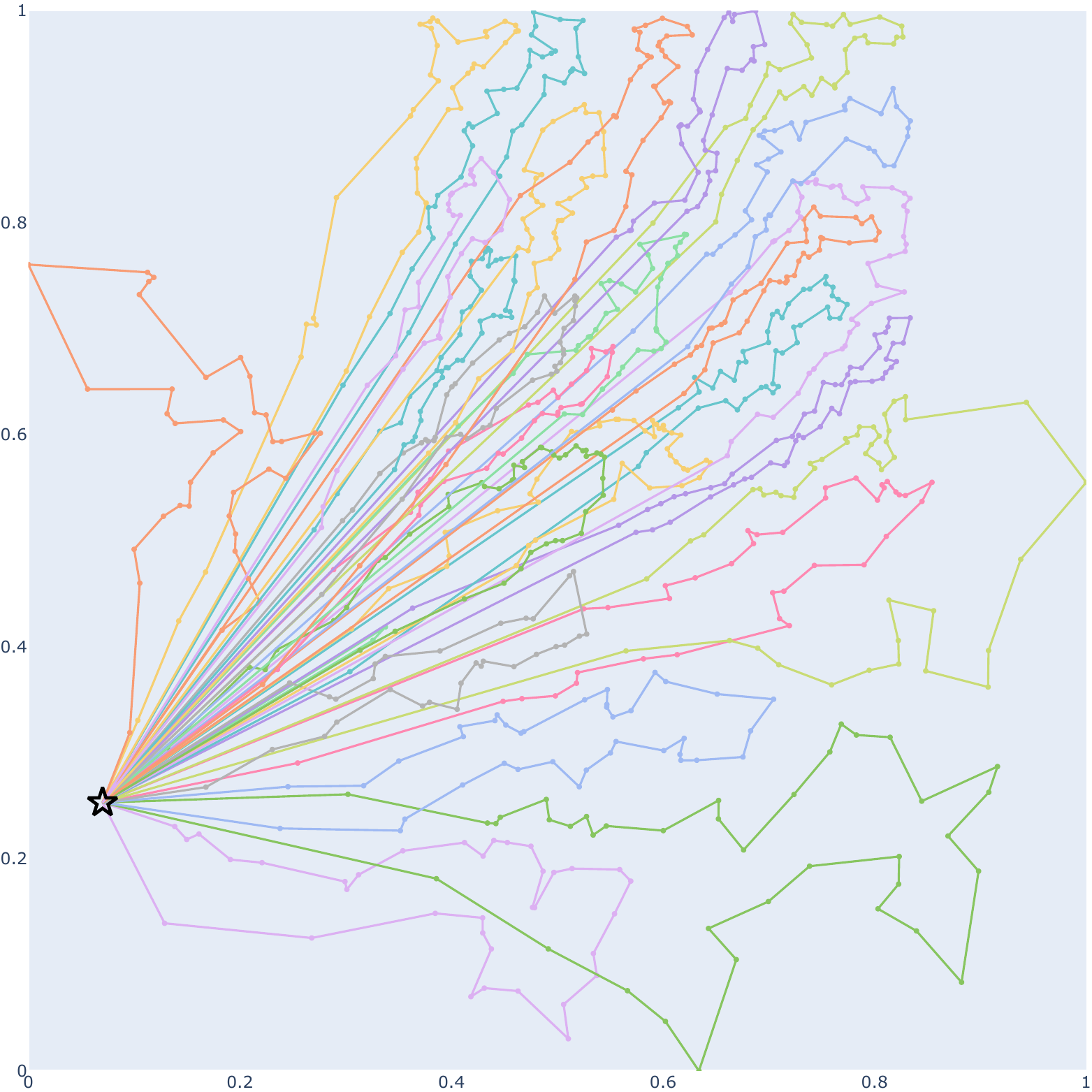}}

\caption{visualization of SL-driven HLGP routes.}
\label{fig:sl_visual}
\end{figure*}

\end{document}